\documentclass[12pt]{article}
\RequirePackage[OT1]{fontenc}
\RequirePackage{amsthm,amsmath,amssymb,subcaption,graphicx,epstopdf,enumitem,lmodern}
\RequirePackage[round,colon,authoryear]{natbib}
\RequirePackage[colorlinks,citecolor=blue,urlcolor=blue]{hyperref}
\usepackage{bm,color,fancyvrb,xcolor,mathtools}
\usepackage{amscd}
\usepackage{mathrsfs}
\usepackage{bbm}

\def\eps{\varepsilon}
\font\tencmmib=cmmib10 \skewchar\tencmmib '60
\newfam\cmmibfam
\textfont\cmmibfam=\tencmmib

\def\lessim{\ \lower4pt\hbox{$
		\buildrel{\displaystyle <}\over\sim$}\ }
\def\gessim{\ \lower4pt\hbox{$\buildrel{\displaystyle >}
		\over\sim$}\ }

\usepackage{algorithm,algorithmicx}
\usepackage[noend]{algpseudocode}
\usepackage{tensor}
\usepackage{verbatim}

\newtheorem{theorem}{Theorem}[section]
\newtheorem{proposition}[theorem]{Proposition}
\newtheorem{lemma}{Lemma}

\newtheorem{corollary}[theorem]{Corollary}
\newtheorem{remark}{Remark}[section]
\newtheorem{assumption}{Assumption}[section]

\DeclarePairedDelimiter{\abs}{\lvert}{\rvert}

\providecommand{\abs}[1]{\left\lvert#1\right\rvert}

\renewcommand{\hat}{\widehat}
\renewcommand{\tilde}{\widetilde}

\renewcommand{\hat}{\widehat}

\newcommand{\bfm}[1]{\ensuremath{\mathbf{#1}}}

\newcommand\numberthis{\addtocounter{equation}{1}\tag{\theequation}}

\def\bc{\bfm c}   \def\bC{\bfm C}  \def\CC{\mathbb{C}}
     
\def\be{\bfm e}     \def\EE{\mathbb{E}}

     \def\II{\mathbb{I}}

   \def\bM{\bfm M}

     \def\PP{\mathbb{P}}
     
     \def\RR{\mathbb{R}}

\def\bu{\bfm u}     
\def\bv{\bfm v}     
\def\bw{\bfm w}     
\def\bx{\bfm x}   \def\bX{\bfm X}  
\def\by{\bfm y}     
   \def\bZ{\bfm Z}  

\def\calA{{\cal  A}}

\def\calE{{\cal  E}}

\def\calM{{\cal  M}}

\def\calR{{\cal  R}}

\def\calU{{\cal  U}}

\def\calY{{\cal  Y}}

\newcommand{\bfsym}[1]{\ensuremath{\boldsymbol{#1}}}

\def\balpha{\bfsym \alpha}

\def\bgamma{\bfsym \gamma}

\DeclareMathOperator{\Var}{Var}

\DeclareMathOperator{\Cor}{Cor}

\DeclareMathOperator{\tr}{tr}

\def\eps{\varepsilon}

\def\newpage{\vfill\eject}

\newcommand{\vertiii}[1]{{\left\vert\kern-0.25ex\left\vert\kern-0.25ex\left\vert #1 
	\right\vert\kern-0.25ex\right\vert\kern-0.25ex\right\vert}}

\usepackage{tikz}
\usetikzlibrary{decorations.pathreplacing,calc}

\title{Transfer Learning via Regularized Random-effects Linear  Discriminant Analysis}
\author{Hongzhe Zhang$^1$, Arnab Auddy$^2$, Hongzhe Li$^1$\\
University of Pennsylvania$^1$ and The Ohio State University$^2$}
\date{\today}
\begin{document}
	\maketitle

\begin{abstract}
	Linear discriminant analysis is a widely used method for classification. However, the high dimensionality of predictors combined with small sample sizes often results in large classification errors. To address this challenge, it is crucial to leverage data from related source models to enhance the classification performance of a target model.
This paper proposes a transfer learning approach via regularized random-effects linear discriminant analysis, where the discriminant direction is estimated as a weighted combination of ridge estimates obtained from both the target and source models. Multiple strategies for determining these weights are introduced and evaluated, including one that minimizes the estimation risk of the discriminant vector and another that minimizes the classification error.
Utilizing results from random matrix theory, we explicitly derive the asymptotic values of these weights and the associated classification error rates in the high-dimensional setting, where $p/n \rightarrow \gamma$, with $p$ representing the predictor dimension and $n$ the sample size. Extensive numerical studies, including simulations and analysis of proteomics-based 10-year cardiovascular disease risk classification, demonstrate the effectiveness of the proposed approach.
\end{abstract}

\section{Introduction}

Large and diverse data sets are ubiquitous in modern applications, including those in genomics and medical decisions. It is of significant interest to integrate different data sets to obtain more accurate parameter estimates or to make a more accurate prediction or classification of an outcome. The success of a supervised statistical learning method  relies on the availability  of training data.  When data is scarce, choosing an appropriately flexible model becomes critical to achieving optimal prediction accuracy. This is a classic illustration of the well-known ``bias-variance tradeoff".
When building a prediction model for a target population, many auxiliary source data sets may exist and provide additional information for building the model for the target population. Modern techniques in the field of transfer learning \citep{pan2009survey,weiss2016survey} aim to exploit these additional information. Given a target problem to solve, transfer learning \citep{Torrey10} aims at transferring the knowledge from different but related samples or studies to improve the learning performance of the target problem. In biomedical studies, some clinical or biological outcomes are hard to obtain due to ethical or cost issues, in which case transfer learning can be leveraged to boost the prediction and estimation performance by effectively utilizing information from related studies. Other relevant approaches include meta-learning \citep{peng2020comprehensive, huisman2021survey}, domain adaptation \citep{redko2020survey, sun2015survey}, and, more recently, continual learning \citep{de2021continual}.

In the high dimensional setting, \cite{Li2022Transfer, li2023estimation} developed transfer learning methods for sparse high dimensional regressions and demonstrated that one can improve  predictions of gene expression levels using data across different tissues. Such sparse models work well when the true models are sparse and the sample sizes are large.  However, there are settings where sparse model assumption may not be valid. In genetics, estimating polygenetic risk scores (\textsc{prs}s) using genome-wide genotype data \citep{mak2017polygenic,poly2} is an active area of research. Such \textsc{prs}s can be used in risk stratification, or can be treated as risk factors in population health studies. However, due to very large number of genetic variants but relatively small sample sizes, building a \textsc{prs} model that accurately predicts the \textsc{prs} scores is challenging. 
Alternatively,  ridge regression, which does not require the sparseness assumption, but can handle the linkage disequilibrium among the genetic variants, provides a viable method for \textsc{prs} prediction.  \cite{zhang2023transfer} studied the 
estimation and prediction of random coefficient ridge regression in the setting of transfer learning, where in addition to observations from the target model, source samples from different but possibly related regression models are available. The informativeness of the source model to the target model can be quantified by the correlation between the regression coefficients. 

The method of  \cite{zhang2023transfer} is developed for continuous outcomes.  In this paper, we propose a transfer learning framework for regularized linear discriminant analysis (RDA) \citep{friedman1989regularized}, which we term TL-RDA. Our model assumes a random classification weights setup, where the means of the covariates between two classes differ by a random quantity, $\delta$, with zero mean and constant variance. The auxiliary and target populations are related through the correlation structure of $\delta$. The TL-RDA framework aims to estimate the Bayes optimal predictor by combining naive RDA estimates from both the auxiliary and target populations through a weighted summation. The weights are designed to minimize the distance between the TL-RDA estimator and the Bayes optimal direction in a high-dimensional setting, where the number of features $p$ grows proportionally to the sample size $n$ in all populations. We derive the explicit asymptotic error rate for TL-RDA and show that it achieves the lowest error rate among all estimators based on weighted summations, including naive RDA using only target population data.

The remainder of the paper is organized as follows.  We first provide a detailed problem setup and outline the proposed TL-RDA approach in Section 2. We then introduce various types of weights used to integrate auxiliary information, followed by the technical assumptions required for analyzing the estimator in the context of random matrix theory in Section 3.  In Section 4, we present the analysis of TL-RDA in the proportional regime, providing explicit asymptotic expressions for the different weighting schemes, along with their corresponding error rates. Section 5 offers interpretations of these weights and guidance on how users can select the most appropriate weights for their applications. In the first five sections, we have assumed all populations share the same population covariance matrix. Section 6 extends the TL-RDA to a heterogeneous population covariance matrix set up. Finally, in Section 7, we demonstrate the performance of TL-RDA on a binary 10-year cardiovascular disease risk classification  proteomics data in the Chronic Renal Insufficiency Cohort (CRIC).

\section{Transfer Learning via Regularized Discriminant Analysis}

We consider the setting of two-class LDA in the setting of transfer learning, where we have data observed from both the target model, indexd by $K$, and $K-1$ source models, indexed by $k=1, \cdots, K-1$. 
We assume that all models,   $k = 1,\cdots, K$,  follow the classic two-class Gaussian mixture model. More specifically, for $i=1,\dots,n_k$ and $k=1,\dots,K$,
\begin{align}\label{as:TCG}
	y_k \in \{-1, +1\} \ \ \
	\PP(y_k = \pm 1) = \pi_{\pm 1} \ \ \
	(X_k)_i|y_k \sim N(\mu_{y, k}, \Sigma)
	.
\end{align}
Here $(X_k)_i, i = 1, \cdots, n_k$ is a $p-$dimensional vector, and for ease of notation, we write $\bX_k=((X_k)_1\,(X_k)_2\,\ldots (X_k)_{n_k})^{\top}\in \RR^{n_k\times p}$ as a $n_k\times p$ dimensional matrix.
For simplicity of notation, we assume that  the mixing proportions $\pi_{\pm 1}$ remain the same across populations. In fact, without much loss of generality, we mainly discuss the simpler case $\pi_{-1} = \pi_{+1}=1/2$. The more general case can be managed in a similar manner without significant technical difficulty (Appendix C). Moreover, in most of the paper, we assume the population covariance matrix are the same for all $K$ populations, and are denoted by $\Sigma$. This assumption is relaxed in Section~\ref{sec:hetero-cov} where we allow each population to have their own covariance matrix $\Sigma_k$ for $k=1,\dots,K$. 

Under this set up, the Bayes optimal prediction direction for the target population $K$ \citep{anderson1958introduction}  is given by, 
\begin{gather*}
d_{Bayes} := \Sigma^{-1} \delta_{K}
\quad
\text{where}
\quad
\bar{\mu}_K = \frac{\mu_{+1, K} + \mu_{-1, K}}{2}
\quad
\text{for}
\quad
\mu_{\pm1, K}  =\bar{\mu}_K \pm \delta_K,
\end{gather*}
and  $\delta_{K}=(\mu_{+1, K} - \mu_{-1, K})/2$. 
The Bayes prediction for a testing data point $x_0$ from population $K$ is
\[\hat{y}_{Bayes}(x_0) = {\rm sign}\left[d_{Bayes}^{\top} (x_0 - \bar{\mu}_K )\right].\]

The regularized discriminant analysis (RDA) classifier uses an empirical version of the unkown Bayes direction. The RDA classifier for population $k$ is a linear classifier 
\[
\hat{y}_{RDA, k}(x_0) = {\rm sign}\left(\hat{d}^{\top}_k x_0 + \hat{b}_k\right),
\]
where we use the plug-in estimates for the population parameters $\Sigma_k$ and $\delta_k$ as follows:
\begin{gather*}
\hat{d}_k = (\hat{\Sigma}_k + \lambda_k \II_p)^{-1} \hat{\delta}_k\\
\hat{b}_k = - \hat{\delta}^{\top}_k (\hat{\Sigma}_k + \lambda_k \II_p)^{-1}(\hat{\mu}_{-1, k} + \hat{\mu}_{+1, k}) / 2\\
\hat{\mu}_{\pm1, k} = \frac{2}{n_k} \sum_{i : (y_k)_i  = \pm 1} (X_k)_i \ \ \ \hat{\delta}_k = \frac{\hat{\mu}_{+1, k} - \hat{\mu}_{-1, k}}{2} \\ 
 \hat{\Sigma}_k = 
 \frac{1}{n_k-2}
 \sum_{i = 1}^{n_k} [(X_k)_i - \hat{\mu}_{(y_k)_i, k}] [(X_k)_i - \hat{\mu}_{(y_k)_i, k}]^{\top}.
\end{gather*}
 Here $ \hat{\Sigma}_k$ is the usual sample covariance matrix, while $\hat{\mu}_{\pm1, k}, \hat{\delta}_k$ are simple estimators and the population-level counterpart $\mu_{\pm1, k}, \delta_{k}$. The sample covariance matrix is penalized on its diagonal to overcome the overwhelming variances in this estimation when $p$ grows proportionally with $n$. As the penalization parameter $\lambda_{k}$ goes to zero or infinity, this classifier recover the Fisher's discriminant analysis or the naive Bayes method \citep{bickel2004some}. 
 
 To utilize the underlying relatedness of the LDA problems in $K$ populations, for a vector of weights $\bw\in \RR^K$, we now define a classifier based on a weighted linear combination of the population specific discriminator vectors $\hat{d}_k$s as follows:
 \[
 \hat{d}(\bw)=\sum_{k=1}^Kw_k\hat{d}_k.
 \]
 Note that taking the co-ordinate basis vectors as weights, i.e., $\bw=\be_k\in \RR^K$,  we end up with the population specific discriminant directions $\hat{d}(\be_k)=\hat{d}_k$ for $k=1,\dots,K$. Now using the weighted combination discriminant vector we define the \emph{transfer-learning} (TL) classifier
 \begin{equation}\label{eq:tl-rda}
 	\hat{y}_{\bw}(x_0) = {\rm sign}\left( \hat{d}(\bw)^{\top} x_0 + \hat{b}_K\right)
 \end{equation}
 where $x_0\in \RR^p$ is a test point in the target population. We call this TL-RDA. We use the regular intercept term $\hat{b}_K$ as all intercept terms are asymptotically zero in the regime considered. 
 
 We formulate two criteria to be optimized over $\bw$. Let us recall the Bayes classification direction $d_{Bayes}$. Now for a testing data pair $(x_0, y_0)$ sampled from population $K$, we consider the two errors:
 \begin{gather}
	\left\Vert d_{Bayes} - \sum_{k = 1}^{K} w_k \hat{d}_{k}\right\Vert_2^2, \label{eq:CrEst} \\
\EE_{x_0} \left [(d_{Bayes} - \sum_{k = 1}^{K} w_k \hat{d}_{k})^{\top} x_0 \right ]^2.  \label{eq:CrPred}
 \end{gather}
 The first criteria compares the transfer-learning classifier (TL-RDA) with the Bayes optimal classifier in terms of the estimation error, while the second criteria compares the two classifiers in terms of their prediction errors. In later sections, we present explicit solutions for optimal $\bw$ that optimizes the two criteria above. Though none of \eqref{eq:CrEst} or \eqref{eq:CrPred} are explicitly related to prediction error rate, we show the optimal weight that minimizes  \eqref{eq:CrPred} also minimizes the prediction error rate in the target population. We will also show the weights that minimizes these two criteria, are in fact related to one another.
 
 In addition, we assume that the population covariance matrix are the same for all $K$ populations (assumption $\ref{as:TCG}$). With this in mind, a natural way to further exploit mutual information across populations is to use a pooled sample covariance for all discriminant directions $\hat{d}_k$. We write $\hat{\Sigma}_P$ to denote the pooled sample covariance matrix 
 \[\hat{\Sigma}_P = \sum_{k = 1}^{K} \sum_{i = 1}^{n_k} [(X_k)_i - \hat{\mu}_{y_i, k}] [(X_k)_i - \hat{\mu}_{y_i, k}]^{\top} /  \sum_{k = 1}^{K} (n_k - 2). \] 
 We then define the pooled classification weights and pooled transfer learning regularized discriminant analysis (TLP-RDA) estimator as 
 \begin{align*}\label{eq:tlp-rda}
 	\hat{y}^P_{\bw}(x_0) =&~ {\rm sign}\left(\hat{d}_k^P(\bw)^{\top} x_0 + \hat{b}_K\right)\numberthis
 \end{align*}
where the direction estimates $\hat{d}_k^P$ are now estimated using the pooled covariance matrix as follows
\[
\hat{d}_k^P := ~  (\hat{\Sigma}_P + \lambda_k \II_p)^{-1} \hat{\delta}_k 
\quad \text{and} \quad 
\hat{d}^P(\bw) :=~ \sum_{k = 1}^{K} w_k \hat{d}_k^P.
\]
Once again using \eqref{eq:CrEst} and \eqref{eq:CrPred}, we can optimize $\bw$ in $\hat{y}^P_{\bw}(x_0)$ with respect to the two criteria above as well, and we will discuss how to choose between $\hat{y}^P_{\bw}(x_0)$ and $\hat{y}_{\bw}(x_0)$. 

This form of TL estimator \eqref{eq:tl-rda} has been used in \citet{zhang2023transfer} and \citet{dobriban2020wonder} for regression methods. \cite{helm2024approximately} also considers a TL estimator that aggregates auxiliary information by weighted summations of individual discriminant directions. However, their aggregation is less adaptive in the form of $ \hat{d}(\bw)^*= a \hat{d}_1 + (1 - a) \sum_{k=2}^K \hat{d}_k / (K-1)$ where $a$ is a constant. In addition, the final weight $a$ minimizes only empirical quantities and provides no guarantee of minimizing the population level criteria such as \eqref{eq:CrEst} or \eqref{eq:CrPred}. 

  \section{Random-effects LDA and the Assumptions on Population Parameters }
 To find the weights that minimize the estimation or prediction risk defined as \ref{eq:CrEst} and \ref{eq:CrPred},  we consider random-effects LDA where the population means are random and are potentially correlated across different models.  In addition, we also make the random matrix assumptions on the covariance matrix, which allows us to apply the  recent advances in random matrix theory to derive the limiting values of the weights and the corresponding classification errors.

  \subsection{Classification Weights}
A key parameter in two-class classification problem is the classification weights $\delta_k$ that separate the two classes. In this paper, we consider a random classification weights set up formalized below.
 \begin{assumption}[Random Classification Weights]\label{as:RCW} The following conditions hold for populations $k = 1, \cdots, K$.
 	\begin{enumerate}
 		\item  The class-specific population mean vectors $\mu_{-1, k}, \mu_{+1, k}\in \RR^p$ are randomly generated as $\mu_{\pm 1, k} = \bar{\mu}_k \pm \delta_k$, where each  $\delta_k$ has i.i.d. coordinates with 
 		\[
 		\EE((\delta_{k})_i) = 0 \ \ \ \Var((\delta_{k})_i) = \alpha_{k}^2 / p \ \ \ 
 		\EE(|(\delta_{k})_i|^{4 + \eta}) \leq \frac{C}{p^{2 + \eta}}\]
 		for fixed $C, \eta > 0$.
 		
 		\item $\bar{\mu}_k$ are either fixed or randomly distributed independent of $\delta_{k}, \bX_k, y_k$. The second moment of $\bar{\mu}_k$ are bounded almost surely such that $\lim \sup _{p\rightarrow \infty}  \Vert\bar{\mu}_k\Vert / p^{1/2 - \xi} \leq C$ for some constants $\xi , C>0$.
 	\end{enumerate}
 \end{assumption}
 The parameter $\alpha_{k}^2$ here plays the role of signal strength. This key assumption asserts that all coordinates of $\delta_{k}$ have the same zero mean and diminishing variance, therefore, all coordinates of an observation play equally important roles on determining the response class. This assumption also used by \citet{dobriban2018high}, is standard in the large $n$, large $p$ regime. In addition, the bounded moment assumptions allow use to circumvent difficulties arise when minimizing criteria \eqref{eq:CrEst} and  \eqref{eq:CrPred} which depends on unknown quantities such as $\delta_k, \Sigma^{-1}$ by invoking Lemma~\ref{lem:quadConv}. The performance of the TL estimators clearly  depends on how the $K$ populations are related. We further pose the correlated classification weight assumption below.
 \begin{assumption}[Correlated Classification Weights]
 	\label{as:CCW} The $\delta_{k}$ for $k = 1, \cdots, K$ are correlated across populations such that $\Cor(\delta_{k}, \delta_{k'}) = \rho_{k k'}\II_p, k \neq k'$.
 \end{assumption}
 
 \noindent For most of the paper, we assume that the parameters $\alpha_k^2$ and $\rho_{k,k'}$ are known constants. We discuss consistent estimators for these parameters in Appendix C.

 \subsection{Random Matrix Assumption and Related Results}
 The sample covariance matrix is an important part of the transfer learning estimator. We consider the Marchenko-Pastur type sample covariance matrices as in several relevant works \citep{dobriban2018high, zhao2023cross, zhang2023transfer}. For a symmetric matrix, we can characterize its spectral distribution by a cumulative distribution function that places equal point mass on its eigenvalues. Our asymptotic analysis is based on the convergences of spectral distributions of sample covariance matrices, for which the following assumptions are required.
 \begin{assumption}[\textsc{rmt} assumption]\label{as:RMT} 
 	For $k = 1, \cdots, K$, the design matrix $\bX_k\in \RR^{n_k\times p}$ is generated as 
 	\[
 	\bX_k=
 	\left(\mu_{(y_k)_1,k}\,\mu_{(y_k)_2,k}\,\ldots\,\mu_{(y_k)_{n_k},k}\right)^{\top}
 	+\bZ_k \Sigma^{1/2}
 	\]
 	for a matrix $\bZ_k\in \RR^{n_k\times p}$ with $i.i.d.$ entries coming from an infinite array. The entries $(Z_k)_{ij}$ of $\bZ_k$ satisfy the moment conditions: $\EE[(Z_k)_{ij}] = 0$ and $\EE[(Z_k)_{ij}^2] = 1$. 
 	\begin{enumerate}
 		\item The population covariance matrix $\Sigma\in \RR^{p\times p}$ is deterministic. The observations have unit variance, i.e., $\Sigma_{jj} = 1$ for $j = 1, \cdots, p$. 
 		\item The eigenvalues of $\Sigma$ are uniformly bounded from above and away from zero with constants independent of the dimension $p$.
 		\item The sequence of spectral distributions $T := T_{\Sigma, p}$ of $\Sigma := \Sigma_{p}$ converges weakly to a limiting distribution $H$ supported on $[0, \infty)$, called the population spectral distribution (\textsc{psd}).
 	\end{enumerate}
 \end{assumption}

\noindent Under the assumptions above, the Marchenko-Pastur theorem \citep{MPlaw} claims that the empirical spectral distribution (\textsc{esd}) of the sample covariance matrix $\hat{\Sigma}$ converges weakly (in distribution) to a limiting distribution $F_{\gamma} := F_{\gamma}(H)$  supported on $[0, \infty)$ with probability $1$. 
 
 For any distribution $G$ supported on $[0, \infty)$, we define its Stieltjes transform as
 \begin{equation}\label{eq:def-Stieltjes}
 	m_G(z) := \int_{l = 0}^\infty \frac{d G(l)}{l-z}, \ z \in \CC \ \backslash \ \RR^+. 
 \end{equation}
 The \textsc{esd} of sample covariance matrix is uniquely determined by a fixed-point equation for its Stieltjes transform. The limit of the Stieltjes transform for the \textsc{ESD} is given by: 
 \begin{equation}\label{eq:lim-samp-Stieltjes}
 	\tr\{(\hat{\Sigma}- z \II_p)^{-1}/p\} \rightarrow_{a.s.} m_{F_\gamma}(z). 
 \end{equation}
For a matrix $\bX\in \RR^{n\times p}$ generated following Assumption~\ref{as:RMT}, we define the Stieltjes transform of the limiting spectral distribution of $(\bX-\EE\bX)(\bX-\EE\bX)^{\top}/n$ as $v_{F_{\gamma}}(z)$, called companion Stieltjes transform. For all $z \in \CC \ \backslash \ \RR^+$, the Stieltjes transform $v_{F_\gamma}(z)$ is related to $m_{F_\gamma}(z)$ by
\begin{equation}\label{eq:def-comp-Stielt}
	 \gamma \left[m_{F_\gamma}(z) + \frac{1}{z}\right] 
 = v_{F_\gamma}(z) + \frac{1}{z}.
\end{equation}
 In addition, we denote by $m'_F(-\lambda)$ the derivative of the Stieltjes transform $m_F(z)$ evaluated at $z = -\lambda$, where 
 \begin{equation}\label{eq:def-deriv-Steilt}
 	m'_{F_\gamma}(z) = \int_{l = 0}^\infty \frac{d G(l)}{(l-z)^2} \ \ \ v'_{F_\gamma}(z) = \gamma \left(m'_{F_\gamma}(z) - \frac{1}{z^2}\right) + \frac{1}{z^2}.
 \end{equation}
 In terms of the empirical quantities,
 \[\tr\{(\hat{\Sigma}- z \II_p)^{-2}/p\} \rightarrow_{a.s.} m'_{F_\gamma}(z). \]
 These convergences form the bases on which we develop the limiting error rate and the limiting optimal weights according to criteria \eqref{eq:CrEst} and \eqref{eq:CrPred}.


\section{Asymptotic Analysis of Weights and Classification Errors}\label{sec:asymp-exp}
In this section, we present the expressions for the limiting prediction erorrs and the optimal weights. All formulae are compared with simulated data, and in Appendix B, they are demonstrated to be accurate even under small data sizes.
\subsection{Classification Error}
Under the two-class Gaussian classification model (Assumption~\ref{as:TCG}), the expected test error of the linear classifier TL-RDA under weight $\bw$ in target population $K$ can be written as 
\[Err(\bw) = \pi_{K, -} \Phi\left(\frac{(\hat{d}(\bw))^{\top} \mu_{-1,K}+ \hat{b}_K}{\sqrt{(\hat{d}(\bw))^{\top} \Sigma (\hat{d}(\bw))}}\right)+\pi_{K, +} \Phi\left(-\frac{(\hat{d}(\bw))^{\top} \mu_{1,K}+ \hat{b}_K}{\sqrt{(\hat{d}(\bw))^{\top} \Sigma (\hat{d}(\bw))}}\right)\]
\[\hat{d}(\bw) := \sum_{k=1}^{K} w_k \hat{d}_k \]
where $\hat{d}(\bw)$ is the transfer learning discriminating vector. A simple proof of this formula is given in the Appendix. In this paper, we assume the balanced class such that $\pi_+ = \pi_-$. The proof techniques generalize immediately to unbalanced cases, which is discussed briefly in Appendix C. The limiting form of $Err(\bw)$ is given by Theorem~\ref{th:tl-rda-ErrorAsp} below.  Two technical assumptions are required so $Z_k$ and the spectrum of $\Sigma$ are well-behaved.
\begin{assumption}[Bounded Moment] \label{as:moment}	
	Assume for each natural number $p$, the entries of $Z$ written by $(Z_k)_{ij}$ has uniformly bounded $p$-th moment. That is, there are constants $C_p$ such that
	\[\EE|(Z_k)_{ij}|^p \leq C_p.\]
\end{assumption}
\begin{assumption}[Anisotropic Local Laws] \label{as:aniso}
	Define the cumulative distribution function of $H$ as $F(x)= \sum_{i = 1}^p  I(\lambda_i \leq x)/p$, where $I(\cdot)$ is the indicator function. Recall $\lambda_1 \geq \lambda_2 \geq \cdots \geq \lambda_p \geq 0$ are the eigenvalues of $\Sigma$.	For a arbitrarily small positive constant $\tau >0$, we assume 
	\[F(\tau) \leq 1 - \tau.\]
\end{assumption}

 We are now in position to state the first main result of this section. The following theorem describes the asymptotic classification error for the regularized transfer learning classifier defined in \eqref{eq:tl-rda}. The result holds under the aforementioned model assumptions.

\begin{theorem}[Asymptotic Classification Error for TL-RDA] \label{th:tl-rda-ErrorAsp}  Suppose that assumptions \ref{as:TCG}-\ref{as:CCW} as well as \ref{as:moment} and \ref{as:aniso} hold.  Then for a fixed $K\ge 2$, as $n_k, p \rightarrow \infty, p / n_k \rightarrow \gamma_k\in (0, \infty]$ for $1\le k\le K$, we have that the limiting form of $Err(\bw)$ for a given weight vector $\bw\in \RR^K$ is given by:
	\begin{align} \label{limError}
		\Phi\left(- \frac{\bu^{\top} \bw }{\sqrt{\bw^{\top} \mathcal{A} \bw}}\right)
	\end{align}
where the elements of $\bu\in \RR^K$ and $\calA\in \RR^{K\times K}$ are as follows:
\begin{align*}
	u_k = \rho_{k K} \alpha_k \alpha_K  m_{F_{\gamma_k}} (-\lambda_k) \quad
	\text{for }k=1,\dots,K,
\end{align*}
and
\begin{align*}
	\mathcal{A}_{kk'} 
	= 
	\begin{cases}
		\alpha_k^2  \left[\frac{v_{F_{\gamma_k}}(-\lambda_k) - \lambda_k v'_{F_{\gamma_k}}(-\lambda_k) }{\gamma_k [\lambda_k v_{F_{\gamma_k}}(-\lambda_k) ]^2}\right] 
		+ \frac{v_{F_{\gamma_k}}' (-\lambda_{k}) - v_{F_{\gamma_k}}^2(-\lambda_{k})}{\lambda_{k}^2 v_{F_{\gamma_k}}^4(-\lambda_{k})}
		\quad
		&\text{if }k=k',\\
		 \rho_{k k'} \alpha_k \alpha_{k'} \mathcal{M}_{k k'}
		 \quad
		 &\text{otherwise,}
	\end{cases}	
\end{align*}
for $k, k' = 1,\cdots , K$. Here $m_{F_{\gamma_k}}(\cdot)$, $v_{F_{\gamma_k}}(\cdot)$ and $v_{F_{\gamma_k}}'(\cdot)$ are the Stieltjes transform, the companion Stieltjes transform and its derivative respectively, as defined in \eqref{eq:def-Stieltjes}-\eqref{eq:def-deriv-Steilt}. Moreover for $k, k' = 1,\cdots , K$ let $\calM_{kk'}$ be the constants defined as the following limiting quantities:
\[\tr[(\hat{\Sigma}_k + \lambda_{k} \II_p)^{-1} (\hat{\Sigma}_{k'} + \lambda_{k'} \II_p)^{-1}\Sigma] / p \rightarrow_{a.s.} \mathcal{M}_{k k'}.\]	 
\end{theorem}
\begin{remark}
The limits of the off-diagonal terms in $\mathcal{A}$ depends on the interplay of sample covariances in separated populations. Similar terms also appear later in the paper when finding the optimal weight vector (see Theorems \ref{th:tl-rda-MinEst}, \ref{th:tl-rda-MinPred}). Depending on the specific scenario, the term $\mathcal{M}_{k k'}$ has explicit expressions, some of which are presented separately in Lemma~\ref{lem:predijasp}.
\end{remark}

We now consider the classification error of the transfer learning classifier, when using the pooled covariance matrix. Let us replace $\hat{d}(\bw)$ with $\hat{d}_P(\bw)$ in $Err(\bw)$ and define this new function of $\bw$ as $Err_P(\bw)$. This gives us  the classification error for TLP-RDA given a weight vector $\bw$.

\begin{corollary}[Asymptotic Classification Error for TLP-RDA]
	Under the same set up as Theorem~\ref{th:tl-rda-ErrorAsp}, we further define $p / \sum_{k = 1}^{K} n_k \rightarrow \bar{\gamma}$. Then assuming $\lambda_1 = \cdots = \lambda_{K} = \lambda$, we have 
\begin{align*} 
		Err_P(\bw) = \Phi\left(- \frac{\bu^\top_P \bw }
		{\sqrt{\bw^{\top} \mathcal{A}_P \bw}}\right)
\end{align*}
where the elements of $\bu_P\in \RR^K$ and $\calA_P\in \RR^{K\times K}$ are as follows:
\[
(u_P)_k = \rho_{k K} \alpha_k \alpha_K  m_{F_{\bar{\gamma}}} (-\lambda) 
\quad
\text{for }k=1,\dots,K,
\]
and
\begin{align*}
(\mathcal{A}_P)_{kk'} 
= 
\begin{cases}
	\alpha_k^2  \left[\frac{v_{F_{\bar\gamma}}(-\lambda) - \lambda v'_{F_{\bar\gamma}}(-\lambda) }{\bar{\gamma} [\lambda v_{F_{\bar\gamma}}(-\lambda) ]^2}\right] 
	+ \gamma_k 
	\left[\frac{v_{F_{\bar\gamma}}' (-\lambda_{k}) - v_{F_{\bar\gamma}}^2(-\lambda)}{\bar{\gamma} \lambda^2 v_{F_{\bar\gamma}}^4(-\lambda)}\right]
	\quad
	&\text{if }k=k',\\
	\rho_{k  k'} \alpha_k \alpha_{k'}  \left[\frac{v_{F_{\bar\gamma}}(-\lambda) - \lambda v'_{F_{\bar\gamma}}(-\lambda) }{\bar{\gamma} [\lambda v_{F_{\bar\gamma}}(-\lambda) ]^2}\right] 
	\quad
	&\text{otherwise,}
\end{cases}	
\end{align*}
for $k, k' = 1,\cdots , K$. Here $m_{F_{\bar{\gamma}}}(\cdot)$, $v_{F_{\bar{\gamma}}}(\cdot)$ and $v_{F_{\bar{\gamma}}}'(\cdot)$ are the Stieltjes transform, the companion Stieltjes transform and its derivative respectively, as defined in \eqref{eq:def-Stieltjes}-\eqref{eq:def-deriv-Steilt}. 
\end{corollary}

In this corollary, we make the assumption that all studies have the same degree of penalization, which brings us the simplified cross terms in $\mathcal{A}_P$ and a more amenable expression for $\mathcal{A}_P$ as a whole. This assumption might not be reasonable when the signal strengths $\alpha_{k}^2$ are vastly different, as lighter penalization may be given to population with stronger signal strength. In this case, one can use the general formula in Theorem \ref{th:tl-rda-ErrorAsp}.


\subsection{Minimum Estimation Risk Weight}

We firstly present the way to minimize the coordinate-wise estimation error of $\hat{d}(W)$ with respect to $d_{Bayes}$.

\begin{theorem}[Asymptotic Estimation Error Minimization for TL-RDA] \label{th:tl-rda-MinEst}
	Suppose that assumptions \ref{as:TCG}-\ref{as:CCW} as well as \ref{as:moment} and \ref{as:aniso} hold.  Then for a fixed $K\ge 2$, as $n_k, p \rightarrow \infty, p / n_k \rightarrow \gamma_k\in (0, \infty]$ for $1\le k\le K$, the weight for minimizing the error in estimating the Bayes optimal discriminator $d_{Bayes}$ is given by:
	\[\bw^E := 
	\arg \min_{\bw} 
	\left\Vert d_{Bayes} - 
	\sum_{k = 1}^{K} w_k \hat{d}_{k}\right\Vert_2^2 = (\mathcal{A}^E + \mathcal{R}^E)^{-1} \bu^E,
	\]
	where the elements of $\bu^E\in\RR^K$, $\calA^E\in \RR^{K\times K}$, and $\calR^E\in \RR^{K\times K}$ are:
	\[
	(u^E)_k
	= \rho_{k K} \alpha_k\alpha_K  
	\left[\frac{1}{\lambda_k} \EE(T^{-1})  - m_{F_{\gamma_k}}(-\lambda_k)^2\right]
	\quad
	\text{for }
	k=1,\dots,K
	,
	\]
	\[
	\mathcal{A}^E_{kk'} 
	= 
	\begin{cases}
		\alpha_{k}^2 m_{F_{\gamma_k}}'(-\lambda_{k})
		\quad &\text{if }k=k',\\
		\rho_{k k'} \alpha_k \alpha_{k'} \mathcal{E}_{kk'}
		\quad &\text{otherwise,}
	\end{cases}
	\]
	 and
	 \[	
	 \mathcal{R}^E_{kk'} 
	 =
	 \begin{cases}
	 	\frac{v_{F_{\gamma_k}}(-\lambda_{k}) - \lambda_k v'_{F_{\gamma_k}}(-\lambda_k)}{\lambda_{k} v_{F_{\gamma_k}}(-\lambda_{k})^2}  
	 	\quad &\text{if }k=k',
	 	\\
	 	0 \quad &\text{otherwise,}
	 \end{cases}
	 \]
	 for $k, k' = 1,\cdots , K$. Here $m_{F_{\gamma_k}}(\cdot)$, $v_{F_{\gamma_k}}(\cdot)$ and $v_{F_{\gamma_k}}'(\cdot)$ are the Stieltjes transform, the companion Stieltjes transform and its derivative respectively, as defined in \eqref{eq:def-Stieltjes}-\eqref{eq:def-deriv-Steilt}, while $T$ is the limiting spectral distribution of the population covariance matrix $\Sigma$. Moreover for $k, k' = 1,\cdots , K$, let $\calE_{kk'}$ be the constants defined as the following limiting quantities:
	 \[\tr[(\hat{\Sigma}_k + \lambda_{k} \II_p)^{-1} (\hat{\Sigma}_{k'} + \lambda_{k'} \II_p)^{-1}] / p \rightarrow_{a.s.} \mathcal{E}_{k k'}.\]	 
\end{theorem}
Just like the expressions for the asymptotic classification rate in Theorem~\ref{th:tl-rda-ErrorAsp}, more explicit forms of $\calE_{kk'}$ are available if $\gamma_k$ and $\lambda_{k}$ are the same across populations. We call $\bw^E$ the minimum estimation weight, as it minimizes the $\ell_2$ error in estimating the Bayes optimal discriminating vector $d_{Bayes}$. In order to better illustrate the effect of transfer learning, we consider next a simpler situation where all the sources are homogeneous, in that $\gamma_1=\gamma_2=\dots=\gamma_K$ and the corresponding regularization parameters are also the same, i.e., $\lambda_1=\lambda_2=\dots=\lambda_K$. We further assume equal correlation among all sources, i.e., $\rho_{kk'}=\rho$ whenever $k\neq k'$.

\begin{corollary}[Estimation Error for Homogeneous Sources]\label{cor:homog-corr-est} Under the setup of Theorem~\ref{th:tl-rda-MinEst}, suppose $\gamma_1=\gamma_2=\dots=\gamma_K=:\gamma$, $\lambda_1=\lambda_2=\dots=\lambda_K=:\lambda$, and $\rho_{kk'}=\rho$ whenever $k\neq k'$, $1\le k, k'\le K$. Then the weight vector $\bw^E$ that minimizes the asymptotic error in estimating $d^{\rm Bayes}$ is given by:
	\begin{align*}
	\bw^E
	=&~
	\alpha_K
	\left[\frac{1}{\lambda} \EE(T^{-1})  - m_{F_{\gamma}}(-\lambda)^2\right]
	{\rm vec}
	\bigg( 
	\dfrac{\{\rho-\xi+(1-\rho)I(k=K)\}\alpha_k }{t_{m,\eps,\lambda}\alpha_k^2
		+
		t_{v,\lambda}}
	:1\le k\le K
	\bigg)
	\end{align*}
	where $\eps=\calE_{12}$, $t_{\rho,\eps,\lambda}:=-\rho\eps +m'_{F_\gamma}(-\lambda)$, $t_{v,\lambda}:=\frac{v_{F_{\gamma}}(-\lambda) - \lambda v'_{F_{\gamma}}(-\lambda)}{\lambda v_{F_{\gamma}}(-\lambda)^2}$ and 
	\[
	\xi=
		\dfrac{\rho\eps \sum_{k=1}^K\{\rho+(1-\rho)I(k=K)\}\frac{\alpha_k^2}{t_{m,\eps,\lambda}\alpha_k^2
			+
			t_{v,\lambda}} }{
		1+\rho\eps\displaystyle \sum_{k=1}^K\frac{\alpha_k^2}{t_{m,\eps,\lambda}\alpha_k^2
			+
			t_{v,\lambda}}}.
	\]
	
\end{corollary}

The above corollary shows that when all sources have equal sample sizes, and equal correlation, the optimal weight is proportional to $\alpha_k\alpha_K/(t_{m,\eps,\lambda}\alpha_k^2+t_{v,\lambda})$. In particular, if $\alpha_k\to 0$ for some source population, then the corresponding weight for that source increases at the rate of $1/\alpha_k$. This is intuitive since in the homogeneous setting of equal sample sizes, the signal strength is completely determined by the inverse of the variance of $\delta_k$, which is precisely $1/\alpha_k$. For practical usage of the optimal transfer weights, we now turn to the estimation of the above quantities. 

\begin{remark}[Estimating optimal weights]
	The Marchenko-Pastur law, along with equations~\eqref{eq:def-Stieltjes}, \eqref{eq:lim-samp-Stieltjes}, \eqref{eq:def-comp-Stielt}, and \eqref{eq:def-deriv-Steilt} imply that the sample Stieltjes transform $\tr[(\hat{\Sigma}_k-z\II_p)^{-1}]$ can be suitably used to estimate the functions $m_{F_{\gamma_k}}(\cdot)$, $v_{F_{\gamma_k}}(\cdot)$, and $m'_{F_{\gamma_k}}(\cdot)$ reliably. We remind the reader that the estimation of $\alpha_k^2$ and $\rho_{kk'}$ is tackled in Appendix~\ref{sec:estim-alpha-rho}. Finally, we note that the expression of $\bu^E$ in all the above cases involves the expectation of $T^{-1}$, and we discuss a consistent estimator for $\EE(T^{-1})$ in the following proposition, the proof of which is immediate from Corollary 4.2 of \cite{dobriban2021distributed}.
\end{remark}

The following proposition holds under assumptions~\ref{as:RMT}.
\begin{proposition} \label{pr:mean-invT} When $\gamma_k < 1$,
	we have $\tr (\hat{\Sigma}^{-1}_k)/p \rightarrow_{a.s.} \EE(T^{-1}) / (1 - \gamma_k)$.
\end{proposition}
Since in our developments so far we have assumed all populations to share the same covariance matrix $\Sigma$, one can always use the pooled sample covariance in the estimation scheme above. That means we can reliably estimate $\EE(T^{-1})$ as long as $\bar \gamma < 1$. We then present the optimal estimation for TLP-RDA, i.e., the transfer learning discriminant analysis done using the pooled covariance matrix.
\begin{corollary}[Asymptotic Estimation Error Minimization for TLP-RDA]
Under the same set up as Theorem~\ref{th:tl-rda-MinEst}, assuming $\lambda_1 = \cdots = \lambda_{K} = \lambda$, we have 
\[\bw^E_P := \arg \min_{\bw} 
\left\Vert 
d_{Bayes} - \sum_{k = 1}^{K} w_k \hat{d}^P_{k}
\right\Vert_2^2 = (\mathcal{A}^E_P + \mathcal{R}_P^E)^{-1} \bu^E_P,
\]
where the elements of $\bu^E_P\in\RR^K$, $\calA^E_P\in \RR^{K\times K}$, and $\calR^E_P\in \RR^{K\times K}$ are:
\[
(u^E_P)_k = \rho_{k K} \alpha_k\alpha_K  \left[\frac{1}{\lambda} \EE(T^{-1})  - m_{F_{\bar\gamma}}(-\lambda)^2\right]
\quad 
\text{for }
k=1,\dots,K,
\]
while
\[	(\mathcal{A}^E_P)_{kk'} = \rho_{k k'} \alpha_k \alpha_{k'} m'_{F_{\bar\gamma}}(-\lambda) \]
and
\[
(\mathcal{R}^E_P)_{kk'} =  
\begin{cases}
	\gamma_k \frac{v_{F_{\bar\gamma}}(-\lambda) - \lambda v'_{F_{\bar\gamma}}(-\lambda)}{\bar{\gamma}\lambda v_{F_{\bar\gamma}}(-\lambda)^2} \quad &\text{if }k=k',\\
	0 \quad &\text{otherwise,}
\end{cases}
\]
for $k, k'= 1, \cdots, K$.
\end{corollary}

Once again we can follow the remark and proposition above for estimating the weight that minimizes the estimation error when using the pooled covariance matrix.

\subsection{Minimum Prediction Risk Weight}
In the high dimensional regime considered in this paper, a weight minimizing the estimation error does not translate into maximizing the classification score $\hat{d}(\bw)^{\top} x_0$. We can as well compute the optimal prediction weight by directly minimizing the difference between the TL-RDA classification score and the Bayes classification score. This is precisely the objective of the next theorem.
\begin{theorem}[Asymptotic Prediction Error Minimization for TL-RDA] \label{th:tl-rda-MinPred}
	Suppose that Assumptions \ref{as:TCG}-\ref{as:CCW} as well as \ref{as:moment} and \ref{as:aniso} hold.  Then for a fixed $K\ge 2$, as $n_k, p \rightarrow \infty, p / n_k \rightarrow \gamma_k\in (0, \infty]$ for $1\le k\le K$, the weight for minimizing the excess risk, i.e., the error in predicting the class at a random test point $x_0$, when compared to the Bayes optimal discriminator $d_{Bayes}$, is given by:
	\[\bw^P := 
	 \arg \min_{\bw} \EE_{x_0} \left[(d_{Bayes} - \sum_{k = 1}^{K} w_k \hat{d}_{k})^{\top} x_0\right]^2 = 
	 (\mathcal{A}^P + \mathcal{R}^P)^{-1} \bu^P,
	 \]
	 where the elements of $\bu^P\in\RR^K$, $\calA^P\in \RR^{K\times K}$, and $\calR^P\in \RR^{K\times K}$ are:
	 \[
	 (u^P)_k
	 = \rho_{k K} \alpha_k\alpha_K m_{F_{\gamma_k}}(-\lambda_k)
	 \quad
	 \text{for }
	 k=1,\dots,K
	 ,
	 \]
	 \[
	 \mathcal{A}^P_{kk'} 
	 = 
	 \begin{cases}
	 	\alpha_k^2  \left[\frac{v_{F_{\gamma_k}}(-\lambda_k) - \lambda_k v'_{F_{\gamma_k}}(-\lambda_k) }{\gamma_k [\lambda_k v_{F_{\gamma_k}}(-\lambda_k) ]^2}\right]
	 	\quad &\text{if }k=k',\\
	 	\rho_{k k'} \alpha_k \alpha_{k'}  \mathcal{M}_{k k'} 
	 	\quad &\text{otherwise,}
	 \end{cases}
	 \]
	 and
	 \[	
	 \mathcal{R}^E_{kk'} 
	 =
	 \begin{cases}
	 	\frac{v'_k(-\lambda_{k}) - v_k^2(-\lambda_k)}{\lambda_k^2 v^4_k(-\lambda_{k})}
	 	\quad &\text{if }k=k'
	 	\\
	 	0 \quad &\text{otherwise,}
	 \end{cases}
	 \]
	 for $k, k' = 1,\cdots , K$. Here $m_{F_{\gamma_k}}(\cdot)$, $v_{F_{\gamma_k}}(\cdot)$ and $v_{F_{\gamma_k}}'(\cdot)$ are the Stieltjes transform, the companion Stieltjes transform and its derivative respectively, as defined in \eqref{eq:def-Stieltjes}-\eqref{eq:def-deriv-Steilt}. Moreover for $k, k' = 1,\cdots , K$, let $\calM_{kk'}$ be the constants defined as the following limiting quantities:
	 \[\tr[(\hat{\Sigma}_k + \lambda_{k} \II_p)^{-1} (\hat{\Sigma}_{k'} + \lambda_{k'} \II_p)^{-1}\Sigma] / p \rightarrow_{a.s.} \mathcal{M}_{k k'}.\]
\end{theorem}

As in the case of estimation error weights, we discuss the issue of estimating the optimal weight that minimizes prediction error in this setting. Note that unlike the estimation error minimizing weight $\mathcal{W}^E$, the prediction risk minimizing weight $\mathcal{W}^P$, involves no population spectral distribution. It is thus estimable in all cases, including the case when $\gamma >1$. On the other hand, estimating $\mathcal{M}_{k k'}$ requires more care, especially in the case $k=K$ or $k'=K$. The following proposition describes the estimation in this case.

\begin{proposition}\label{pr:heterog-cons-est-mkk} We have the following consistent estimators $\hat{\calM}_{kk'}$ for $\calM_{kk'}$ separately for three cases:
	\[
	\hat{\calM}_{kk'}
	=
	\begin{cases}
		\tr[(\hat{\Sigma}_k + \lambda_{k} \II_p)^{-1} (\hat{\Sigma}_{k'} + \lambda_{k'} \II_p)^{-1} \hat \Sigma_K] / p\quad &\text{if }k\neq k', k\neq K, k'\neq K\\
		\tr[(\hat{\Sigma}_k + \lambda_{k} \II_p)^{-2} \hat \Sigma_K] / p
		\quad &\text{if }k=k'\neq K\\
		\frac{1}{px_p}
		\tr[(\hat{\Sigma}_{k'} + \lambda_{k'} \II_p)^{-1}]
		-
		\frac{\lambda_K}{px_p} 
		\tr[(\hat{\Sigma}_K +\lambda_K \II_p )^{-1} (\hat{\Sigma}_{k'} + \lambda_{k'} \II_p)^{-1}]
		\quad &\text{if }k=K,\,k'\neq K
	\end{cases}
	\]
	where $x_p=x(\gamma_K,\lambda_K)$ is the solution to the equation:
	\[
	1-x_p=\gamma_K
	\left
	[1-\lambda_K\int(x_pt+\lambda_K)^{-1}dH_K(t)
	\right]
	\]
	Then 
	\[\calM_{kk'} -\hat{\calM}_{kk'}  \rightarrow_{a.s.} 0.\]
\end{proposition}

In order to better illustrate the effect of transfer learning, we consider next a simpler situation where all the sources are homogeneous, in that $\gamma_1=\gamma_2=\dots=\gamma_K$ and the corresponding regularization parameters are also the same, i.e., $\lambda_1=\lambda_2=\dots=\lambda_K$. We further assume equal correlation among all sources, i.e., $\rho_{kk'}=\rho$ whenever $k\neq k'$.

\begin{corollary}[Prediction Error for Homogeneous Sources]\label{cor:homog-corr-pred} Under the setup of Theorem~\ref{th:tl-rda-MinEst}, suppose $\gamma_1=\gamma_2=\dots=\gamma_K=:\gamma$, $\lambda_1=\lambda_2=\dots=\lambda_K=:\lambda$, and $\rho_{kk'}=\rho$ whenever $k\neq k'$, $1\le k, k'\le K$. Then the weight vector $\bw^P$ that minimizes the asymptotic excess risk is given by:
	\begin{align*}
		\bw^P
		=&~
		\alpha_K
		m_{F_{\gamma}}(-\lambda)
		{\rm vec}
		\bigg( 
		\dfrac{\{\rho-\xi^P+(1-\rho)I(k=K)\}\alpha_k }{t^P_{v,\rho,\lambda}\alpha_k^2
			+
			t^P_{v,\lambda}}
		:1\le k\le K
		\bigg)
	\end{align*}
	where $m=\calM_{12}$, $t^P_{v,\rho,\lambda}:=-m\rho +\frac{v_{F_{\gamma}}(-\lambda) - \lambda v'_{F_{\gamma}}(-\lambda) }{\gamma [\lambda v_{F_{\gamma}}(-\lambda) ]^2}$, $t^P_{v,\lambda}:=\frac{v'(-\lambda) - v^2(-\lambda)}{\lambda^2 v^4(-\lambda)}$ and 
	\[
	\xi^P=
	\dfrac{m\rho \sum_{k=1}^K\{\rho+(1-\rho)I(k=K)\}\frac{\alpha_k^2}{t^P_{v,\rho,\lambda}\alpha_k^2
			+
			t^P_{v,\lambda}} }{
		1+m\rho\displaystyle \sum_{k=1}^K\frac{\alpha_k^2}{t^P_{v,\rho,\lambda}\alpha_k^2
			+
			t^P_{v,\lambda}}}.
	\]
	
\end{corollary}

The above corollary shows that when all sources have equal sample sizes, and equal correlation, the optimal weight is proportional to $\alpha_k\alpha_K/(t^P_{v,\rho,\lambda}\alpha_k^2
+
t^P_{v,\lambda})$. Thus similar to Corollary~\ref{cor:homog-corr-est}, if $\alpha_k\to 0$ for some source population, then the corresponding weight for that source increases at the rate of $1/\alpha_k$. This is intuitive since in the homogeneous setting of equal sample sizes, the signal strength is completely determined by the inverse of the variance of $\delta_k$, which is precisely $1/\alpha_k$. Finally, since our assumption so far posits the same variance $\Sigma$ for all populations, we also describe the estimation based on the pooled sample covariance matrix. This is given in the following corollary for the pooled covariance based classifier TLP-RDA.
\begin{corollary}[Asymptotic Prediction Error Minimization for TLP-RDA]
	Under the same set up as theorem \ref{th:tl-rda-MinPred}, assume $\lambda_1 = \cdots = \lambda_{K} = \lambda$, we have
	\[\bw^P_P := \arg \min_{\bw}
	\EE_{x_0} 
	\left[(d_{Bayes} - \sum_{k = 1}^{K} w_k \hat{d}^P_{k})^{\top} x_0\right]^2 = 
	(\mathcal{A}^P_P + \mathcal{R}^P_P)^{-1} \bu^P_P\]
	where the elements of $\bu^P_P\in\RR^K$, $\calA^P_P\in \RR^{K\times K}$, and $\calR^P_P\in \RR^{K\times K}$ are:
	\[
	(u^P_P)_k = \rho_{k K} \alpha_k\alpha_K   m_{F_{\bar\gamma}}(-\lambda) 
	\quad 
	\text{for }
	k=1,\dots,K,
	\]
	while
	\[	(\mathcal{A}^P_P)_{kk'} =
	\rho_{k k'} \alpha_k \alpha_{k'} \left[\frac{v_{F_{\bar\gamma}}(-\lambda) - \lambda v'_{F_{\bar\gamma}}(-\lambda) }{\bar{\gamma} [\lambda v_{F_{\bar\gamma}}(-\lambda) ]^2}\right]
	\]
	and
	\[
	(\mathcal{R}^P_P)_{kk'} =  
	\begin{cases}
		\gamma_k \frac{v_{F_{\bar\gamma}}' (-\lambda_{k}) - v_{F_{\bar\gamma}}^2(-\lambda)}{\bar{\gamma} \lambda^2 v_{F_{\bar\gamma}}^4(-\lambda)}
		 \quad &\text{if }k=k',\\
		0 \quad &\text{otherwise,}
	\end{cases}
	\]
	for $k, k'= 1, \cdots, K$.
\end{corollary}

The optimal prediction weights $\bw^E_P, \bw^P_P$ dominate $\bw^E, \bw^P$ respectively in classification error, meaning 
\[Err(\bw^E) \geq Err(\bw^P),  Err_P(\bw^E_P) \geq Err_P(\bw^P_P).
\]
Note that a priori it is not immediate that the weight which minimizes the difference in prediction error from the Bayes optimal direction, also minimizes the classification error. In particular, the accuracy of classification depends on the sign of $\hat{d}(\bw)^{\top}x_0$, and not the actual value. However, we now show that in addition to optimizing the score $\hat{d}(\bw)^{\top}x_0$, the prediction weights $\bw^P$ and $\bw^P_P$ also minimize the misclassification error. This is formalized by the following proposition. 
\begin{proposition} \label{pr:PredIsOpt}
	The optimal prediction weight minimizes the testing data classification error.
	\begin{align*}
		\bw^P &= \arg \min_{\bw} 
		\EE_{x_0, y_0} [I({\rm sign}[(\hat{d}(\bw))^{\top}x_0] \neq y_0)] \\
		\bw^P_P &= \arg \min_{\bw}
		\EE_{x_0, y_0} [I({\rm sign}[(\hat{d}_P(\bw))^{\top}x_0] \neq y_0)].
	\end{align*}
\end{proposition}

We conclude this section by validating the existence of the optimal weights. Indeed, the solution to all four types of weights implicitly assume that the matrices $\mathcal{A}^E + \mathcal{R}^E, \mathcal{A}^P + \mathcal{R}^P, \mathcal{A}^E_P + \mathcal{R}^E_P, \mathcal{A}^P_P + \mathcal{R}^P_P$ are invertible. We prove this is always the case in the following proposition. 
\begin{proposition}[Existence of Optimal Weights] \label{pr:pos-def}
The matrices $\mathcal{A}^E + \mathcal{R}^E, \mathcal{A}^P + \mathcal{R}^P, \mathcal{A}^E_P + \mathcal{R}^E_P, \mathcal{A}^P_P + \mathcal{R}^P_P$ are invertible, and hence the limiting optimal weights $\bw^E, \bw^P, \bw^E_P, \bw^P_P$ exist.
\end{proposition}
We compare the error rates of the four transfer learning estimators under several different scenarios in Appendix B.

\subsection{Geometric Interpretation}
We now provide some geometric interpretations on the optimal weights derived in the previous subsections. Let us define the following discriminant directions:
\[d_{est} := \hat{d}(\bw_E),  \ \ \  d_{err} := \hat{d}(\bw_P). \]
Also recall the Bayes discriminant direction $d_{Bayes}$. These discriminant directions are visualized in Figure \ref{fig:WIllu}. The blue plane is the space spanned by linear combinations of the local discriminant directions $\{\hat{d}_k\}$, and the top black line stands for the Bayes direction, which is not necessarily in the linear span of $\{\hat{d}_k\}$. Since both the TL-RDA directions $d_{est}$ and $d_{err}$ are linear combinations of $\hat{d}_k$, they both lie on the blue plane and are denoted by colored lines. The direction obtained by minimizing the error in estimating $d_{Bayes}$ is given by $d_{est}$, which is the projection of $d_{Bayes}$ onto the blue plane, denoting ${\rm span}\{\hat{d}_k\}$. The criterion \ref{eq:CrEst} suggests that $d_{est}$ is the minimizer of the OLS loss when fitting $d_{Bayes}$ with linear combinations of $d_k$. 

Denoting the angle between $d_{Bayes}$ and $d_{err}$ as $\theta$, we now show that the cosine of $\theta$ is directly related to $Err(\bw)$ accounting for the scaling $\Sigma$. Let us define the scaled inner product $\langle a, b \rangle_\Sigma  = a^{\top} \Sigma b$ and the scaled cos angle $\cos_\Sigma(a, b) = \langle a, b \rangle_\Sigma / \sqrt{\langle a, a \rangle_\Sigma \langle b, b \rangle_\Sigma}$. Then we have 
\begin{align*}
	\cos\theta 
	=
	\cos_\Sigma\angle(\hat{d}(\bw), d_{Bayes}) &= \hat{d}(\bw)^\top \delta_{K} / \sqrt{\hat{d}(\bw)^\top\Sigma \hat{d}(\bw) \delta_{K}^\top \Sigma^{-1} \delta_{K}} := \frac{\Theta(\bw)}{\Theta_{Bayes}}
\end{align*}
\[\Theta(\bw):= \frac{\hat{d}(\bw)^\top \delta_{K}}{\sqrt{\hat{d}(\bw)^\top\Sigma \hat{d}(\bw)}}, \ \ \ \Theta_{Bayes} = \sqrt{\delta_{K}^\top \Sigma^{-1} \delta_{K}} \]
Recall that $\Phi(- \Theta(\bw))$ is the classification error rate of TL-RDA (Theorem \ref{th:tl-rda-ErrorAsp}) and $\Phi(-\Theta_{Bayes})$ is the Bayes error rate. This implies that
\[
	\cos\theta 
	=
	\frac{\Theta(\bw)}{\Theta_{Bayes}}
	=
	\frac{\Phi^{-1}(Err(\bw))}{\Phi^{-1}(Err_{Bayes})}.
\]
Since $\Phi^{-1}(\cdot)$ is a monotonically increasing function, it is clear that $\cos(\theta)$ is close to one, i.e., $\theta$ is close to zero, if and only if $Err(\bw)$ is close to the Bayes error $Err_{Bayes}$. That is, the size of $\theta$ directly quantifies the inefficiency of TL-RDA relative to $d_{Bayes}$ in terms of classification error: a smaller $\theta$ is equivalent to a near-optimal transfer combination weight $\bw$. 

Finally, let us also consider an observation $x_0$ pointing in a random direction, and denote the angle between $d_{err}$ and $x_0$ as $\beta$ and the angle between $d_{Bayes}$ and $x_0$ as $\alpha$. Straightforwardly, $d_{err}$ minimizes the difference between $\beta$ and $\alpha$ as $\bw^P$ minimizes the difference between the inner products $d_{err}^{\top} x_0$ and $d_{Bayes}^{\top} x_0$. Proposition \ref{pr:PredIsOpt} suggests $d_{err}$ minimizes $\theta$ simultaneously.
\begin{figure}
	\includegraphics[width= \textwidth]{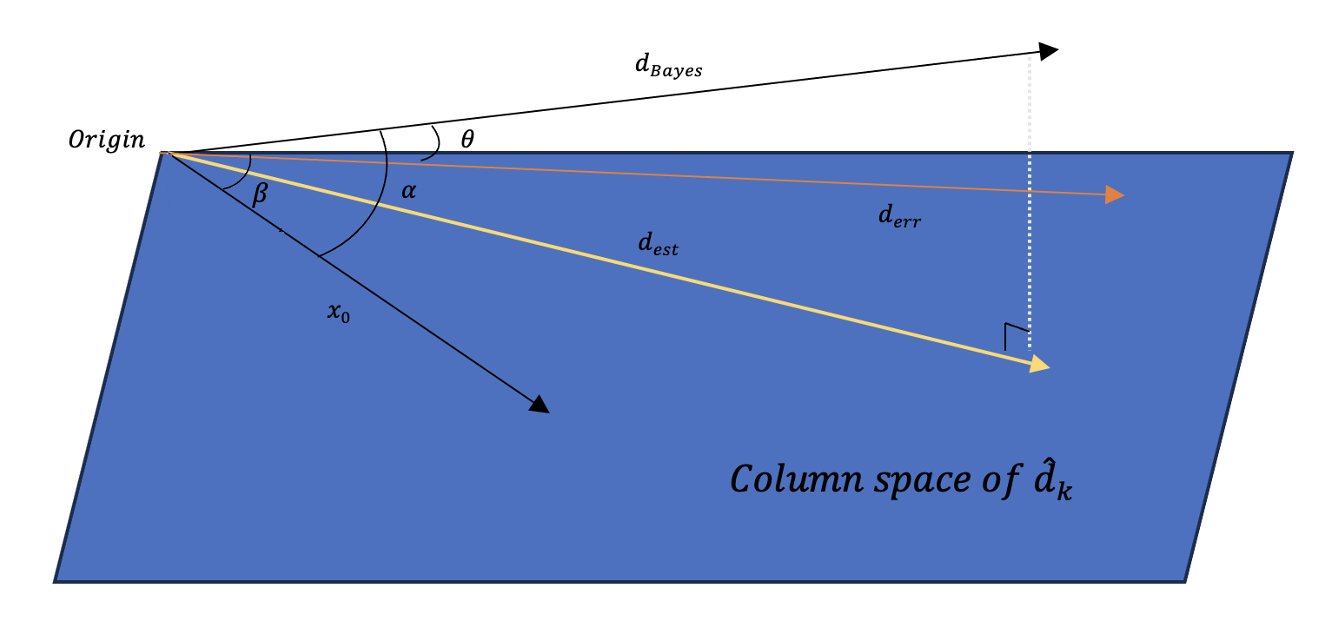}
	\centering
	\caption{Geometric interpretations of optimal weights}
	\label{fig:WIllu}
\end{figure}

\section{Robustness and Weight Selection}\label{sec:weight-select}
In this section, we present some guidance on how one can choose between different weighting schemes. 
As we have repeatedly illustrated in the previous section, the $d_{est}$ is dominated by $d_{err}$ in terms of classification error on unseen data. This claim, however, holds only when the test data distribution is as given by Assumption~\ref{as:TCG}. As we shall demonstrate, $d_{est}$ out performs $d_{err}$ considerably when there is a distribution shift in terms of classification error. The TLP-LDA estimator instead consider the space spanned by $\hat{d}_k^P$, thus, the final TL estimator is in a different but related column space. We will prove, at least in special cases, optimal TLP-RDA outperforms the optimal TL-RDA when the aspect ratio is large enough.

\subsection{Robustness of Optimal Estimation Weight}
In this section we demonstrate the robustness of $d_{est}$ to covariate shifts in test data. This robustness is intuitive as we know the optimal estimation weight only attempts to minimize the difference in TL discriminant direction ($\hat{d}(\bw)$) and the Bayes discriminant direction $d_{Bayes}$. In fact, we can show the weights obtained by minimizing criteria \eqref{eq:CrEst} are equivalent to a conservative solution to the problem of minimizing criteria \eqref{eq:CrPred} when test data distribution is unknown. A similar argument is in \citet{zhang2023transfer}, we summarize it in the following proposition.
\begin{proposition}[Robustness of Estimation Weight]\label{pr:minmax} For the class of covariate distributions given by $\mathcal{P} :=~ \{P: x\sim P, \EE_P(||x||_2) \leq c\}$, we have:
\begin{align*}
	\underset{\bw}{\arg \min} \ 
	\left\Vert 
	d_{Bayes} - \sum_{k = 1}^{K} w_k \hat{d}_k
	\right\Vert^2 
	=&~
	\underset{\bw}{\arg \min} \ \underset{x_0 \in \mathcal{P}}{\max} ~\EE_{x_0} \left [  
	\left(d_{Bayes} -\sum_{k = 1}^{K} w_k \hat{d}_k\right)^{\top} x_0 \right ]^2.
\end{align*}
\end{proposition}
Proposition~\ref{pr:minmax} claims that, when we only know the target data comes from a distribution with bounded expected norm, the safest option to minimize prediction error criteria \eqref{eq:CrPred} is to focus only on the estimation error criteria \eqref{eq:CrEst}. We demonstrate the usefulness of this in Figure~\ref{fig:WIllu-shift}. 
\begin{figure}
	\includegraphics[width=4in]{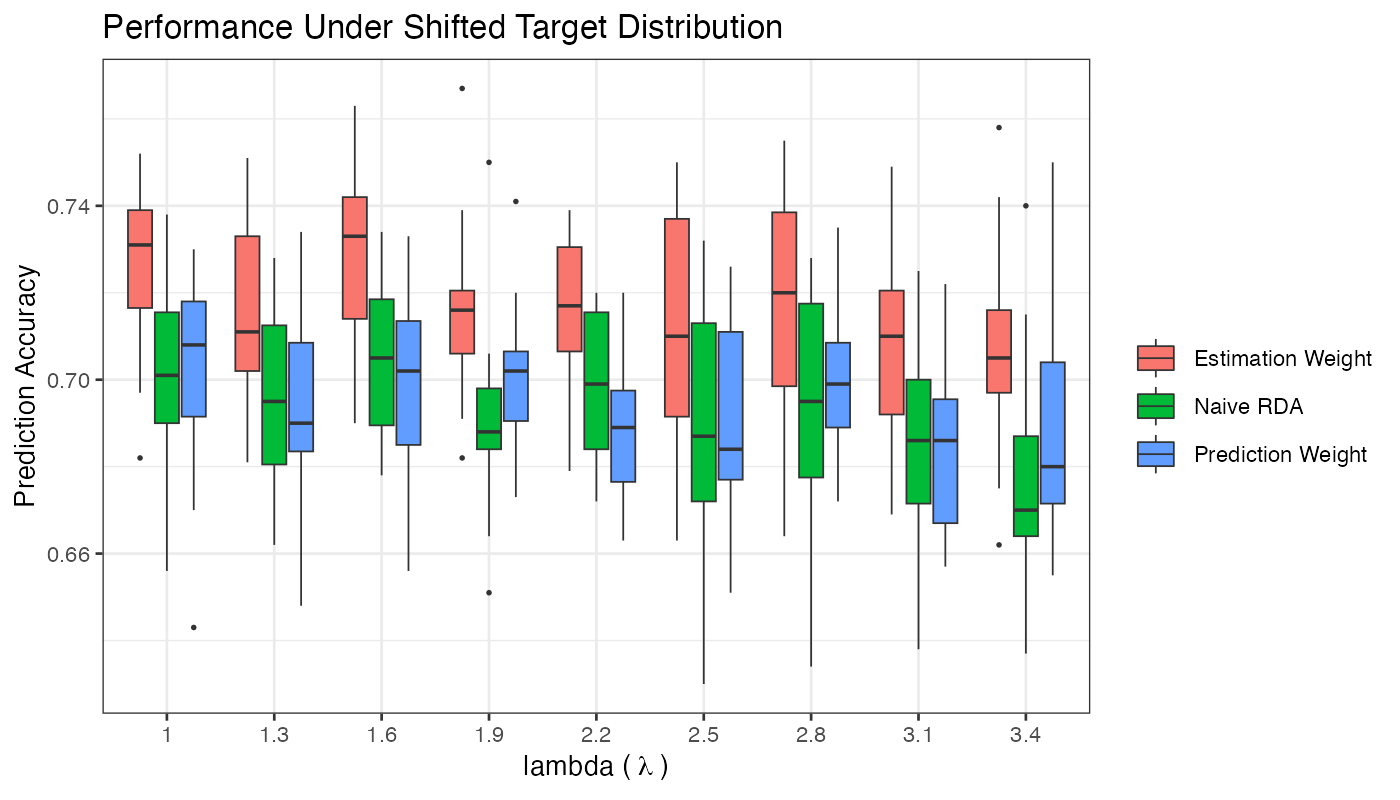}
	\centering
	\caption{TL-RDA with the optimal estimation weight outperforms reugular RDA and TL-RDA with the optimal prediction weight when the target distribution changes.}
	\label{fig:WIllu-shift}
\end{figure}
Here $p = 150$ and $n_k = 250, \cdots, 160$ for study $1$ to study $K$. The pairwise correlations across studies are fixed at $0.5$. We use the same Toeplitz covariance matrix for all training data, however, the eigenvalues of the test data $x_0$ covariance matrix are modified to decay much faster. The testing accuracy of the three methods are shown on the $y$ axis as $\lambda$ changes. One can see the optimal estimation weight consistently outperforms others, suggesting the robustness and conservative nature of the optimal estimation weight boost the performance of TL-RDA when there is a change in testing data distribution.

\subsection{Pooled Sample Covariance and Individual Sample Covariance Matrix}
For TLP-RDA, all discriminant directions uses the same covariance estimate $\hat{\Sigma}_P$. Although $\hat{\Sigma}_P$ is a better covariance estimate than all $\hat{\Sigma}_k$, $\hat{d}_k^P$ are inevitably similar and the columns space of $\hat{d}_k^P$ would be less informational. This essentially becomes a bias-variance trade off. When $\gamma_k$ is small, the estimates $\hat{d}_k$ are reliable already. In this case, individual covariance matrices bring more variances to this column space, therefore, increase the quality of final TL-RDA estimator. When $\gamma_k$ are large, direction estimate based on $\hat{\Sigma}_k$ are no longer reliable  and one should consider the more stable $\hat{\Sigma}_P$. We will formalize this statement in this section also.
\begin{proposition}[TLP-RDA out performs TL-RDA when $\gamma$ is large]\label{pr:TLPwins}
Assume $\Sigma = \II_p$, $\gamma_1 = \cdots = \gamma_K=\gamma$, $\lambda_1, \cdots, \lambda_{K}=\lambda =r\left(\gamma -\frac{1}{r+1}\right)$ for some fixed $r>(1-\gamma)_+/\gamma$. For the pooled covariance matrix, we choose $\lambda'=r'\left(\gamma/K -\frac{1}{r'+1}\right)$ for some $r'>(K-\gamma)_+/\gamma$.

\begin{enumerate}
	\item When $\rho=1$, $Err(\bw^P) \geq Err_P(\bw^P_P)$, if and only if 
	\begin{align*}
		\gamma^2[(1+r')^2
		- K(1+r)^2]
		\ge&~
		K
		\left([\gamma(1+r)^2-1]\sum_k\alpha_k^2 \right).
	\end{align*}

	\item When $\rho=0$, $Err(\bw^P) \geq Err_P(\bw^P_P)$, if and only if
	\[
	\gamma[(1+r')^2-(1+r)^2]\ge K-1.
	\]
\end{enumerate}
\end{proposition}
We also numerically demonstrate this phenomenon by plugging values into the limiting expressions under a more general set up (Figure \ref{fig:TLPwins}). We again use $\rho = 0.5$ while changing the other parameters, including the number of auxiliary studies $K$, the decay rate of the eigenvalues of $\Sigma$ and the signal strength $\alpha$. We can see that TLP-RDA can outperform TL-RDA in all cases as $\gamma_K$ grows, as the error rates of TLP-RDA decrease at slower rates. The transition point $\gamma^*$, defined as the $\gamma_K$ when TLP-RDA outperforms TL-RDA,  differs in different scenarios. One can see that it decreases when the eigenvalues of $\Sigma$ decrease slower, and interestingly, also when the number of auxiliary population decreases.
\begin{figure}[h]
		\includegraphics[width=4in]{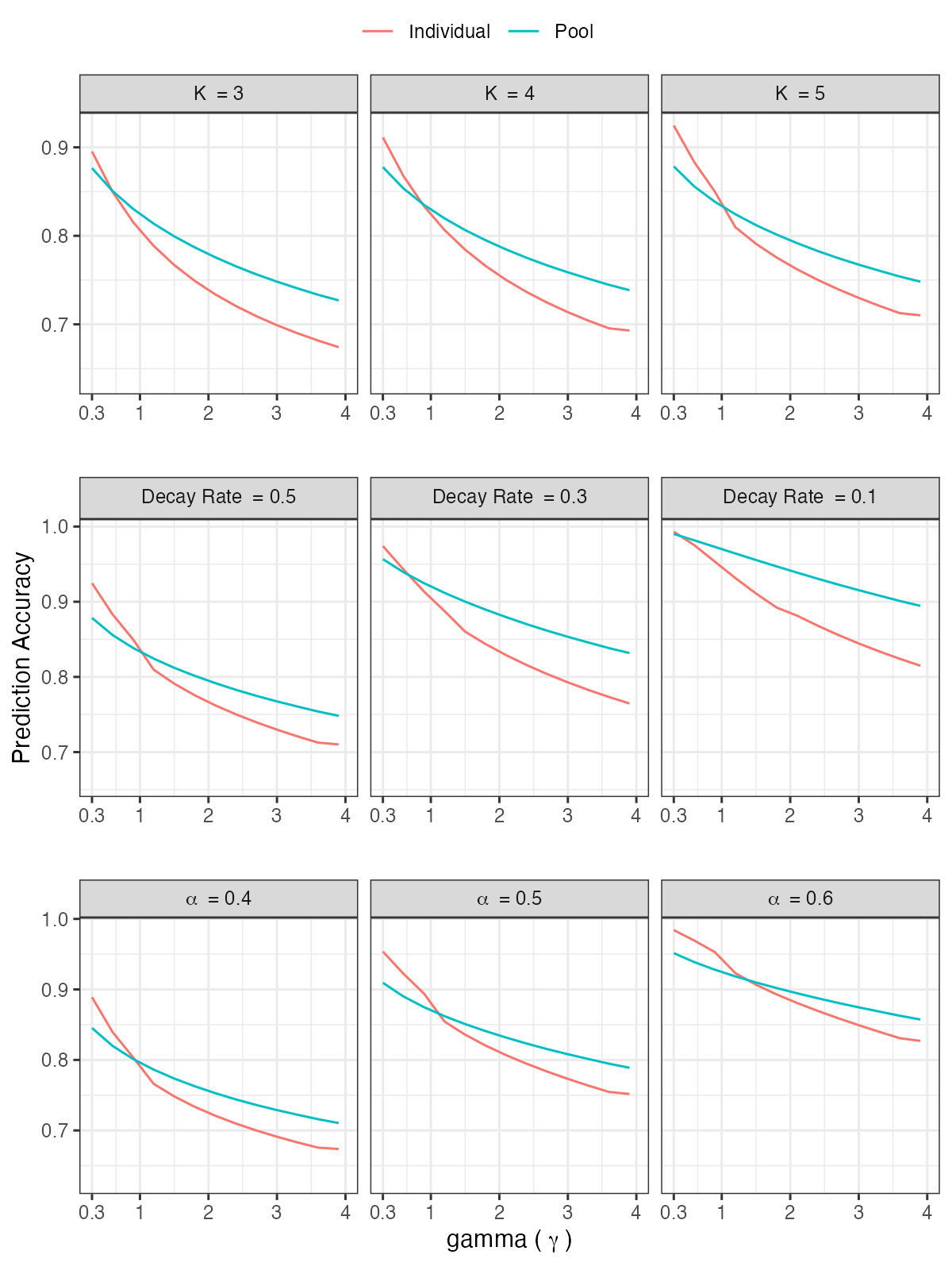}
	\centering
	\caption{TLP-RDA outperforms TL-RDA when $\gamma$ is large under general set ups.}
	\label{fig:TLPwins}
\end{figure}

\section{Heterogeneous Population Covariance Matrix}
\label{sec:hetero-cov}

The previous discussions of TL-RDA have been restricted to the case where observed covariates $X_k$ in all populations $k = 1, \cdots, K$ share the same covariance matrix $\Sigma$. In this section we extend TL-RDA to accommodate the scenario when covariance matrices are different. We call this generalization beyond identical matrices as transfer-learning-heterogeneous (TLH)-RDA. We firstly formalize this set up.
\begin{assumption}[Heterogeneous Two-class Gaussian]\label{as:HTCG}
	We assume all populations $k = 1,\cdots, K$ follow the classic two-class Gaussian mixture model. More specifically, for $i=1,\dots,n_k$ and $k=1,\dots,K$,
	\begin{align}
		(y_k)_i \in \{-1, +1\} \ \ \
		\PP((y_k)_i = \pm 1) = \pi_{\pm 1} \ \ \
		(X_k)_i|(y_k)_i \sim N(\mu_{(y_k)_i}, \Sigma_k)
	\end{align}
\end{assumption}
Note that this is identical to the original set up  except each population has a different population covariance matrix $\Sigma_k$. In addition, we assume the assumption \ref{as:RMT} holds for all covariance matrices $\Sigma_k$.
 \begin{assumption}[Heterogeneous \textsc{rmt} assumption]\label{as:HRMT} 
	For $k = 1, \cdots, K$, the design matrix $\bX_k\in \RR^{n_k\times p}$ is generated as 
	\[
	\bX_k=
	\left(\mu_{(y_k)_1,k}\,\mu_{(y_k)_2,k}\,\ldots\,\mu_{(y_k)_{n_k},k}\right)^{\top}
	+
	\bZ_k \Sigma^{1/2}_k
	\]
	for a matrix $\bZ_k\in \RR^{n_k\times p}$ with $i.i.d.$ entries coming from an infinite array. The entries $(Z_k)_{ij}$ of $\bZ_k$ satisfy the moment conditions: 
	\[\EE[(Z_k)_{ij}] = 0,\,\, \EE[(Z_k)_{ij}^2] = 1
	\text{ and }
	\EE[(Z_k)_{ij}^4] \le C
	.
	\] 
	\begin{enumerate}
		\item The population covariance matrix $\Sigma_k\in \RR^{p\times p}$ is deterministic. The observations have unit variance, i.e., $(\Sigma_k)_{jj} = 1$ for $j = 1, \cdots, p$. 
		\item The eigenvalues of $\Sigma_k$ are uniformly bounded from above and away from zero with constants independent of the dimension $p$.
		\item The sequence of spectral distributions $T_k := (T_k)_{\Sigma, p}$ of $\Sigma_k := (\Sigma_k)_{p}$ converges weakly to a limiting distribution $H_k$ supported on $[0, \infty)$, called the population spectral distribution (\textsc{psd}).
	\end{enumerate}
\end{assumption}
We can then derive the optimal prediction and estimation weights in a manner identical to Theorems~\ref{th:tl-rda-MinEst} and \ref{th:tl-rda-MinPred}. Since  the Bayes optimal discriminant direction is given by $\Sigma_K^{-1}\delta_K$, and as before we aim to leverage the related observations in each source through a weighted linear combination of their discriminant directions $\hat{d}_k$, for $k=1,\dots, K$.

\begin{theorem}[Asymptotic Estimation Error Minimization for TLH-RDA] \label{th:tlH-rda-MinEst}
	Suppose that assumptions \ref{as:RCW}, \ref{as:CCW}, \ref{as:moment}, \ref{as:aniso}  as well as \ref{as:HTCG} and \ref{as:HRMT} hold.  Then for a fixed $K\ge 2$, as $n_k, p \rightarrow \infty, p / n_k \rightarrow \gamma_k\in (0, \infty]$ for $1\le k\le K$, the weight for minimizing the error in estimating the Bayes optimal discriminator $d_{Bayes}$ is given by:
	\[\bw^E_H := 
	\arg \min_{\bw} 
	\left\Vert d_{Bayes} - 
	\sum_{k = 1}^{K} w_k \hat{d}_{k}\right\Vert_2^2 = (\mathcal{A}^E_H + \mathcal{R}^E_H)^{-1} \bu^E_H,
	\]
	where the elements of $\bu^E_H\in\RR^K$, $\calA^E_H\in \RR^{K\times K}$, and $\calR^E_H\in \RR^{K\times K}$ are:
	\[
	(u^E_H)_k
	= 
	\begin{cases}
		\rho_{k K} \alpha_k\alpha_K  \tr(\Sigma_K^{-1} (\hat{\Sigma}_k + \lambda_{k} \II_p)^{-1}) 
		\quad
		&\text{if }
		k=1,\dots,K-1
		,\\
		\alpha_K^2 \left[\frac{1}{\lambda_k} \EE(T^{-1}_K)  - m_{F_{\gamma_k}}(-\lambda_k)^2\right]
		&\text{if }
		k=K.
	\end{cases}
	\]
	\[
	(\mathcal{A}^E_H)_{kk'} 
	= 
	\begin{cases}
		\alpha_{k}^2 m_{F_{\gamma_k}}'(-\lambda_{k})
		\quad &\text{if }k=k',\\
		\rho_{k k'} \alpha_k \alpha_{k'} \mathcal{U}_{kk'}
		\quad &\text{otherwise,}
	\end{cases}
	\]
	and
	\[	
	(\mathcal{R}^E_H)_{kk'} 
	=
	\begin{cases}
		\frac{v_{F_{\gamma_k}}(-\lambda_{k}) - \lambda_k v'_{F_{\gamma_k}}(-\lambda_k)}{\lambda_{k} v_{F_{\gamma_k}}(-\lambda_{k})^2}  
		\quad &\text{if }k=k',
		\\
		0 \quad &\text{otherwise,}
	\end{cases}
	\]
	for $k, k' = 1,\cdots , K$.  Moreover for $k, k' = 1,\cdots , K$, let $\calU_{kk'}$ be the constants defined as the following limiting quantities:
	\begin{equation}\label{eq:def-Ukk'}
		\tr[(\hat{\Sigma}_k + \lambda_{k} \II_p)^{-1} (\hat{\Sigma}_{k'} + \lambda_{k'} \II_p)^{-1}] / p \rightarrow_{a.s.} \mathcal{U}_{k k'}.
	\end{equation}	 
\end{theorem}
Since the quantities on the left hand side of \eqref{eq:def-Ukk'} are exactly known in terms of sample quantities, we do not seek the exact limits of the traces of the cross sample covariance terms $\mathcal{U}_{k k'}$. Instead, we advocate directly using the known quantities 
\[\tr[(\hat{\Sigma}_k + \lambda_{k} \II_p)^{-1} (\hat{\Sigma}_{k'} + \lambda_{k'} \II_p)^{-1}] / p.
\] 
The situation changes however for estimating $(u^E_H)_k$ when $k \neq K$, since they depend on the unknown target population covariance $\Sigma_K$. This is guaranteed by the almost sure convergence of a suitable sample based quantity, as recorded in the following proposition.
\begin{proposition}
	\[\tr(\Sigma_K^{-1} (\hat{\Sigma}_k + \lambda_{k} \II_p)^{-1}) - (1-\gamma_K)\tr(\hat \Sigma_K^{-1} (\hat{\Sigma}_k + \lambda_{k} \II_p)^{-1}) \rightarrow_{a.s.} 0 \]
	as $n_k, p \rightarrow 0; p/n_k \rightarrow \gamma_k$ for $k = 1, \cdots, K-1$.
\end{proposition}
The proof of this proposition is immediate from Corollary 4.2 of \cite{dobriban2021distributed}. We next move on to the optimal prediction weights for TLH-RDA.
\begin{theorem}[Asymptotic Prediction Error Minimization for TLH-RDA] \label{th:tlH-rda-MinPred}
	Suppose that assumptions \ref{as:RCW}, \ref{as:CCW}  as well as \ref{as:HTCG} and \ref{as:HRMT} hold. Then for a fixed $K\ge 2$, as $n_k, p \rightarrow \infty, p / n_k \rightarrow \gamma_k\in (0, \infty]$ for $1\le k\le K$, the weight for minimizing the excess risk, i.e., the error in predicting the class at a random test point $x_0$, when compared to the Bayes optimal discriminator $d_{Bayes}$, is given by:
	\[\bw^P_H := 
	\arg \min_{\bw} \EE_{x_0} \left[(d_{Bayes} - \sum_{k = 1}^{K} w_k \hat{d}_{k})^{\top} x_0\right]^2 = 
	(\mathcal{A}_H^P + \mathcal{R}_H^P)^{-1} \bu^P_H,
	\]
	where the elements of $\bu^P\in\RR^K$, $\calA^P\in \RR^{K\times K}$, and $\calR^P\in \RR^{K\times K}$ are:
	\[
	(u^P_H)_k
	= \rho_{k K} \alpha_k\alpha_K m_{F_{\gamma_k}}(-\lambda_k)
	\quad
	\text{for }
	k=1,\dots,K
	,
	\]
	\[
	(\mathcal{A}^P_H)_{kk'} 
	= 
	\begin{cases}
		\alpha_k^2  \left[\frac{v_{F_{\gamma_k}}(-\lambda_k) - \lambda_k v'_{F_{\gamma_k}}(-\lambda_k) }{\gamma_k [\lambda_k v_{F_{\gamma_k}}(-\lambda_k) ]^2}\right]
		\quad &\text{if }k=k'=K,\\
		\rho_{k k'} \alpha_k \alpha_{k'}  \mathcal{Y}_{k k'} 
		\quad &\text{otherwise,}
	\end{cases}
	\]
	and
	\[	
	(\mathcal{R}^E_H)_{kk'} 
	=
	\begin{cases}
		\frac{v'_k(-\lambda_{k}) - v_k^2(-\lambda_k)}{\lambda_k^2 v^4_k(-\lambda_{k})}
		\quad &\text{if }k=k'
		\\
		0 \quad &\text{otherwise,}
	\end{cases}
	\]
	for $k, k' = 1,\cdots , K$. Moreover for $k, k' = 1,\cdots , K$, let $\calY_{kk'}$ be the constants defined as the following limiting quantities:
	\[\tr[(\hat{\Sigma}_k + \lambda_{k} \II_p)^{-1} (\hat{\Sigma}_{k'} + \lambda_{k'} \II_p)^{-1}\Sigma_K] / p \rightarrow_{a.s.} \mathcal{Y}_{k k'}.\]
\end{theorem}

In order to utilize the above weights in a practical scenario, we need to replace all unknown quantities by their estimates. As before, the Marcenko Pastur law is crucial in estimating the parameters related to the spectral distributions of $\Sigma_k$, for $k=1,\dots,K$. Most of the details are straightforward and identical in estimation procedure as the rest of the paper. It remains to provide consistent estimators for $\calY_{kk'}$, which we describe in the following proposition.
\begin{proposition}\label{pr:heterog-cons-est-ykk} We have the following consistent estimators $\hat{\calY}_{kk'}$ for $\calY_{kk'}$ separately for three cases:
	\[
	\hat{\calY}_{kk'}
	=
	\begin{cases}
		\tr[(\hat{\Sigma}_k + \lambda_{k} \II_p)^{-1} (\hat{\Sigma}_{k'} + \lambda_{k'} \II_p)^{-1} \hat \Sigma_K] / p\quad &\text{if }k\neq k', k\neq K, k'\neq K\\
		\tr[(\hat{\Sigma}_k + \lambda_{k} \II_p)^{-2} \hat \Sigma_K] / p
		\quad &\text{if }k=k'\neq K\\
		\frac{1}{px_p}
		\tr[(\hat{\Sigma}_{k'} + \lambda_{k'} \II_p)^{-1}]
		-
		\frac{\lambda_K}{px_p} 
		\tr[(\hat{\Sigma}_K +\lambda_K \II_p )^{-1} (\hat{\Sigma}_{k'} + \lambda_{k'} \II_p)^{-1}]
		\quad &\text{if }k=K,\,k'\neq K
	\end{cases}
	\]
	where $x_p=x(\gamma_K,\lambda_K)$ is the solution to the equation:
	\[
	1-x_p=\gamma_K
	\left
	[1-\lambda_K\int(x_pt+\lambda_K)^{-1}dH_K(t)
	\right]
	\]
	Then 
	\[\calY_{kk'} -\hat{\calY}_{kk'}  \rightarrow_{a.s.} 0.\]
\end{proposition}


\section{Proteomics-based Prediction of 10-year Cardiovascular Disease Risk}\label{sec:real-data}
We utilize the data set from the Chronic Renal Insufficiency Cohort (CRIC) study  to evaluate the performance of the proposed transfer learning methods. The dataset comprises observations from $2,182$ subjects who had not had any cardiovascular disease at the baseline, with $4,830$ protein measurements collected. These subjects were approximately evenly distributed across seven university sites.  Our goal is to build a classification/prediction model for 10-year cardiovascular disease risk based on the baseline plasma proteomic data. We hypothesize that while the mechanisms by which proteins predict events are related across these sites, they are not identical. Consequently, we sequentially designate each site as the target population, with the remaining sites serving as auxiliary populations, aiming to enhance prediction accuracy using TL-RDA or TLP-RDA.

For each target population, approximately 20\% of the patients were set aside as a testing dataset, while the remaining 80\% were used for training. We ensured that the proportion of events was consistent between the training and testing datasets ($\sim$7\%). A univariate filtering procedure was applied to the training dataset using two-sample t-tests, and we evaluated the benefits of transfer learning using the top 500, 1000, and 1500 proteins. The tuning parameters $\lambda_k$ for each model were selected independently based on cross-validation.

We considered all four variants of transfer learning RDA: TL-RDA with optimal estimation/prediction weights and TLP-RDA with optimal estimation/prediction weights. Among these, the method yielding the highest cross-validated AUC was selected and tested on the held-out testing dataset. Its testing AUC was then compared with those obtained from competing methods, namely naive RDA and pool-RDA. As implied by their names, naive RDA is fitted using only the target population data, while pool-RDA is fitted on data pooled from all populations. The resulting testing AUCs are presented in Table~\ref{tab:Ttrait}.

\begin{table}[htbp]
	\begin{center}
		\caption{AUC on testing data for 10-year cardiovascular disease risk in CRIC cohort. The TL column show the testing AUC of the transfer learning RDA method with highest cross-validated AUC. The TL type columns show the type of the selected transfer learning RDA method.  TL-RDA-E, TL-RDA-P, TLP-RDA-E and  TLP-RDA-P are respectively TL-RDA with the optimal estimation / prediction weights, TLP-RDA with the optimal estimation / prediction weights. The bold numbers in each row of the table highlight the best-performing algorithm. Transfer learning methods that perform worse than the competing methods are highlighted in red.}
		\label{tab:Ttrait}
			\begin{tabular}{ c cccccc}
            \hline
			& Naive & Naive Pool & TL-RDA-E & TL-RDA-P & TLP-RDA-E & TLP-RDA-P \\
            \cline{2-7}
			\multicolumn{1}{c}{} & \multicolumn{6}{c}{$p=500$}\\
			\textsc{Site 1} &0.471 &0.594 &{0.652} &{0.653} &\textbf{0.703} &\textbf{0.703} \\
			\textsc{Site 2} &0.584 &0.584 &{0.621}  &\textbf{0.634} &{0.598} &0.622 \\
			\textsc{Site 3} &0.643 &0.617 &0.725 &0.719 &{0.698} &\textbf{0.742} \\
			\textsc{Site 4} &0.523 &0.536 &0.602 &0.593 &0.577 &\textbf{0.603} \\
			\textsc{Site 5} &0.502 &0.536 &0.609 &\textbf{0.640} &0.614 &0.621 \\
			\textsc{Site 6} &0.537 &{0.589} &0.590 &\textcolor{red}{0.583} &0.619 &\textbf{0.635} \\
			\textsc{Site 7} &{0.665} &{0.629} &\textcolor{red}{0.552} &\textcolor{red}{0.579} &\textcolor{red}{0.650} &\textbf{0.728} \\
			\multicolumn{1}{c}{} & \multicolumn{6}{c}{$p=1000$}\\
			\textsc{Site 1} &0.638 &0.722 &\textcolor{red}{0.662} &\textcolor{red}{0.658} &0.762 &\textbf{0.829} \\
			\textsc{Site 2} &0.525 &0.521 &\textbf{0.659} &{0.655} &{0.606} &{0.616} \\
			\textsc{Site 3} &{0.536} &{0.643} &\textcolor{red}{0.611} &\textcolor{red}{0.597} &0.706 &\textbf{0.776} \\
			\textsc{Site 4} &0.580 &\textbf{0.663} &\textcolor{red}{0.602} &\textcolor{red}{0.606} &\textcolor{red}{0.631} &\textcolor{red}{0.633} \\
			\textsc{Site 5} &0.629 &0.591 &\textcolor{red}{0.585} &{0.636} &{0.686} &\textbf{0.785} \\
			\textsc{Site 6} &0.530 &{0.700} &{0.587} &{0.553} &0.705 &\textbf{0.718} \\
			\textsc{Site 7} &{0.495} &{0.589} &{0.633} &{0.600} &0.733 &\textbf{0.822} \\
			\multicolumn{1}{c}{} & \multicolumn{6}{c}{$p=1500$}\\
			\textsc{Site 1} &0.509 &0.597 &0.693 &\textcolor{red}{0.578} &{0.762} &\textbf{0.770} \\
			\textsc{Site 2} &0.530 &{0.420} &{0.681} &{0.641} &0.832 &\textbf{0.863} \\
			\textsc{Site 3} &{0.551} &0.530 &{0.607} &{0.611} &0.701 &\textbf{0.705} \\
			\textsc{Site 4} &{0.635} &{0.430} &0.814 &{0.747} &0.864 &\textbf{0.899} \\
			\textsc{Site 5} &0.573 &0.622 &0.673 &\textcolor{red}{0.602} &{0.702} &\textbf{0.707} \\
			\textsc{Site 6} &{0.473} &0.624 &{0.696} &{0.681} &0.819 &\textbf{0.848} \\
			\textsc{Site 7} &{0.620} &{0.611} &0.657 &\textcolor{red}{0.619} &0.711 &\textbf{0.716} \\
            \hline
		\end{tabular}
	\end{center}
\end{table}
	
The bold numbers in each row of the table highlight the best-performing algorithm. Transfer learning methods that perform worse than the competing methods are highlighted in red.  It is evident that the transfer learning RDA methods outperform the others in the vast majority of cases. This demonstrates that not only are TL-RDAs capable of borrowing information to improve prediction accuracy, but they also do so more efficiently than simply pooling the data. Notably, TLP-RDA with optimal prediction weights outperforms the other methods most frequently. The amount of improvement brought by transfer learning also seemingly increase as we include more proteins in the model.  Additionally, we observe that methods based on pooled sample covariance tend to perform better when the number of features $p$ is larger, consistent with Proposition~\ref{pr:TLPwins}.

\section{Discussion}
We have developed methods of transfer learning for classification using regularized random effects linear discriminant analysis. In this approach, the discriminant direction is estimated through a weighted combination of the regularized estimates of discriminant directions derived from both the target and source models. By leveraging results from random matrix theory, we have demonstrated that the weights, as well as the classification error rate, can be accurately estimated using the data. Our findings highlight that the optimal choice of weights depends critically on the underlying true models and the presence or absence of distributional shifts between the testing and training data in the target model.

Through comprehensive empirical evaluations, we have illustrated the practical utility of the proposed methods. Specifically, we applied the transfer learning approach to predict 10-year cardiovascular disease risk using high-dimensional protein expression data. The results demonstrate that incorporating information from related source datasets substantially improved classification performance compared to target-only models or models based on pooled data. This improvement underscores the importance of accounting for shared information across datasets while accommodating potential distributional differences.

 This paper focuses on two-class LDA models. A promising direction for future research is extending these methods to multi-class classification within the framework of transfer learning. Another interesting avenue is exploring this LDA problem under privacy or communication constraints, as has been recently studied for other classification methods \citep[see, e.g.,][]{auddy2024minimax}.
\section*{Acknowledgments}
This research was supported by NIH grants R01GM129781 and U01DK108809.

\bibliographystyle{abbrvnat}
\bibliography{paper-ref}

\newpage 
\appendix

\section{Appendix A}
\subsection{Proofs of Propositions}
\begin{proof}[Proof of Proposition \ref{pr:mean-invT}] 
	Denote $E$ as the random variable distributed according to the empirical spectral distribution $F_\gamma$. When $\gamma < 1$, $E$ is supported on a compact set bounded away from $0$ \citep{bai2010spectral}. Therefore, we can take the limit of $m_{F_\gamma}(-\lambda)$ as $\lambda \rightarrow 0 $. Recall the Marchenko–Pastur equation
	\[m_{F_\gamma}(z) = \int_{t = 0}^{\infty} \frac{dH(t)}{t(1 - \gamma - \gamma z m_{F_\gamma}(z)) - z}\]
	We find $m_{F_\gamma}(0) = \int 1/[t(1 - \gamma)] dH(t)$ or equivalently $\tr (\hat{\Sigma}^{-1})/p \rightarrow_{a.s.} \EE(T^{-1}) / (1 - \gamma)$.
\end{proof}

\medskip

\begin{proof}[Proof of Proposition~\ref{pr:heterog-cons-est-mkk}]
	The proof is identical to the proof of Proposition~\ref{pr:heterog-cons-est-ykk} which considers the more general case of unequal covariance matrices $\Sigma_k$.
\end{proof}

\medskip

\begin{proof}[Proof of Proposition \ref{pr:PredIsOpt}] 
	Recall from Theorem \ref{limError} that the limiting prediction error is in the form of 
	\[\Phi\left(- \frac{\bw^{\top} \bu}{\sqrt{\bw^{\top} A \bw}}\right)\]
	Any weight $\bw$ with negative linear term $\bw^{\top} \bu$ has higher prediction error than the weight with positive linear term. Without loss of generality, we assume $\bw^{\top} v$ to be positive. Then we know
	\[\bw^* = \arg \max_\bw \frac{\bw^{\top} \bu}{\sqrt{\bw^{\top} \mathcal{A} \bw}} =  \arg \max \frac{\bw^{\top} \bu \bu^{\top} \bw}{\bw^{\top} \mathcal{A} \bw} \]
	this becomes the generalized Rayleigh quotient problem and the solution is given by the first unit eigenvector of 
	$\mathcal{A}^{-1} \bu\bu^{\top} $,
	which is a rank one matrix, and hence
	$$
	\bw^* =\dfrac{1}{\| \mathcal{A}^{-1} \bu \|} \mathcal{A}^{-1} \bu,
	\text{ which implies, }
	Err_{opt} := \Phi\left(-\sqrt{\bu^{\top} \mathcal{A}^{-1}\bu} \right).
	$$
	Note the standardization factor in $\bw$ is not necessary due to the form of $Err(\bw)$. The proof finishes as one recognizes the $\mathcal{A} = \mathcal{A}^P + \mathcal{R}^P$ and $\bu = \bu^P$. We can prove the statement for $\bw^P_P$ in the same manner.
\end{proof}

\medskip

\begin{proof}[Proof of Proposition \ref{pr:pos-def}] 
	We only need to prove	$\mathcal{A}^E + \mathcal{R}^E, \mathcal{A}^P + \mathcal{R}^P, \mathcal{A}^E_P + \mathcal{R}^E_P, \mathcal{A}^P_P + \mathcal{R}^P_P$ are positive definite, therefore, invertible. Looking at the individual sample covariance matrix cases, taking individual prediction weight as an example. We will firstly show that $\calA^{P}$ is positive semi-definite. Note that
	\begin{align*}
		(\calA^{P})_{kk'} =&~
		\alpha_k\alpha_{k'}\rho_{kk'}
		{\rm tr}[\Sigma(\hat{\Sigma}_k+\lambda_k \II_p)^{-1}
		(\hat{\Sigma}_{k'}+\lambda_{k'}\II_p)^{-1}
		]/p\\
		=&~
		\alpha_k \alpha_{k'}\rho_{kk'}
		{\rm tr}[\bM_k^{\top}\bM_{k'} 
		]/p\\
		=&~
		\alpha_k\alpha_{k'}\rho_{kk'}
		\bv_{M_k}^{\top}\bv_{M_{k'}}/p
	\end{align*}
	where $\bM_k=\Sigma^{1/2}(\hat{\Sigma}_k+\lambda_k \II_p)^{-1}$ and $\bv_{M_k}={\rm vec}(\bM_k)\in \RR^{K^2}$.
	For any $\bx\in \RR^K$, $\bx\neq \mathbf{0}$, we have
	\begin{align*}
		\bx^{\top} \calA^{P} \bx
		=&~
		\sum_{k,k'}(\alpha_kx_k)(\alpha_{k'}x_{k'})
		\rho_{kk'} \bv_{M_k}^{\top}\bv_{M_{k'}}/p\\
		=&~
		\by^{\top}\calA^{P} \by/p
	\end{align*}
	where $\by =((\alpha_ix_i))_i\in \RR^K$. Since $\alpha_i>0$ for all $i$, we have $\by\neq \mathbf{0}$ if and only if $\bx\neq \mathbf{0}$. It is thus enough to show that $\calA^{(3)}$, defined through
	\[
	(\bar{\calA})_{ij}=~
	\rho_{ij} \bv_{M_i}^{\top}\bv_{M_j}
	\]
	is positive-semi definite. To this end, note that
	\begin{align*}
		\by^{\top}\bar{\calA} \by
		=&~
		\sum_{i,j}\rho_{ij}(y_i\bv_{M_i})^{\top}
		(y_j\bv_{M_j})\\
		=&~
		\sum_{i,j} 
		\rho_{ij} (\bc_i)^{\top}\bc_j 
		=
		{\rm tr}(\Sigma_{\delta} \bC^{\top}\bC)\\
		=&~
		{\rm tr}(\bC \Sigma_{\delta}^{1/2} \Sigma_{\delta}^{1/2}\bC^{\top}) \\
		=&~
		\|\boldsymbol{\rho}^{1/2}\bC^{\top}\|_{\rm F}^2
		\ge 0.
	\end{align*}
	Here $\bC\in \RR^{K^2\times K}$ matrix with columns $\bc_i=y_i\bv_{M_i}$, for $i=1,\dots,K$. This shows that $\calA^{(1)}$ is positive semi-definite since for any $\bx\in \RR^K$, $\bx\neq\mathbf{0}$ we have
	\begin{equation}\label{eq:A1-psd}
		\bx^{\top}\calA^{P}\bx
		=\by^{\top} \bar{\calA} \by\ge 0.
	\end{equation}
	Now note that $\calA^{(2)}$ is a diagonal matrix with non-negative elements, this finishes the proof of the positive definiteness of $\calA^{P} + \calA^{R}$. To prove the case of estimation weight, one can simply define $\bM_k=(\hat{\Sigma}_k+\lambda_k \II_p)^{-1}$. 
	
	Looking at the pooled sample covariance matrix cases, taking pooled prediction weight as an example, we would have
	\begin{align*}
		\mathcal{A}^P_P &= c^A \Sigma_{\alpha\delta} \\
		\Sigma_{\alpha\delta} &= mat \left [ \rho_{k  k'} \alpha_{k} \alpha_{k'} \right] \\
		c^A &= \tr [ (\hat{\Sigma}_P + \lambda I_p)^{-2} \Sigma] / p 
	\end{align*}
	Since $\mathcal{R}^P$ has non-negative diagonal entries and $c^A > 0$, we only need to prove the positive semi-definiteness of $\Sigma_{\alpha\delta}$. Consider the quadratic form \( x^{\top} \Sigma_{\alpha\delta} x \):
	\[
	x^{\top} \Sigma_{\alpha\delta} x  = \sum_{i=1}^K \sum_{j=1}^K \alpha_i x_i \alpha_j (\Sigma_{\alpha\delta})_{ij} x_j = \left( \sum_{i=1}^K \alpha_i x_i \right) \left( \sum_{j=1}^K \alpha_j \Sigma_{ij} x_j \right)
	\]
	Define \( y_i = \alpha_i x_i \). Then we have:
	\[
	x^{\top} \Sigma_{\alpha\delta} x = \sum_{i=1}^K y_i \sum_{j=1}^K (\Sigma_\delta)_{ij} y_j = y^{\top} \Sigma y
	\]
	Since \( \Sigma_\delta \) is positive semi-definite, $\mathcal{A}^P_P + \mathcal{R}^P_P$ will be positive definite. We can prove the positive definiteness of $\mathcal{A}^E_P + \mathcal{R}^E_P$ in exactly the same way. 
\end{proof}

\medskip

\begin{proof}[Proof of Proposition~\ref{pr:minmax}] 
	Since the maximum value of the dot product of two vectors is achieved when the two vectors are aligned, the maximization of the min-max problem is solved when $x_0$ is in the same direction as the error with $\EE_P(\|x_0\|_2) = c$. Since  $x_0$ is independent of the estimation error vector $(\sum_{k = 1}^{K} w_k \hat{d}_k - d_{Bayes})^{\top}$,  we have the solution of the $\max$ problem as 
	\[x_0^* =  c \frac{d_{Bayes} - \sum_{k = 1}^{K} w_k \hat{d}_k}{\|d_{Bayes} -  \sum_{k = 1}^{K} w_k \hat{d}_k\|} \text{ with probability } 1.\]
	This leads to the estimation risk minimization problem with criteria \eqref{eq:CrEst}.
\end{proof}

\medskip

\begin{proof}[Proof of Proposition \ref{pr:TLPwins}]
	For simplicity, we assume $\Sigma = \II_p$, in which case the optimal esimation weight coincides with the optimal prediction weights. In addition, we have the following simplifications
	\begin{corollary}[Simplification of Limiting Error and Optimal Weights for Individual Sample Covariance Matrix] \label{cor:lim-err-indiv}
		Recall the form of limiting prediction error given a weight vector $\bw$ \eqref{limError}.We have the term $\mathcal{A}$ simplified as following. For $k = 1, \cdots, K$, 
		\[\mathcal{A}_{kk} = \alpha_k^2  m'_k(-\lambda_k) +\gamma_k m'_k(-\lambda_k)\]
		else when  $k \neq k'; k, k' = 1,\cdots , K$ 
		\[\mathcal{A}_{kk'} = \rho_{k k'} \alpha_k \alpha_{k'}  m_k(-\lambda_k)  m_{k'}(-\lambda_{k'}) \]
		and we have $\bu^E = vec[\rho_{k K} \alpha_k \alpha_K m_{F_{\gamma_k}}(-\lambda_k)]$, $\mathcal{A}_{kk'}^E + \mathcal{R}^E  = \mathcal{A}$. 
	\end{corollary}
	\begin{corollary}[Simplification of Limiting Error and Optimal Weights for Pooled Sample Covariance Matrix]\label{cor:lim-err-pooled}
		Define the weighted covariance matrix as $\Sigma_\delta = mat[\rho_{k k'} \alpha_{k} \alpha_{k'}], \rho_{kk} = 1$.
		We have the term $\mathcal{A}$ simplified as:
		\[\mathcal{A}_P = \Sigma_\delta m'_{F_{\gamma/K}}(-\lambda) + \gamma_k m'_{F_{\gamma/K}}(-\lambda) \II_K \]
		and $\bu^E_P = vec[\rho_{k K} \alpha_k \alpha_K m_{F_{\gamma/K}}(-\lambda)]$, $\mathcal{A}_P^E + \mathcal{R}^E_P  = \mathcal{A}_P$. 
	\end{corollary}
	
	We can now write the limiting error of using individual sample covariance matrix and pooled covariance matrix as 
	\begin{align*}
		Err_{ind} = \Phi\left(- \sqrt{ (\bu^E)^\top \mathcal{A}^{-1} \bu^E}\right)
	\text{ and }
		Err_{pool} = \Phi\left(- \sqrt{(\bu^E_P)^\top \mathcal{A}_P^{-1} \bu^E_P }\right).
	\end{align*}
	Recall that we have the explicit expressions
	\begin{equation}\label{eq:m-quad-eq-identity}
		m_{F_\gamma}(-\lambda) = \frac{-(1 - \gamma + \lambda) + \sqrt{(1 - \gamma + \lambda)^2 + 4 \gamma \lambda}}{2 \gamma \lambda}
	\end{equation}
	
	\begin{equation}\label{eq:mm'-identity}
		m'_{F_\gamma}(-\lambda) = \frac{m^2_{F_\gamma}(-\lambda) [1 +\gamma m_{F_\gamma}(-\lambda)]}{1 + \gamma \lambda m^2_{F_\gamma}(-\lambda)}.
	\end{equation}
	
	For some $r>(1-\gamma)_+/\gamma +c$ (for some small constant $c$ to be chosen later), we choose $\lambda$ as follows:
	\begin{align*}
		1-\gamma+\lambda = r\gamma-\lambda/r
		\iff (1+1/r)\lambda =&~ (r+1)\gamma -1\\
		\iff \lambda_* =&~ r\left(\gamma -\frac{1}{r+1}\right) 
	\end{align*}

	The following proposition contains a specific choice of $r$ and some inequalities that will be useful for the rest of this proof. The proof of this proposition is immediate and hence omitted.
	
	\begin{proposition}\label{pr:r-mm'-ineq} The following statements hold for $r>(1-\gamma)_+/\gamma +c$ (for some constant $c>0$):
		\begin{enumerate}
			\item $m_{\gamma}(-\lambda_*)=\frac{1}{r\gamma}$
			\item $\frac{m'_{F_\gamma}(-\lambda_*)}{m^2_{F_\gamma}(-\lambda_*)}
			=
			\frac{\gamma(1+r)^2}{\gamma(1+r)^2 -1}$
			\item $\Delta:=	\frac{m'_{F_\gamma}(-\lambda_*)}{m^2_{F_\gamma}(-\lambda_*)}-1
			=\frac{1}{\gamma(1+r)^2-1}
			\le\frac{1}{\gamma-1}
			$
		\end{enumerate}
	\end{proposition}

	We now look at the quadratic term in $Err_{pool}$ and assume $\rho_{k k'} = \rho$. Then by Corollary~\ref{cor:lim-err-pooled} we have
	\begin{align*} 
		\mathcal{A}_P 
		=&~ \Sigma_\delta m'_{F_{\gamma/K}}(-\lambda) +  m'_{F_{\gamma/K}}(-\lambda){\rm diag}(\{\gamma_k:1\le k\le K\}) \\
		=&~ m'_{F_{\gamma/K}}(-\lambda)\left[
		\rho
		\tilde{\boldsymbol{\alpha}}\tilde{\boldsymbol{\alpha}}^{\top}
		+
		{\rm diag}(\{(1-\rho)\alpha_k^2+\gamma_k:1\le k\le K\})
		\right]\\
		=:&~ m'_{F_{\gamma/K}}(-\lambda)
		\left[ 
		\rho
		\tilde{\boldsymbol{\alpha}}\tilde{\boldsymbol{\alpha}}^{\top}
		+
		D_{\rho,\tilde{\bgamma}}
		\right]
	\end{align*}
	and
	\[
	\bu^E_P = vec[\rho_{k K} \alpha_K \alpha_k m_{F_{\gamma/K}}(-\lambda)]
	=m_{F_{\gamma/K}}(-\lambda)\alpha_K
	[\rho\tilde{\boldsymbol{\alpha}}+(1-\rho)\alpha_K\be_K]
	.
	\]
	By the Sherman Morrison formula, we have:
	\begin{align*}
		\calA_P^{-1}
		=&~\dfrac{1}{m'_{F_{\gamma/K}}(-\lambda)}
		\left[
		D_{\rho,\tilde{\bgamma}}^{-1}
		-
		\dfrac{\rho}{1+
			\rho \sum_{k=1}^K \frac{\alpha_k^2}{(1-\rho)\alpha_k^2+\gamma_k}}
		D_{\rho,\tilde{\bgamma}}^{-1}\tilde{\balpha}\tilde{\balpha}^{\top}
		D_{\rho,\tilde{\bgamma}}^{-1}
		\right].
	\end{align*}
	Consequently:
	\begin{align*}
		\tilde{\balpha}^{\top}\calA_P^{-1}\tilde{\balpha}
		=&~
		\dfrac{1}{m'_{F_{\gamma/K}}(-\lambda)}
		\cdot
		\dfrac{\tilde{\balpha}^{\top}D_{\rho,\tilde{\bgamma}}^{-1}\tilde{\balpha}}
		{1+\rho \tilde{\balpha}^{\top}D_{\rho,\tilde{\bgamma}}^{-1}\tilde{\balpha}}\\
		\tilde{\balpha}^{\top}\calA_P^{-1}\be_K 
		=&~
		\dfrac{1}{m'_{F_{\gamma/K}}(-\lambda)}
		\cdot 
		\dfrac{ \tilde{\balpha}^{\top}D_{\rho,\tilde{\bgamma}}^{-1}\be_K}
		{1+\rho \tilde{\balpha}^{\top}D_{\rho,\tilde{\bgamma}}^{-1}\tilde{\balpha}}
		=~
		\dfrac{\alpha_K}{m'_{F_{\gamma/K}}(-\lambda)}
		\cdot 
		\dfrac{ \be_K^{\top}D_{\rho,\tilde{\bgamma}}^{-1}\be_K}
		{1+\rho \tilde{\balpha}^{\top}D_{\rho,\tilde{\bgamma}}^{-1}\tilde{\balpha}}
		\\
		\be_K^{\top}\calA_P^{-1}\be_K
		=&~
		\dfrac{1}{m'_{F_{\gamma/K}}(-\lambda)}
		\cdot 
		\left[ \be_K^{\top}D_{\rho,\tilde{\bgamma}}^{-1}\be_K 
		-
		\dfrac{\alpha_K^2}{1+\rho \tilde{\balpha}^{\top}D_{\rho,\tilde{\bgamma}}^{-1}\tilde{\balpha}}
		(\be_K^{\top}D_{\rho,\tilde{\bgamma}}^{-1}\be_K )^2
		\right]\\
		=&~
		\dfrac{\be_K^{\top}D_{\rho,\tilde{\bgamma}}^{-1}\be_K }{m'_{F_{\gamma/K}}(-\lambda)}
		\cdot
		\left[
		\dfrac{1+\rho \tilde{\balpha}^{\top}D_{\rho,\tilde{\bgamma}}^{-1}\tilde{\balpha}
		-\rho\alpha_K^2\be_K^{\top}D_{\rho,\tilde{\bgamma}}^{-1}\be_K
		}{1+\rho \tilde{\balpha}^{\top}D_{\rho,\tilde{\bgamma}}^{-1}\tilde{\balpha}}
		\right]
	\end{align*}
	which, upon writing $\tilde{\balpha}_{/K}=(\alpha_1\,\alpha_2\,\dots\,\alpha_{K-1}\,0)^{\top}$, implies that:	
	\begin{align*}
		&~(\bu^E_P)^\top \mathcal{A}_P^{-1} \bu^E_P \\
		=&~
		\frac{\alpha_K^2\cdot m^2_{F_{\gamma/K}}(-\lambda)}{m'_{F_{\gamma/K}}(-\lambda)} 
		\left[
		\rho^2\tilde{\boldsymbol{\alpha}}^{\top}\calA_P^{-1}\tilde{\balpha}
		+
		2\rho(1-\rho)\alpha_K\tilde{\balpha}^{\top}\calA_P^{-1}\be_K 
		+
		(1-\rho)^2\alpha_K^2\be_K^{\top}\calA_P^{-1}\be_K 
		\right]\\
		=&~ \frac{m^2_{F_{\gamma/K}}(-\lambda)}{m'_{F_{\gamma/K}}(-\lambda)} 
		\cdot
		\dfrac{\alpha_K^2}{1+\rho\tilde{\balpha}^{\top}D_{\rho,\tilde{\bgamma}}^{-1}\tilde{\balpha}}\times\\
		&~\times
		\left[
		\rho^2\tilde{\balpha}^{\top}D_{\rho,\tilde{\bgamma}}^{-1}\tilde{\balpha}
		+
		(1-\rho)
		\alpha_K^2\be_K^{\top}D_{\rho,\tilde{\bgamma}}^{-1}\be_K 
		[
		2\rho 
		+
		1-\rho 
		+\rho(1-\rho)\tilde{\balpha}_{/K}^{\top}D_{\rho,\tilde{\bgamma}}^{-1}\tilde{\balpha}_{/K}
		]
		\right]\\
		=&~
		\frac{m^2_{F_{\gamma/K}}(-\lambda)}{m'_{F_{\gamma/K}}(-\lambda)} 
		\cdot
		\dfrac{\alpha_K^2}{1+\rho\tilde{\balpha}^{\top}D_{\rho,\tilde{\bgamma}}^{-1}\tilde{\balpha}}\times\\
		&~\times 
		\left[
		\alpha_K^2\be_K^{\top}D_{\rho,\tilde{\bgamma}}^{-1}\be_K
		+
		\rho^2\tilde{\balpha}_{/K}^{\top}D_{\rho,\tilde{\bgamma}}^{-1}\tilde{\balpha}_{/K}
		+
		\rho(1-\rho)^2(\alpha_K^2\be_K^{\top}D_{\rho,\tilde{\bgamma}}^{-1}\be_K)
		\tilde{\balpha}_{/K}^{\top}D_{\rho,\tilde{\bgamma}}^{-1}\tilde{\balpha}_{/K}
		\right]
		\numberthis\label{eq:decom-pool}
	\end{align*}

	Let us now look at the quadratic term of the individual directions, once again assuming $\rho_{k k'} = \rho$. 	We make another simplifying assumption at this point:
	
	\noindent
	\textbf{Assumption:} $\gamma_k=\gamma$ and $\lambda_k=\lambda$ for $k=1,\dots,K$. Then the population specific covariance matrix based quadratic form reduces through the following calculation.
	
	By Corollary~\ref{cor:lim-err-indiv} we now have
	\begin{align*}
		\calA =&~ m^2_{F_\gamma}(-\lambda)\left[
		\rho\tilde{\balpha}\tilde{\balpha}^{\top}
		+
		{\rm diag}\left(\left\{
		\left(\frac{m'_{F_{\gamma}}(-\lambda)}{m^2_{F_{\gamma}}(-\lambda)}-\rho\right)\alpha_k^2
		+\frac{m'_{F_{\gamma}}(-\lambda)\gamma}{m^2_{F_{\gamma}}(-\lambda)}
		:1\le k\le K\right\}
		\right)
		\right]\\
		=:&~ m^2_{F_\gamma}(-\lambda)\left[
		\rho\tilde{\balpha}\tilde{\balpha}^{\top}
		+ D_{\rho,\gamma,ind}\right]
	\end{align*}
	and
	\[
	\bu^E = vec[\rho_{k K} \alpha_K \alpha_k m_{F_{\gamma}}(-\lambda)]
	=m_{F_{\gamma}}(-\lambda)\alpha_K
	[\rho\tilde{\boldsymbol{\alpha}}+(1-\rho)\alpha_K\be_K]
	.
	\]
	Then following the exact steps leading to \eqref{eq:decom-pool} we arrive at:
	\begin{align*}
		&~(\bu^E)^\top \mathcal{A}^{-1} \bu^E \\
		=&~
		\dfrac{\alpha_K^2}{1+\rho\tilde{\balpha}^{\top}D_{\rho,\gamma,ind}^{-1}\tilde{\balpha}}\times\\
		&~\times 
		\left[
		\alpha_K^2\be_K^{\top}D_{\rho,\gamma,ind}^{-1}\be_K
		+
		\rho^2\tilde{\balpha}_{/K}^{\top}D_{\rho,\gamma,ind}^{-1}\tilde{\balpha}_{/K}
		+
		\rho(1-\rho)^2(\alpha_K^2\be_K^{\top}D_{\rho,\gamma,ind}^{-1}\be_K)
		\tilde{\balpha}_{/K}^{\top}D_{\rho,\gamma,ind}^{-1}\tilde{\balpha}_{/K}
		\right]\numberthis.
		\label{eq:decom-indiv}
	\end{align*}
	Now comparing Equations~\eqref{eq:decom-pool} and \eqref{eq:decom-indiv} we can determine whether the pooled or the individual covariance matrix builds a better estimator. We consider two special cases below:
	
	\noindent
	\textbf{Case 1 ($\rho=1$):} In this case we have from \eqref{eq:decom-pool} and \eqref{eq:decom-indiv} that:
	\[
	(\bu^E_P)^\top \mathcal{A}_P^{-1} \bu^E_P=
	\frac{m^2_{F_{\gamma/K}}(-\lambda')\alpha_K^2}{m'_{F_{\gamma/K}}(-\lambda')} 
	\cdot
	\dfrac{ \tilde{\balpha}^{\top}D_{1,\tilde{\bgamma}}^{-1}\tilde{\balpha}}{1+\tilde{\balpha}^{\top}D_{1,\tilde{\bgamma}}^{-1}\tilde{\balpha}}
	\]
	and
	\[
	(\bu^E)^\top \mathcal{A}^{-1} \bu^E \\
	=
	\dfrac{\alpha_K^2(\tilde{\balpha}^{\top}D_{1,\gamma,ind}^{-1}\tilde{\balpha})}{1+\tilde{\balpha}^{\top}D_{1,\gamma,ind}^{-1}\tilde{\balpha}}.
	\]
	Note that by definition of $D_{\rho,\gamma,ind}$ we have
	\begin{align*}
	\tilde{\balpha}^{\top}D_{1,\gamma,ind}^{-1}\tilde{\balpha}
	=&~\sum_{k=1}^K\dfrac{\alpha_k^2}{
	\left(\frac{m'_{F_{\gamma}}(-\lambda)}{m^2_{F_{\gamma}}(-\lambda)}-\rho\right)\alpha_k^2
	+\frac{m'_{F_{\gamma}}(-\lambda)\gamma}{m^2_{F_{\gamma}}(-\lambda)}
	}\\
	\le&~ \frac{m^2_{F_{\gamma}}(-\lambda)}{m'_{F_{\gamma}}(-\lambda)}\sum_{k=1}^K\frac{\alpha_k^2}{\gamma}
	=\frac{m^2_{F_{\gamma}}(-\lambda)}{m'_{F_{\gamma}}(-\lambda)}\tilde{\balpha}^{\top}D_{1,\tilde{\bgamma}}^{-1}\tilde{\balpha}\\
	=&~ \frac{\gamma(1+r)^2-1}{\gamma(1+r)^2}\times \tilde{\balpha}^{\top}D_{1,\tilde{\bgamma}}^{-1}\tilde{\balpha}
	\end{align*}
	where the last line follows from our choice of $\lambda$. Since $f(t)=\frac{t}{1-t}$ is an increasing function of $t$ it follows that:
	\[
	\frac{\tilde{\balpha}^{\top}D_{1,\gamma,ind}^{-1}\tilde{\balpha}}{1+\tilde{\balpha}^{\top}D_{1,\gamma,ind}^{-1}\tilde{\balpha}}
	\le \frac{[\gamma(1+r)^2-1] \tilde{\balpha}^{\top}D_{1,\tilde{\bgamma}}^{-1}\tilde{\balpha}}{\gamma(1+r)^2 +[\gamma(1+r)^2-1] \tilde{\balpha}^{\top}D_{1,\tilde{\bgamma}}^{-1}\tilde{\balpha}}
	=\frac{[\gamma(1+r)^2-1]\sum_k\alpha_k^2}{\gamma^2(1+r)^2+[\gamma(1+r)^2-1]\sum_k\alpha_k^2}.
	\]
	On the other hand,
	\begin{align*}
		\frac{m^2_{F_{\gamma/K}}(-\lambda')}{m'_{F_{\gamma/K}}(-\lambda')} 
		\cdot
		\dfrac{ \tilde{\balpha}^{\top}D_{1,\tilde{\bgamma}}^{-1}\tilde{\balpha}}{1+\tilde{\balpha}^{\top}D_{1,\tilde{\bgamma}}^{-1}\tilde{\balpha}}
		=~
		\frac{\gamma(1+r')^2-K}{\gamma(1+r')^2}
		\times 
		\frac{\sum_k\alpha^2}{\gamma+\sum_k\alpha_k^2}.
	\end{align*}
	Thus the pooled covariance matrix performs better when
	\[
		\left(\gamma(1+r')^2-K\right)\left(\gamma^2(1+r)^2+[\gamma(1+r)^2-1]\sum_k\alpha_k^2 \right)
		\ge ~
		\gamma(1+r')^2[\gamma(1+r)^2-1](\gamma+\sum_k\alpha_k^2)
	\]	
	which happens when
	\begin{align*}
		\left(\gamma(1+r')^2-K\right)\left([\gamma(1+r)^2-1]\sum_k\alpha_k^2 \right)
		&~
		-K\gamma^2(1+r)^2\\
		\ge &~
		\gamma(1+r')^2[\gamma(1+r)^2-1](\sum_k\alpha_k^2)
		-\gamma^2(1+r')^2\\
		\iff 
		\gamma^2[(1+r')^2
		- K(1+r)^2]
		\ge&~
		K
		\left([\gamma(1+r)^2-1]\sum_k\alpha_k^2 \right).
	\end{align*}
	i.e., when $\gamma\ge \gamma_*$ where $\gamma_*$ is the largest value of $\gamma$ for which equality holds in the above inequality. This follows since the coefficient of $\gamma^2$ on the LHS is positive, by our choice of $r,r'$.
	
	\noindent
	\textbf{Case 2 ($\rho=0$):} Once again using \eqref{eq:decom-pool} and \eqref{eq:decom-indiv} we get that in this case:
	\[
	(\bu^E_P)^\top \mathcal{A}_P^{-1} \bu^E_P=
	\frac{m^2_{F_{\gamma/K}}(-\lambda')\alpha_K^2}{m'_{F_{\gamma/K}}(-\lambda')} 
	\cdot
	\dfrac{ \alpha_K^2}{\alpha_K^2+\gamma}
	=
	\dfrac{\gamma(1+r')^2-K}{\gamma}\cdot 
	\dfrac{ \alpha_K^4}{\alpha_K^2+\gamma}
	\]
	and
	\[
	(\bu^E)^\top \mathcal{A}^{-1} \bu^E \\
	=
	\dfrac{m^2_{F_{\gamma}}(-\lambda)\alpha_K^2}{m'_{F_{\gamma}}(-\lambda)}
	\cdot
	\frac{\alpha_K^2}{\alpha_K^2+\gamma}
	=
	\dfrac{\gamma(1+r)^2-1}{\gamma}
	\cdot
	\frac{\alpha_K^4}{\alpha_K^2+\gamma}
	\]
	where we use the definitions of $D_{\rho,\tilde{\gamma}}$ and $D_{\rho,\gamma,ind}$ for $\rho=0$. Thus the pooled covariance matrix leads to the better estimator when 
	\[
	\gamma[(1+r')^2-(1+r)^2]>K-1.
	\]
\end{proof}



\medskip

\begin{proof}[Proof of Proposition~\ref{pr:heterog-cons-est-ykk}]
	Let us first discuss the estimation of $\calY_{Kk'}$ for $k'\neq K$. By Theorem 1 of \cite{serdobolskii2007multiparametric} we have the deterministic equivalence
	\[
	(\hat{\Sigma}_K + \lambda_{K} \II_p)^{-1} 
	\asymp 
	(x_p\Sigma_K +\lambda_K \II_p )^{-1}
	\]
	for $x_p=x(\gamma_K,\lambda_K)$ as specified earlier. Thus
	\begin{align*}
		&~\tr[(\hat{\Sigma}_K + \lambda_K \II_p)^{-1} (\hat{\Sigma}_{k'} + \lambda_{k'} \II_p)^{-1}\Sigma_K]/p\\
		=&~
		\tr[(x_p\Sigma_K +\lambda_K \II_p )^{-1} (\hat{\Sigma}_{k'} + \lambda_{k'} \II_p)^{-1}\Sigma_K]/p+\Omega_n\\
		=&~
		\frac{1}{px_p}
		\tr[(x_p\Sigma_K +\lambda_K \II_p )^{-1} (\hat{\Sigma}_{k'} + \lambda_{k'} \II_p)^{-1}
		(x_p\Sigma_K+\lambda_K\II_p - \lambda_K\II_p)]+\Omega_n\\
		=&~
		\frac{1}{px_p}
		\tr[(\hat{\Sigma}_{k'} + \lambda_{k'} \II_p)^{-1}]
		-
		\frac{\lambda_K}{px_p} 
		\tr[(x_p\Sigma_K +\lambda_K \II_p )^{-1} (\hat{\Sigma}_{k'} + \lambda_{k'} \II_p)^{-1}]
		+\Omega_n\\
		=&~
		\frac{1}{px_p}
		\tr[(\hat{\Sigma}_{k'} + \lambda_{k'} \II_p)^{-1}]
		-
		\frac{\lambda_K}{px_p} 
		\tr[(\hat{\Sigma}_K +\lambda_K \II_p )^{-1} (\hat{\Sigma}_{k'} + \lambda_{k'} \II_p)^{-1}]
		+\Omega_n
	\end{align*}
	for a sequence $\Omega_n\rightarrow_{a.s.} 0$ as $p,\,n_k\to\infty$ with $p/n_k=\gamma_k$ for $k=1,\dots,K$. Here the first and last equalities follow by the deterministic equivalent for $(\hat{\Sigma}_K + \lambda_{K} \II_p)^{-1} $ quoted above. For estimating $\calY_{kk'}$ where $k\neq K$ and $k'\neq K$, the result follows by the definition of $\hat{\calY}_{kk'}$ and then using Lemma~\ref{lem:samp-cov-DES} along with the independence of $\bX_k$ for $k=1,\dots,K$.
\end{proof}

\subsection{Technical Lemmas and Their Proofs}

\begin{lemma}\label{lem:samp-cov-DES}
	For a deterministic matrix $B\in \RR^{p\times p}$ with operator norm $\|B\|\le c$, the sample and population covariance matrices, $\hat \Sigma_k$ and $\Sigma_{k}$, for $k = 1, \cdots, K$; satisfy:
	\[
	\tr(B(\hat{\Sigma}_k-\Sigma_k))/p\rightarrow_{a.s.} 0.
	\]
\end{lemma}

\begin{proof}[Proof of Lemma~\ref{lem:samp-cov-DES}]
	Denoting the rows of $\bX_k$ and $\bZ_k$ by $X_{k,i}$ and $Z_{k,i}\in \RR^p$ for $i=1,\dots,n_k$, we have
	$X_{k,i}-\EE(X_{k,i})=\Sigma_k^{1/2}Z_{k,i}$, and hence
	\begin{align*}
		\tr(B(\hat{\Sigma}_k-\Sigma_k))/p
		=&~\frac{1}{pn_k}\sum_{i=1}^{n_k}
		\tr(B((X_{k,i}-\EE(X_{k,i}))(X_{k,i}-\EE(X_{k,i}))^{\top}-\Sigma_k))\\
		=&~\frac{1}{pn_k}\sum_{i=1}^{n_k}
		\left(Z_{k,i}^{\top}\Sigma_k^{1/2}B\Sigma_k^{1/2}Z_{k,i}
		-
		\EE [Z_{k,i}^{\top}\Sigma_k^{1/2}B\Sigma_k^{1/2}Z_{k,i}]
		\right)\\
		=&~\frac{1}{pn_k}\sum_{i=1}^{n_k} T_{k,i}
	\end{align*}
	where for $i=1,\dots,n_k$,
	\[
	T_{k,i}:= Z_{k,i}^{\top}\Sigma_k^{1/2}B\Sigma_k^{1/2}Z_{k,i}
	-
	\EE [Z_{k,i}^{\top}\Sigma_k^{1/2}B\Sigma_k^{1/2}Z_{k,i}]
	\]
	are i.i.d.\ random variables with $\EE T_{k,i}=0$. By assumption~\ref{as:HRMT}, we have
	\[
	\EE [Z_{k,i}^{\top}\Sigma_k^{1/2}B\Sigma_k^{1/2}Z_{k,i}]
	=\tr(\Sigma_k^{1/2}B\Sigma_k^{1/2}).
	\]
	We now compute the variance, writing $B'=\Sigma_k^{1/2}B\Sigma_k^{1/2}$ and using moment assumptions on $Z$ from assumption~\ref{as:HRMT}:
	\begin{align*}
		&~\Var(T_{k,i})\\
		=&~\Var\left(\sum_{l_1,l_2}B'_{l_1l_2}Z_{l_1}Z_{l_2}\right)\\
		=&~
		\sum_{l_1,l_2,l_3,l_4}
		\EE(B'_{l_1l_2}B'_{l_3l_4}Z_{l_1}Z_{l_2}Z_{l_3}Z_{l_4})
		-(\tr(B'))^2\\
		=&~
		\sum_{l_1=1}^p\sum_{l_3=1}^p (B'_{l_1l_1})(B'_{l_3l_3})
		+
		\sum_{l_1=1}^p\sum_{l_2=1}^p (B'_{l_1l_2})^2
		+
		\sum_{l_1=1}^p\sum_{l_2=1}^p (B'_{l_1l_2})(B'_{l_2l_1})
		+
		\sum_{l_1=1}^p(B'_{l_1l_1})^2(\EE Z_{l_1}^4-1)
		-(\tr(B'))^2\\
		=&~
		2\|B'\|_{\rm F}^2+
		\sum_{l_1=1}^p(B'_{l_1l_1})^2(\EE Z_{l_1}^4-1)
		\le Cp
	\end{align*}	
	for a constant $C>0$. The last line follows since 
	\[
	\|B'\|_{\rm F}^2
	\le 
	p\|\Sigma_k^{1/2}B\Sigma_k^{1/2}\|^2
	\le Cp
	\]
	due to our assumption on $\|B\|$. Thus by Chebyshev inequality, for any $t>0$, we have
	\begin{align*}
		\PP\left(
		\abs*{\frac{1}{pn_k}\sum_{i=1}^{n_k} T_{k,i}} > t
		\right)
		\le&~
		\frac{1}{p^2n_kt^2}\Var(T_{k,i})
		\le \frac{C}{pn_kt^2}
		=\frac{C}{\gamma_kn_k^2t^2}
		,
	\end{align*}
	from where the almost sure convergence follows via Borel-Cantelli lemma.
\end{proof}

\begin{lemma}\label{lem:quadConv}
	Under assumptions \ref{as:RMT}-\ref{as:CCW}, consider a matrix $\bZ\in \RR^{p\times p}$ whose entries $Z_{ij}$ have finite $(8 + \epsilon)$-th moment for some $\epsilon > 0$. Suppose further that $\bZ$ is independent of $\delta_k$, $k, k' = 1, \cdots, K; k \neq k'$, we have as $p \rightarrow \infty$
	\[\delta_k^{\top} A \delta_{k'} - \rho_{k  k'} \alpha_k \alpha_{k'} \tr (\bZ) / p \rightarrow_{a.s.} 0 \] 
	\[\delta_k^{\top} A \delta_{k} -  \alpha_k^2 \tr (\bZ) / p \rightarrow_{a.s.} 0 \] 
\end{lemma}

\begin{proof}[Proof of Lemma~\ref{lem:quadConv}]
	When $k = k'$, this lemma is the same as Lemma C.3 in \citet{dobriban2018high} and theorem 2 in \citet{sheng2020}. When $k \neq k'$, the same results still holds trivially under the bounded moments condition of $A$. This result has already been used by theorem 3.1 and theorem 4.1 of \citet{zhao2023cross}.
\end{proof}

\medskip

\begin{lemma}\label{lem:cvgHat}
	Under the assumption \ref{as:RMT}, recall the definition of sample covariance matrix $\hat{\Sigma} = (\bX-\mathbf{1}_{n}\bar{X}^{\top})^{\top} (\bX -\mathbf{1}_{n}\bar{X}^{\top})/ n$ and its companion $\underline{\hat{\Sigma}} = (\bX-\mathbf{1}_{n}\bar{X}^{\top}) (\bX-\mathbf{1}_{n}\bar{X}^{\top})^{\top} / p$; we have
	\begin{align*}
		\tr[ (\hat{\Sigma} + \lambda \II_p)^{-1} ]/ p& \rightarrow_{a.s.} m_{F_\gamma}(-\lambda) \\
		\tr [(\underline{\hat{\Sigma}} + \lambda \II_p)^{-1}] / n& \rightarrow_{a.s.} v_{F_\gamma}(-\lambda)\\	
		\tr[(\hat{\Sigma} + \lambda \II_p)^{-2}] / p& \rightarrow_{a.s.} m'_{F_\gamma}(-\lambda) \\
		\tr [(\underline{\hat{\Sigma}} + \lambda \II_p)^{-2}] / n& \rightarrow_{a.s.} v'_{F_\gamma}(-\lambda) \\
		\tr [(\hat{\Sigma} + \lambda  \II_p)^{-1} \Sigma] / p & \rightarrow_{a.s.} \frac{1}{\gamma} \left(\frac{1}{\lambda v_{F_\gamma}(-\lambda)} -1 \right) \\
				\tr [(\hat{\Sigma} + \lambda  \II_p)^{-2} \Sigma] / p & \rightarrow_{a.s.} 
				\frac{1}{\gamma} 
				\left(\frac{v_{F_\gamma}(-\lambda) - \lambda v'_{F_\gamma}(-\lambda)}{[\lambda v_{F_\gamma}(-\lambda)]^2}\right) \\
		\tr [(\hat{\Sigma} + \lambda  \II_p)^{-2} \Sigma^2] /p & \rightarrow_{a.s.} \frac{1}{\gamma} \left(\frac{v'_{F_\gamma}(-\lambda) -v^2_{F_\gamma}(-\lambda)}{\lambda^2 v^4_{F_\gamma}(-\lambda)}\right).
	\end{align*}
\end{lemma}
\begin{proof}[Proof of Lemma~\ref{lem:cvgHat}]
	The first four convergence statements follow from \citet{MPlaw} and \citet{silverstein1995strong}. The convergence of last three trace terms are from Lemma 2 of \citet{ledoit2011eigenvectors}, Lemma 2.2 of \citet{dobriban2018high} and Lemma 3.11 of \citet{dobriban2018high}.
\end{proof}

\medskip

\begin{lemma}\label{lem:aijasp}
	Assume $n_1, \cdots, n_K, p \rightarrow \infty$, $p / n_k \rightarrow \gamma_k$ for $k = 1, \cdots, K$ and assumption \ref{as:RMT}. We have 
	\[E_{kk'} := \tr[(\hat{\Sigma}_k + \lambda_{k} \II_p)^{-1} (\hat{\Sigma}_{k'} + \lambda_{k'} \II_p)^{-1}] / p \rightarrow_{a.s.} \mathcal{E}_{k k'}.\]	
	(1).
	Assume $n_1 = \cdots = n_K = n$, so $\gamma_1 = \cdots = \gamma_K = \gamma$, and use $\lambda_1 = \cdots = \lambda_K = \lambda$. Thus for all $k$, $m_{F_{\gamma_k}}(\lambda_k) = m_{F_{\gamma}}(\lambda)$. We have for $k \neq k'$
	\[\mathcal{E}_{k k'} = \frac{(1 - \gamma) m_{F_{\gamma}}'(-\lambda) + 2 \gamma \lambda m_{F_{\gamma}}(-\lambda) m_{F_{\gamma}}'(-\lambda) - \gamma m_{F_{\gamma}}(-\lambda)^2}{1 - \gamma +\gamma \lambda^2 m_{F_{\gamma}}'(-\lambda)}. \]\\
	(2). 
	Under  the assumption $\Sigma = \II_p$,
	\[\mathcal{E}_{k k'} = m_{F_{\gamma_k}}(-\lambda_k)  m_{F_{\gamma_{k'}}}(-\lambda_{k'}).\]
	(3). 
	Under the Assumptions \ref{as:moment} and \ref{as:aniso}, we have
	\begin{eqnarray*}
		\mathcal{E}_{k k'}& = & \frac{1}{\lambda_k \lambda_{k'}} \left\{ \lambda_{k} m_{F_{\gamma_k}}(-\lambda_k) + \lambda_{k'} m_{F_{\gamma_{k'}}}(-\lambda_{k'})  
		+ \frac{\lambda_k m_{F_{\gamma_k}}(-\lambda_k) m_{F_{\gamma_{k'}}}(-\lambda_{k'})}{ (m_{F_{\gamma_k}}(-\lambda_k) - m_{F_{\gamma_{k'}}}(-\lambda_{k'}))} \right. \\
		&& -\left.  \frac{\lambda_{k'} m_{F_{\gamma_k}}(-\lambda_k) m_{F_{\gamma_{k'}}}(-\lambda_{k'})}{ (m_{F_{\gamma_k}}(-\lambda_k) - m_{F_{\gamma_{k'}}}(-\lambda_{k'}))} \right\}.
	\end{eqnarray*}
\end{lemma}
\begin{proof}[Proof of Lemma~\ref{lem:aijasp}]
	In the first case, when $n_1 = \cdots = n_K = n$ so $\gamma_1 = \cdots = \gamma_K = \gamma$ and use $\lambda_1 = \cdots = \lambda_K = \lambda$, the limits of term $E_{k k'}$ has been found in the proof theorem 3 in \citet{sheng2020}. When $\Sigma = \II_p$, the term $E_{k k'}$ boils down to
	\[\tr[(\hat{\Sigma}_k + \lambda_{k} \II_p)^{-1} (\hat{\Sigma}_{k'} + \lambda_{k'} \II_p)^{-1}] / p = \tr [ (\bZ_k^{\top} \bZ_k / n_k + \lambda_k \II_p)^{-1} (\bZ_{k'}^{\top} \bZ_{k'} / n_{k'} + \lambda_{k'} \II_p)^{-1}]/p]\]
	Note we always have 
	\[E_{k k'} \rightarrow_{a.s.}  \EE_H\frac{1}{(x_k T + \lambda_k) (x_{k'} T + \lambda_{k'})}\]
	Recall $H$ is the limiting population spectral distribution, and $x_k$ is the fixed point solution to
	\begin{equation}\label{eq:xk-fixed-pt}
		1 - x_k = \gamma_k \left[1 - \lambda_k \int \frac{1}{x_k t + \lambda_k} dH(t)\right]
	\end{equation}
	When $\Sigma = \II_p$, $H$ only has a point mass on $1$ so the expectation decomposes and 
	\[E_{k k'} \rightarrow_{a.s.} m_{F_{\gamma_k}} (-\lambda_k) m_{F_{\gamma_{k'}}} (-\lambda_{k'}) \]
	As an alternative proof when we have Assumption~\ref{as:moment}; we know $\bZ_k$ will be asymptotically free from any bounded constant matrices, this is a standard result, ex. theorems 5.4.5 in \citet{anderson2010introduction}. Further, we know sample covariances of the form $\bZ^{\top}_k \bZ_k / n_k$ is asymptotically free from $\bZ^{\top}_{k'} \bZ_{k'} / n_{k'}$ [see \cite{capitaine2004asymptotic}]. Two arguments combined suggests that $(\bZ_k^{\top} \bZ_k / n_k + \lambda_k \II_p)^{-1}$ is asymptotically free from $(\bZ_{k'}^{\top} \bZ_{k'} / n_{k'} + \lambda_{k'} \II_p)^{-1}$ therefore,
	\[E_{k k'} - m_{F_{\gamma_k}} (-\lambda_k) m_{F_{\gamma_{k'}}} (-\lambda_{k'}) \rightarrow_{a.s.} 0\]
	For the third case, a slight generalization of Corollary 3.9 of \citet{knowles2017anisotropic} tells us
	\[\ell_1^{\top} \left[(\hat{\Sigma}_k + \lambda_k \II_p)^{-1} - \frac{1}{\lambda_k (1 + m_{F_{\gamma_k}}(-\lambda) \Sigma)}\right] \ell_2 \rightarrow_{a.s.} 0\]
	where $\ell_1, \ell_2$ can be any continuous random vectors independent from $(\hat{\Sigma}_k + \lambda_k \II_p)^{-1} $. We can decompose $E_{k k'}$ by 
	\begin{align*}
		\underbrace{\tr [ (\hat{\Sigma}_k + \lambda_k \II_p)^{-1} (\hat{\Sigma}_{k'} + \lambda_{k'} \II_p)^{-1}]/p}_{E_{k k'}} - \frac{1}{\lambda_k } \sum_{i = 1}^{p} \ell_{1, i}^{\top} (\II_p + m_{F_{\gamma_k}}(-\lambda) \Sigma)^{-1} \ell_{2, i} /p \rightarrow_{a.s.} 0
	\end{align*}
	where $\ell_{1, i}$ is $e_i$ with $1$ in its $i^{th}$ entry and $0$ else where; and $\ell_{2, i} := (\hat{\Sigma}_{k'} + \lambda_{k'} \II_p)^{-1} e_i$. From now on, simplify the notation by using $m_k := m_{F_{\gamma_k}}(-\lambda_k)$ and $m_{k'} := m_{F_{\gamma_{k'}}}(-\lambda_{k'})$. Perform the similar trick to $\ell_{2, i}$, we have  
	\begin{align*}
		E_{k k'}  - \frac{1}{\lambda_k \lambda_{k'}} \tr( (\II_p + m_k \Sigma)^{-1} (\II_p + m_{k'}  \Sigma)^{-1}) / p \rightarrow_{a.s.} 0
	\end{align*}
	In addition
	\begin{align*}
		&~ \frac{1}{\lambda_k \lambda_{k'}} \tr( (\II_p + m_k \Sigma)^{-1} (\II_p + m_{k'}  \Sigma)^{-1}) / p \\
		=&~  \frac{1}{\lambda_k \lambda_{k'}} [1 - m_k \tr ((\II_p + m_k \Sigma)^{-1} \Sigma) / p - m_{k'} \tr((\II_p + m_{k'} \Sigma)^{-1} \Sigma) / p\\ 
		&~+   m_k m_{k'} \tr ((\II_p + m_k \Sigma)^{-1} \Sigma (\II_p + m_{k'} \Sigma)^{-1} \Sigma)]/ p
	\end{align*}
	where we used the matrix identity
	\[(\II_p + m_{F_{\gamma_k}}(-\lambda_k) \Sigma)^{-1} = \II_p - m_{F_{\gamma_k}}(-\lambda_k) (\II_p + m_k(-\lambda_k) \Sigma)^{-1} \Sigma\] Each of the terms can be expressed in empirical quantities by
	\begin{align*}
		\tr ((\II_p + m_k \Sigma)^{-1} \Sigma)/ p  -  \EE_H \frac{1}{m_k(1 + t m_k)} 
		+  \EE_H \frac{1}{m_k (1 + t m_k)} - \frac{1}{m_k} [1 - \lambda_k m_k] 
		\rightarrow_{a.s.}  0
	\end{align*}
	With the same techniques, we get
	\begin{align*}
		& \tr ((\II_p + m_k \Sigma)^{-1} \Sigma (\II_p + m_{k'} \Sigma)^{-1} \Sigma)/ p  \\
		\rightarrow_{a.s.} & \frac{\lambda_k m_k}{m_k (m_k - m_{k'})} 
		-  \frac{\lambda_{k'} m_{k'}}{m_{k'}  (m_k - m_{k'})} 
		+  \frac{1}{m_k m_{k'}}
	\end{align*}
	Substitute these expressions back into the expressions for $E_{k k'}$ finishes the proof.
\end{proof}

\medskip

\begin{lemma}\label{lem:predijasp}
	With assumption  \ref{as:RMT}, and under $n_k, p \rightarrow \infty$, $p / n_k \rightarrow \gamma_k$,  we have the following convergence results
		\[\tr[(\hat{\Sigma}_k + \lambda_{k} \II_p)^{-1} (\hat{\Sigma}_{k'} + \lambda_{k'} \II_p)^{-1}\Sigma] / p \rightarrow_{a.s.} \mathcal{M}_{k k'}.\]	
		For $k \neq k'; k, k'\in \{1, \cdots, K\}$, we have:
		\begin{enumerate}
			\item Assume $n_1 = \cdots = n_K = n$, so $\gamma_1 = \cdots = \gamma_K = \gamma$, and use $\lambda_1 = \cdots = \lambda_K = \lambda$. 
			\[\mathcal{M}_{k k'} =  \frac{m_{F_{\gamma}}(-\lambda) - \lambda m'_{F_{\gamma}}(-\lambda)}{1 -  \gamma +  \gamma \lambda^2 m'_{F_{\gamma}}(-\lambda)}. \]
			\item Under the assumption $\Sigma = \II_p$,
			\[\mathcal{M}_{k k'} =   m_{F_{\gamma_k}} (-\lambda_k)  m_{F_{\gamma_{k'}}} (-\lambda_{k'}). \]
			\item Under Assumptions \ref{as:moment} and  \ref{as:aniso}, 
			\[\mathcal{M}_{k k'} = \frac{\lambda_k m_{F_{\gamma_k}} (-\lambda_k) - \lambda_{k'} m_{F_{\gamma_{k'}}} (-\lambda_{k'})}{\lambda_{k} \lambda_{k'} (m_{F_{\gamma_{k'}}} (-\lambda_{k'}) - m_{F_{\gamma_k}} (-\lambda_k))}. \]
		\end{enumerate}
\end{lemma}
\begin{proof}[Proof of Lemma~\ref{lem:predijasp}]
	The proof is similar to the proof of lemma \ref{lem:aijasp}. In the first case where $n_1 = \cdots = n_K = n$ so $\gamma_1 = \cdots = \gamma_K = \gamma$ and use $\lambda_1 = \cdots = \lambda_K = \lambda$, we have
	\[\tr[(\hat{\Sigma}_k + \lambda_{k} \II_p)^{-1} (\hat{\Sigma}_{k'} + \lambda_{k'} \II_p)^{-1}\Sigma] / p \rightarrow_{a.s.} \int \frac{t}{(xt +\lambda)^2}dH(t)\]
	When $\gamma$ is equal across all populations, we will use the shorter notation $m := m_{F_{\gamma}} (-\lambda), m' = m'_{F_{\gamma}} (-\lambda)$ in this supplement. By definitions, we have
	\[\int \frac{1}{x t + \lambda}dH(t) := m 
	\text{ and }
	 \int \frac{x't + 1}{(x t + \lambda)^2}dH(t) := m' .\]
	Here $x :=x_k $ is the solution to the fixed point equation in \eqref{eq:xk-fixed-pt}, and $x'$ is the derivative of $x$ with respect to $\lambda$. Then 
	\[m' = \int \frac{ (xt + \lambda - \lambda)\frac{x'}{x} + 1}{(x t + \lambda)^2}dH(t) = \frac{x'}{x} m + (1 - \frac{\lambda x'}{x})  \int \frac{1}{(x t + \lambda)^2}dH(t) \]
	\[\int \frac{1}{(x t + \lambda)^2}dH(t) = \frac{xm' - x'm}{x - \lambda x'} \]
	So the functional of interest is
	\begin{align*}
		 \int \frac{t}{(xt +\lambda)^2} d H(t) 
		=&~ \frac{\int \frac{x' t + 1}{(xt +\lambda)^2} d H(t) -  \int \frac{1}{(xt +\lambda)^2} d H(t)}{x'} \\
		=&~ \frac{m' -  \frac{xm' - x'm}{x - \lambda x'} }{x'} 
		= \frac{\frac{m'x - \lambda m' x' - xm' + x'm}{x - \lambda x'} }{x'} \\
		=&~ \frac{m - \lambda m' }{x - \lambda x'}  
		=~ \frac{m - \lambda m'}{1 -  \gamma +  \gamma \lambda^2 m'}.
	\end{align*}
	For the second case, when $\Sigma = \II_p$, $\mathcal{M}_{k k'} = \mathcal{E}_{k k'}$. So the proof follows from lemma \ref{lem:aijasp}. For the third case, pulling the trick with results from \citet{knowles2017anisotropic} again gives
	\begin{align*}
		 \tr(\Sigma (\hat{\Sigma}_k + \lambda_k \II_p)^{-1}  (\hat{\Sigma}_{k'} + \lambda_{k'} \II_p)^{-1}) / p 
		\rightarrow_{a.s.} & \frac{1}{\lambda_k \lambda_{k'}} \EE_H \frac{t}{(1 + t m_k) (1 + t m_{k'})} 
	\end{align*}
	which can be consistently estimated by
	\[\frac{1}{\lambda_k \lambda_{k'}}  \frac{\lambda_k m_k - \lambda_{k'} m_{k'}}{m_{k'} - m_k }\]
	as claimed by the lemma.
\end{proof}


\subsection{Proofs of Theorems}

\begin{proof} [Proof of Theorem \ref{th:tl-rda-ErrorAsp}] 
	Firstly, when $\pi_{-} = \pi_{+}$, we have 
	\[Err(\bw) = \Phi\left(\frac{(\hat{d}(\bw))^{\top} \mu_{-1}+ \hat{b}_K}{\sqrt{(\hat{d}(\bw))^{\top} \Sigma (\hat{d}(\bw))}}\right)\]
	Under Assumptions \ref{as:RCW} and \ref{as:CCW}; by Lemma 3.7 in \citet{dobriban2018high}, we know $\hat{b} \rightarrow_{a.s.} 0$. By  Lemma 3.8 in the same paper, we can use the almost sure limit $ (\hat{d}(\bw))^{\top} \mu_{-1} \rightarrow_{a.s.} (\hat{d}(\bw))^{\top} \delta_{K}$ and arrive at
	\[Err(\bw) \rightarrow_{a.s.} \Phi\left(\frac{(\hat{d}(\bw))^{\top} \delta_K}{\sqrt{(\hat{d}(\bw))^{\top} \Sigma (\hat{d}(\bw))}}\right).
	\]
	Observe that 
	\[(\hat{d}(\bw))^{\top} \delta_K = \bw^{\top} \hat{\bu}, \ \ \ \hat{\bu} = vec\left[\hat{\delta}_{k} (\hat{\Sigma}_k + \lambda_k)^{-1}  \delta_K\right]\]
	\[(\hat{d}(\bw))^{\top} \Sigma (\hat{d}(\bw)) = \bw^{\top} \hat{\calA} \bw, \ \ \ 
	\hat{\calA} = mat \left[ \hat{\delta}_{k}^{\top} (\hat{\Sigma}_k + \lambda_k\II_p)^{-1}  \Sigma  (\hat{\Sigma}_{k'} + \lambda_{k'}\II_p)^{-1} \hat{\delta}_{k'} \right ]\]
	Here $vec[\cdot], mat[\cdot]$ are the vector operator and matrix operator respectively. As argued in the proof of Lemma 3.7 in \citet{dobriban2018high}, one can decompose $\hat{\delta}_{k}$ in the $k^{th}$ population as 
	\[\hat{\delta}_{k} = \delta_{k} + \frac{1}{\sqrt{n_k}} \Sigma^{1/2} \tilde{\bZ}_k\]
	where $\tilde{\bZ}_k \in \RR^p$ are standard normal random vectors independent of $\bX_{k'}$ conditionally on $\mu_{\pm 1, k'}, \delta_{k'}$ for all $k' = 1,\cdots, K$. Thus, Lemma \ref{lem:quadConv} and Lemma \ref{lem:cvgHat} gives us
	\[\hat{\bu} \rightarrow_{a.s.} \rho_{k K} 
	\alpha_{k} \alpha_{K} \ vec\left[ m_{F_{\gamma_k}}(-\lambda_{k}) \right]\]
	We decompose $\hat{\calA}$ into three parts for further analysis, such that
	\[\hat{\calA}_{k k'} = \hat{\delta}_{k}^\top (\hat{\Sigma}_k + \lambda_k)^{-1}  \Sigma  (\hat{\Sigma}_{k'} + \lambda_{k'})^{-1} \hat{\delta}_{k'} = \tilde{A}_{k k'} + 2 \tilde{B}_{k k'} + \tilde{C}_{k k'} \]
	where 
	\begin{align*}
		M^{(kk')} &:= ( \hat{\Sigma}_k  + \lambda_k \II_p)^{-1} \Sigma (\hat{\Sigma}_{k'} + \lambda_{k'} \II_p)^{-1} \\
		\tilde{A}_{k k'} & := \delta^{\top}_k M^{(kk')} \delta_{k'} \\
		\tilde{B}_{k k'} & := \delta_{k}^\top  M^{(kk')} (\hat{\delta}_{k'} - \delta_{k'})\\
		\tilde{C}_{k k'} & := (\hat{\delta}_{k} - \delta_{k})^{\top} M^{(kk')}  (\hat{\delta}^{\top}_{k'} - \delta_{k'})
	\end{align*}
	Firstly, one can show $\tilde{B}_{k k'} \rightarrow_{a.s.} 0$ with the same techniques used to prove $\hat{b} \rightarrow_{a.s.} 0$. For the off-diagonal terms in $\tilde{A}$, one can again invoke Lemma~\ref{lem:quadConv} and Lemma~\ref{lem:cvgHat} to show 
	\[
	\tilde{A}_{k k'} -\rho_{k  k'} \alpha_{k} \alpha_{k'} \tr(M^{(kk')}) / p\rightarrow_{a.s.} 0.
	\]
	Next, the limit of $\tr(M^{(kk')}) / p$ is provided by Lemma~\ref{lem:predijasp}. For the diagonal terms $\tilde{A}_{k k}$, one can simply read off the limit of $\tr(M^{(kk)}) / p$ from Lemma~\ref{lem:cvgHat}. Finally, we have the following equality for $\tilde{C}$ based on decomposing $\hat{\delta}_{k}$:
	\begin{align*}
		\tilde{C}_{k k'} & = \frac{1}{n_k} \tilde{Z}_k^\top \Sigma^{1/2}   M^{(kk')} \Sigma^{1/2}  \tilde{Z}_{k'}
	\end{align*}
	We know $\tilde{C}_{k k'} \rightarrow_{a.s.} 0$ for $k \neq k'$ based on Lemma~\ref{lem:quadConv} due to the independence between $\tilde{Z}_k, \tilde{Z}_{k'}$. For the diagonal terms of $\tilde{C}$, we have 
	\[	\tilde{C}_{k k} -\gamma_k \tr (\Sigma M^{(kk)}) / p
	\rightarrow_{a.s.} 0
	\]
	and the limit of the trace of $\Sigma M^{(kk)}$ can again be read off from Lemma~\ref{lem:cvgHat}.
\end{proof}

\medskip

\begin{proof}[Proof of Theorem~\ref{th:tl-rda-MinEst}] Recall the decomposition $\hat{\delta}_{k} = \delta_{k} + \frac{1}{\sqrt{n_k}} \Sigma^{1/2} \tilde{\bZ}_k$, we then have from the objective in \eqref{eq:CrEst}, that
\begin{align*}
	\left\Vert \sum_{k = 1}^{K} w_k (\hat{\Sigma}_k + \lambda_k \II_p)^{-1} \hat{\delta}_k - \Sigma^{-1} \delta_K\right\Vert_2^2 = \bw^{\top} \left(\hat{\calA}^E + \hat{\calR}^E \right) \bw 
	- 2 (\hat{\bu}^E)^{\top} \bw + \Vert \delta_K^{\top}\Sigma^{-1} \Vert_2^2 
\end{align*}
where
\begin{align*}	
\hat{\calA}^E 
=&~mat \left[\delta_{k}^{\top} (\hat{\Sigma}_k + \lambda_k \II_p)^{-1} (\hat{\Sigma}_{k'} + \lambda_{k'} \II_p)^{-1} \delta_{k'} \right]\\
\hat{\calR}^E =&~ mat \left[ (\hat{\delta}_k - \delta_k)^\top (\hat{\Sigma}_k + \lambda_k \II_p)^{-1} (\hat{\Sigma}_{k'} + \lambda_{k'} \II_p)^{-1} (\hat{\delta}_{k'} - \delta_{k'}) \right] \\
\hat{\bu}^E =&~ \delta_K^{\top}\Sigma^{-1} vec\left[(\hat{\Sigma}_{k} + \lambda_{k} \II_p)^{-1} \delta_{k'} \right]
\end{align*}
Taking the derivative with respect to $\bw$, we get the finite-sample expression for optimal estimation weight
\[\hat{\bw}^E = \left(\hat{\calA}^E + \hat{\calR}^E  \right)^{-1} \hat{\bu}^E\]
Taking $n,p \rightarrow \infty; p / n_k \rightarrow \gamma_k$, we can get the asymptotic expressions for each of the three terms above. Consider the linear term $\hat{\bu}^E$ first. By Lemma~\ref{lem:quadConv}
\[\hat{\bu}^E \rightarrow_{a.s.}  \bu^E = vec \left[\rho_{k K} \alpha_k \alpha_K \tr[ \Sigma^{-1} (\hat{\Sigma}_{k} + \lambda_k \II_p)^{-1}]/ p \right] 
\]
Again, the \textit{anisotropic local law} tells us that:
\[\tr(\Sigma^{-1} (\hat{\Sigma}_{c, k} + \lambda_k \II_p)^{-1}) / p
\rightarrow_{a.s.}  \frac{1}{\lambda_k} \EE_H \left[\frac{1}{t (1 + t m_k(-\lambda_k)) }\right] 
= \frac{1}{\lambda_k} \EE_H \left[ \frac{1}{t} - \frac{m_k(-\lambda_k)}{1 + t m_k(-\lambda_k) }\right] \]
	\begin{align*}
		\tr(\Sigma^{-1} (\hat{\Sigma}_{c, k} + \lambda_k \II_p)^{-1}) / p
		\rightarrow &~ \frac{1}{\lambda_k} 
		\EE_H \left[\frac{1}{t [1 + t m_k(-\lambda_k)] }\right]  \label{eq1} \\
		= &~  \frac{1}{\lambda_k} \EE_H \left[ \frac{1}{t} - \frac{m_k(-\lambda_k)}{1 + t m_k(-\lambda_k) } \right] \\
		\rightarrow &~ \frac{1}{\lambda_k} \tr(\Sigma^{-1}) / p  - m_k(-\lambda_k) \tr [(\hat{\Sigma}_{c, k} + \lambda_k \II_p)^{-1}) ]/ p.
	\end{align*}
For diagonal terms of $\hat{\calA}^E$, we have 
\[\hat{\calA}^\EE_{kk} \rightarrow_{a.s.} \alpha_{k}^2 m_{F_{\gamma_k}}(-\lambda_k)\]
For off diagonal terms of $\hat{\calA}^E$, we have 
\[\hat{\calA}^\EE_{kk'} \rightarrow_{a.s.} \rho_{k  k'} \alpha_{k} \alpha_{k'} \mathcal{E}_{k k'} \]
where $\mathcal{E}_{k k'}$ can be found in Lemma~\ref{lem:aijasp}. Lastly, we firstly observe that the asymptotic off-diagonal terms of $\hat{\calR}^E$ are zero as $\tilde{\bZ}_k$ and $\tilde{\bZ}_{k'}$ are independent (see Lemma~\ref{lem:quadConv}). For diagonal terms of $R^E$, we have 
\[\hat{\calR}^E_{kk} = \tilde{\bZ}_k^\top \Sigma^{1/2} (\hat{\Sigma}_k + \lambda_k \II_p)^{-1} (\hat{\Sigma}_{k} + \lambda_{k} \II_p)^{-1} \Sigma^{1/2} \tilde{\bZ}_k \rightarrow_{a.s.} \frac{v_{F_{\gamma_k}}(-\lambda_{k}) - \lambda_k v'_{F_{\gamma_k}}(-\lambda_k)}{\lambda_{k} v_{F_{\gamma_k}}(-\lambda_{k})^2} \]
Gathering the asymptotic expressions together, along with the uniqueness of the asymptotic minimizer, we finish the proof.
\end{proof}

\medskip

\begin{proof}[Proof of Theorem~\ref{th:tl-rda-MinPred}] The proof is similar to the proof of Theorem~\ref{th:tl-rda-MinEst}. From the objective function in \eqref{eq:CrPred}, we have
\begin{align*}
	\EE_{x_0} \left\Vert [\sum_{k = 1}^{K} w_k (\hat{\Sigma}_c + \lambda_k \II_p)^{-1} \hat{\delta}_k - \Sigma^{-1} \delta_K]^{\top} x_0 \right\Vert_2^2=
	~ \bw^\top [\hat{\calA}^P + \hat{\calR}^P] \bw - 2 \bw^\top \hat{\bu}^P  
	 + \Vert\delta_K^\top \Sigma^{-1/2}\Vert_2^2 
	 \end{align*}
where
	\begin{align*}  
	\hat{\calA}^P =&~
	mat \left[\delta_{k}^{\top} (\hat{\Sigma}_k + \lambda_k \II_p)^{-1} \Sigma (\hat{\Sigma}_{k'} + \lambda_{k'} \II_p)^{-1} \delta_{k'} \right]\\
	\hat{\calR}^P =&~
	 mat \left[ (\hat{\delta}_k - \delta_k)^\top (\hat{\Sigma}_k + \lambda_k \II_p)^{-1}\Sigma (\hat{\Sigma}_{k'} + \lambda_{k'} \II_p)^{-1} (\hat{\delta}_{k'} - \delta_{k'}) \right] \\
	\hat{\bu}^P =&~ vec\left[\delta_K^{\top}  (\hat{\Sigma}_{k} + \lambda_{k} \II_p)^{-1} \delta_{k'} \right]
\end{align*}
We can again take the derivative with respect to $\bw$ and set it to zero, we get 
\[\hat{\bw}^P = \left( \hat{\calA}^P +  \hat{\calR}^P \right)^{-1} \hat{\bu}^P\]
Taking $n,p \rightarrow \infty; p / n_k \rightarrow \gamma_k$, we can get the expression for each of the three terms above. Consider the linear term $\hat{\bu}^P$, by Lemma~\ref{lem:quadConv}
\[\hat{\bu}^P \rightarrow_{a.s.}  vec \left[\rho_{k K} \alpha_k \alpha_K m_{F_{\gamma_k}}(-\lambda_{k}) \right].
\]
For diagonal terms of $\hat{\calA}^P$, we have 
\[\hat{\calA}^P_{kk} \rightarrow_{a.s.}  \frac{\alpha_{k}^2}{\gamma_k} \frac{v_{F_{\gamma_k}}(-\lambda_k) - \lambda v'_{F_{\gamma_k}}(-\lambda_k)}{[\lambda v_{F_{\gamma_k}}(-\lambda_k)]^2} \]
For off diagonal terms of $ \hat{\calA}^P$, we have 
\[ \hat{\calA}^P_{kk'} \rightarrow_{a.s.} \rho_{k  k'} \alpha_{k} \alpha_{k'} \mathcal{M}_{k k'}. \]
The off-diagonal terms of $ \hat{\calR}^P$ again converges to zero, and the diagonals are
\[ \hat{\calR}^P_{kk} = \tilde{\bZ}_k^\top \Sigma \Sigma^{1/2} (\hat{\Sigma}_k + \lambda_k \II_p)^{-1} (\hat{\Sigma}_{k} + \lambda_{k} \II_p)^{-1} \Sigma^{1/2} \tilde{\bZ}_k \rightarrow_{a.s.} \frac{v'_k(-\lambda_{k}) - v_k^2(-\lambda_k)}{\lambda_k^2 v^4_k(-\lambda_{k})}.
\]
Gathering the asymptotic expressions together, along with the uniqueness of the asymptotic minimizer, we finish the proof.
\end{proof}

\medskip

\begin{proof}[Proof of Theorem~\ref{th:tlH-rda-MinEst}]
	The proof is identical to that of Theorem~\ref{th:tl-rda-MinEst} with minimal changes due to the differing covariance matrices $\Sigma_k$.
\end{proof}

\medskip

\begin{proof}[Proof of Theorem~\ref{th:tlH-rda-MinPred}]
	The proof is identical to that of Theorem~\ref{th:tl-rda-MinPred} with minimal changes due to the differing covariance matrices $\Sigma_k$.
\end{proof}

\subsection{Proofs of Corollaries}

\begin{proof}[Proof of Corollary~\ref{cor:lim-err-indiv}]
	We follow the proof of Theorem~\ref{th:tl-rda-MinPred} for the special case $\Sigma = \II_p$. Consequently,
	\[
	\hat{\calA}^P_{kk'} 
	=
	\delta_{k}^{\top} (\hat{\Sigma}_k + \lambda_k \II_p)^{-1} \Sigma (\hat{\Sigma}_{k'} + \lambda_{k'} \II_p)^{-1} \delta_{k'}	
	\rightarrow_{a.s.}   
	\begin{cases}
	\alpha_k^2m'_k(-\lambda_k) &\text{if }k=k'\\
	\rho_{kk'}\alpha_k\alpha_{k'}
	m_k(-\lambda_k)m_{k'}(-\lambda_{k'}) &\text{otherwise}
	\end{cases}
	\]
	using the limits from Lemma~\ref{lem:cvgHat}. The rest of the expressions follow similarly.
\end{proof}

\section{Numerical Experiments} 
\subsection{Validation of Limiting Error Formulae}
In this subsection, we use synthetic data to verify the accuracy of the theoretical formula in Theorems~\ref{th:tl-rda-ErrorAsp}, \ref{th:tl-rda-MinEst}, \ref{th:tl-rda-MinPred}. We use a Toeplitz-type covariance matrix, $n_1,\cdots, n_K = 150, 140, 130, 120, 110, 100$ and $p = 150$. We use the same signal strength for all population $\alpha^2 = 0.5$ and the same weight correlation $\rho = 0.5$. The prediction error is based on $2000$ simulated testing data points. We then vary $\lambda_{k}$ from $0.3$ to $10$. For each $\lambda_k$, the simulation is repeated $50$ times. One can refer to figure \ref{indTheo} and figure \ref{PoolTheo} for the comparison between simulated testing error rate (boxplots) and theoretical error rate (red line) for TL-RDA and TL-RDA respectively. We can see that the limiting formulae are very accurate across the $\lambda_{k}$ considered.
\begin{figure}
	\includegraphics[width= 0.9\textwidth]{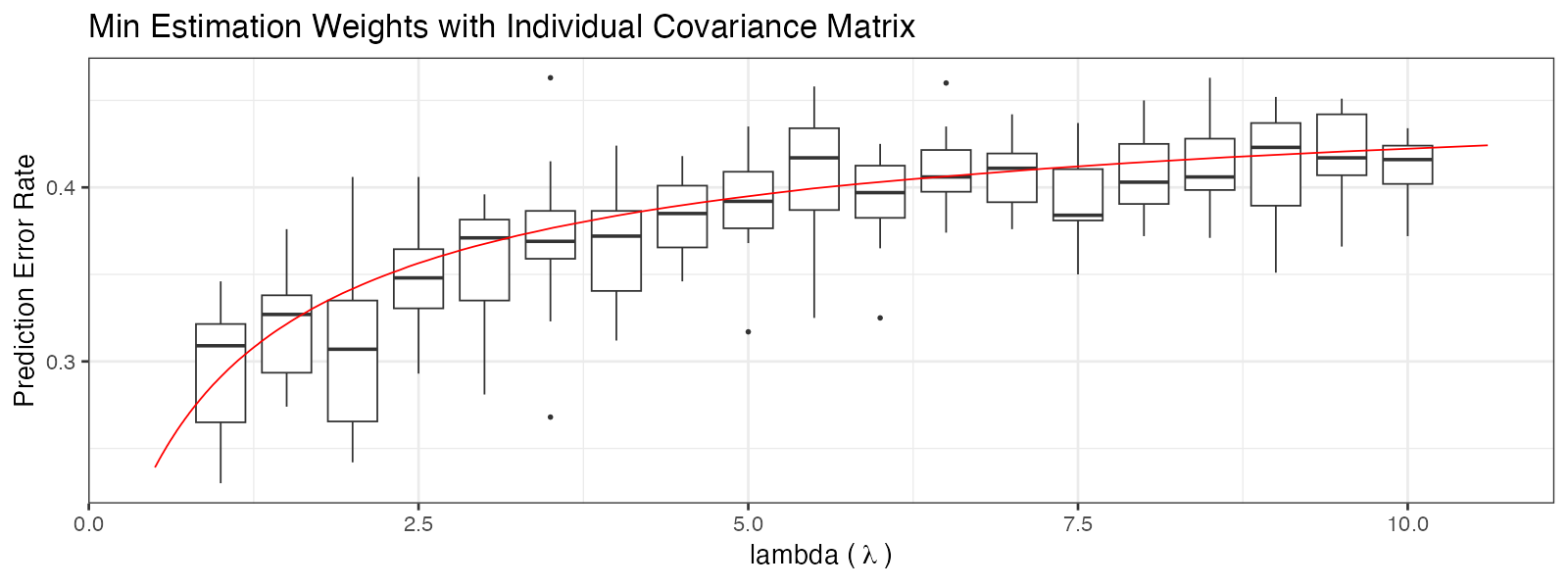}
	\centering
		\includegraphics[width= 0.9\textwidth]{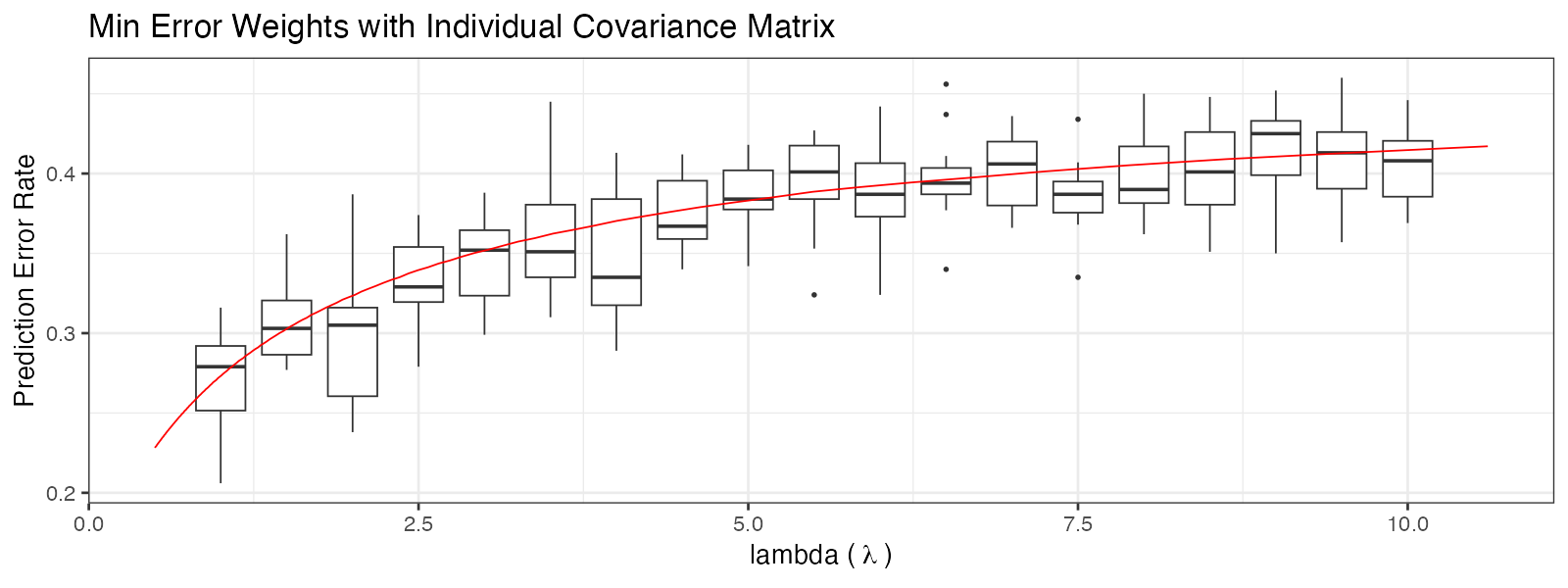}
	\centering
	\caption{Individual sample covariance matrices}
	\label{indTheo}
\end{figure}
\begin{figure}
	\includegraphics[width= 0.9\textwidth]{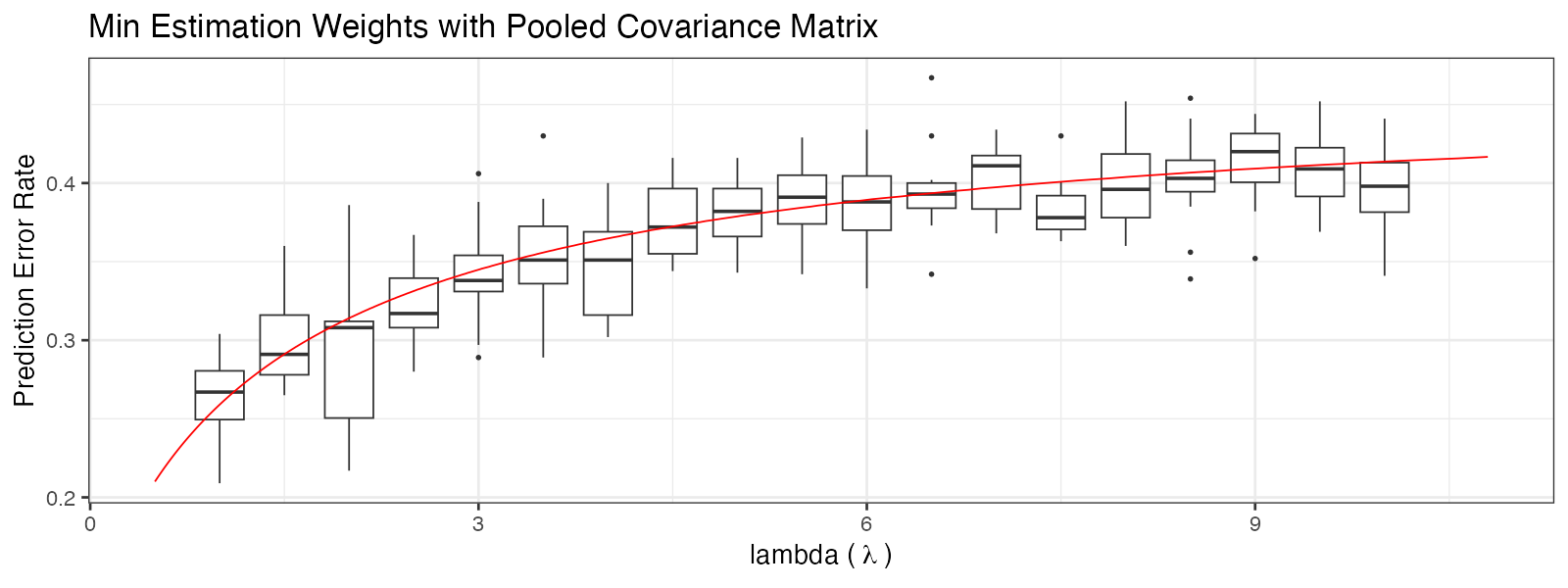}
	\centering
	\includegraphics[width= 0.9\textwidth]{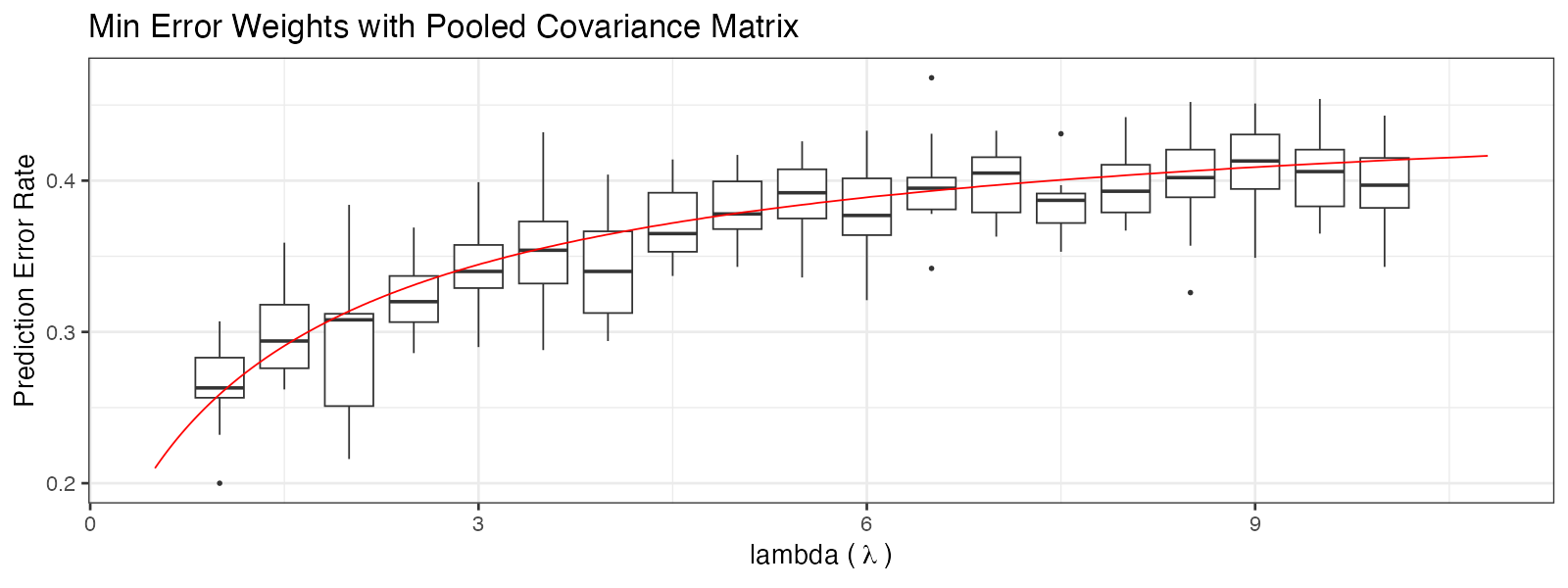}
	\centering
	\caption{Pooled sample covariance matrices}
	\label{PoolTheo}
\end{figure}
\subsection{Theoretical Error comparison}
Here we plot the limiting errors of the naive RDA, TL-RDA with estimation weight and TL-RDA with prediction weight under different scenarios. Note the naive RDA is fitted on target population data only. One can refer to figure \ref{indcompare} and figure \ref{poolcompare} for comparisons of TL-RDA and TLP-RDA respectively. We vary all parameters including correlation $\rho$, aspect ratio $\gamma$, signal strength $\alpha^2$, number of populations $K$ and population covariance matrix eigenvalues $\Sigma$. One can directly observe that TL estimators beat naive RDA in every set up considered, and not surprisingly, the prediction weight (minError) beats estimation weight (MinEst) in all set ups. Another interesting fact is estimation weight is usually not too much worse than prediction weight, especially in the TLP-RDA case. Lastly, we observe that the prediction errors turn to increase when $\gamma_k$s increase, decreases when $\alpha^2$ increases, decrease when the number of populations increases. In addition, the error rate decreases when the eigenvalue of $\Sigma$ decreases faster, that is, when the correlation strength between the features get stronger. 
\begin{figure}
	\includegraphics[width= 1\textwidth, height=0.32\textheight]{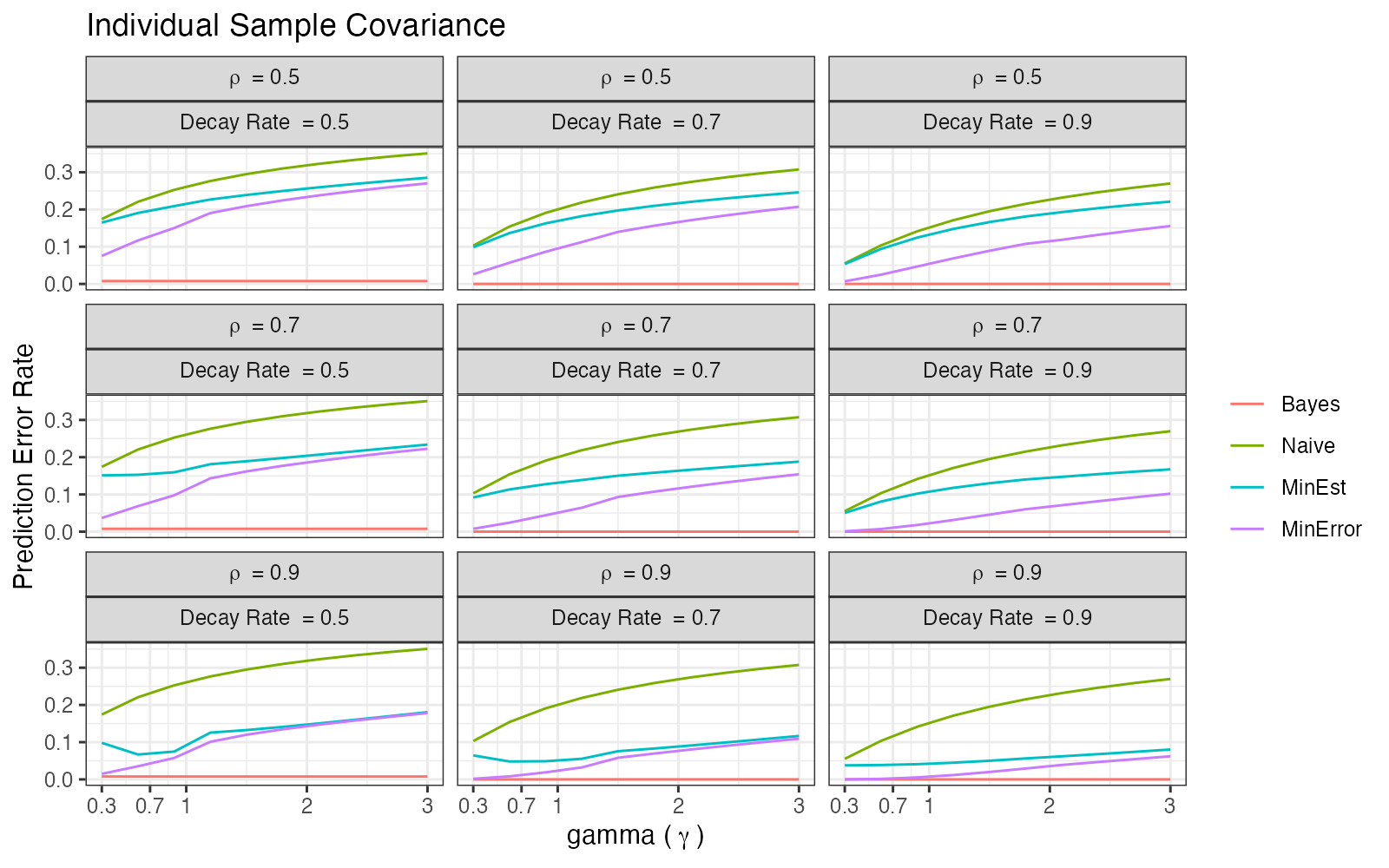}
	\centering
		\includegraphics[width= 1\textwidth,height=0.32\textheight]{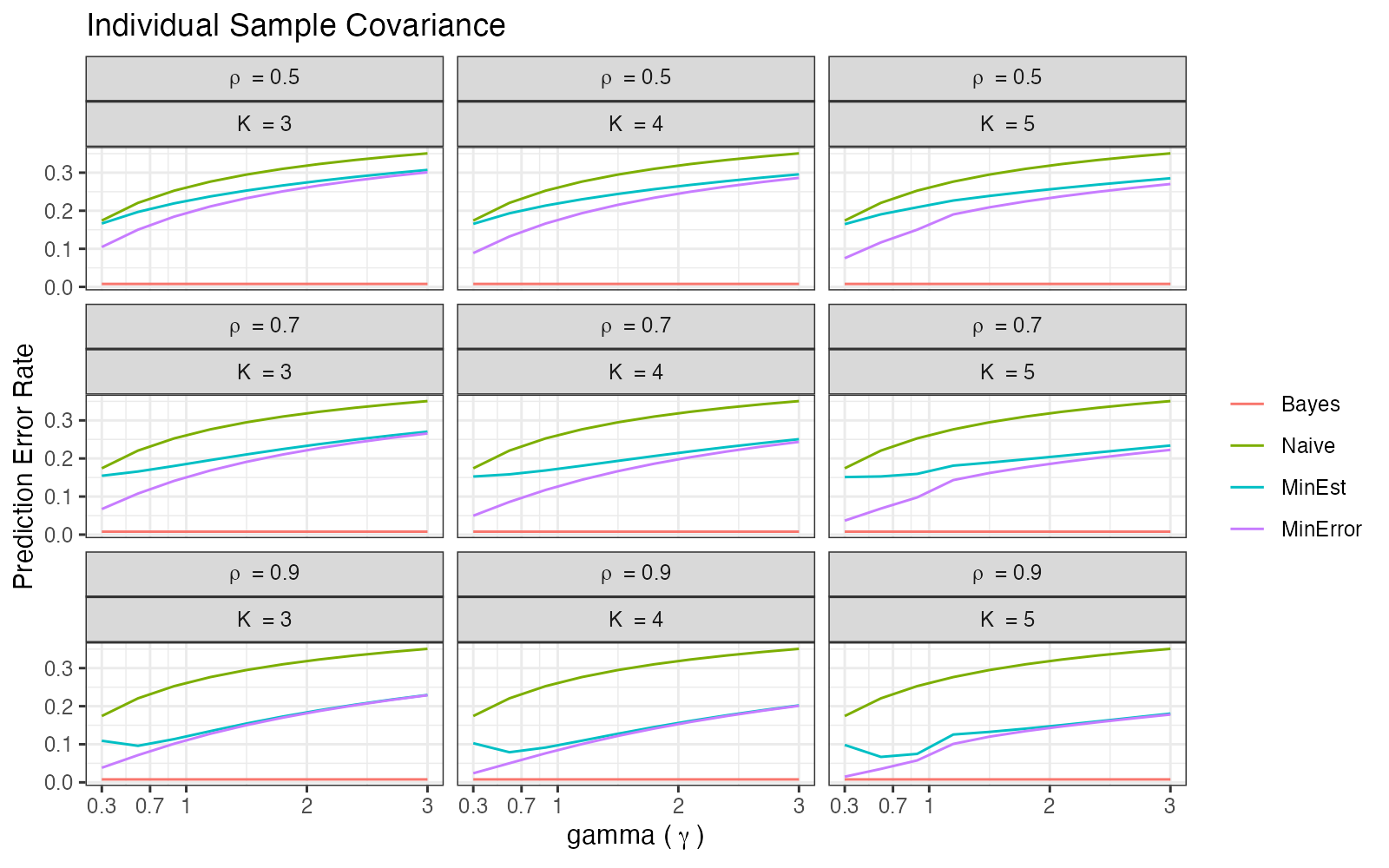}
	\centering
		\includegraphics[width= 1\textwidth,height=0.32\textheight]{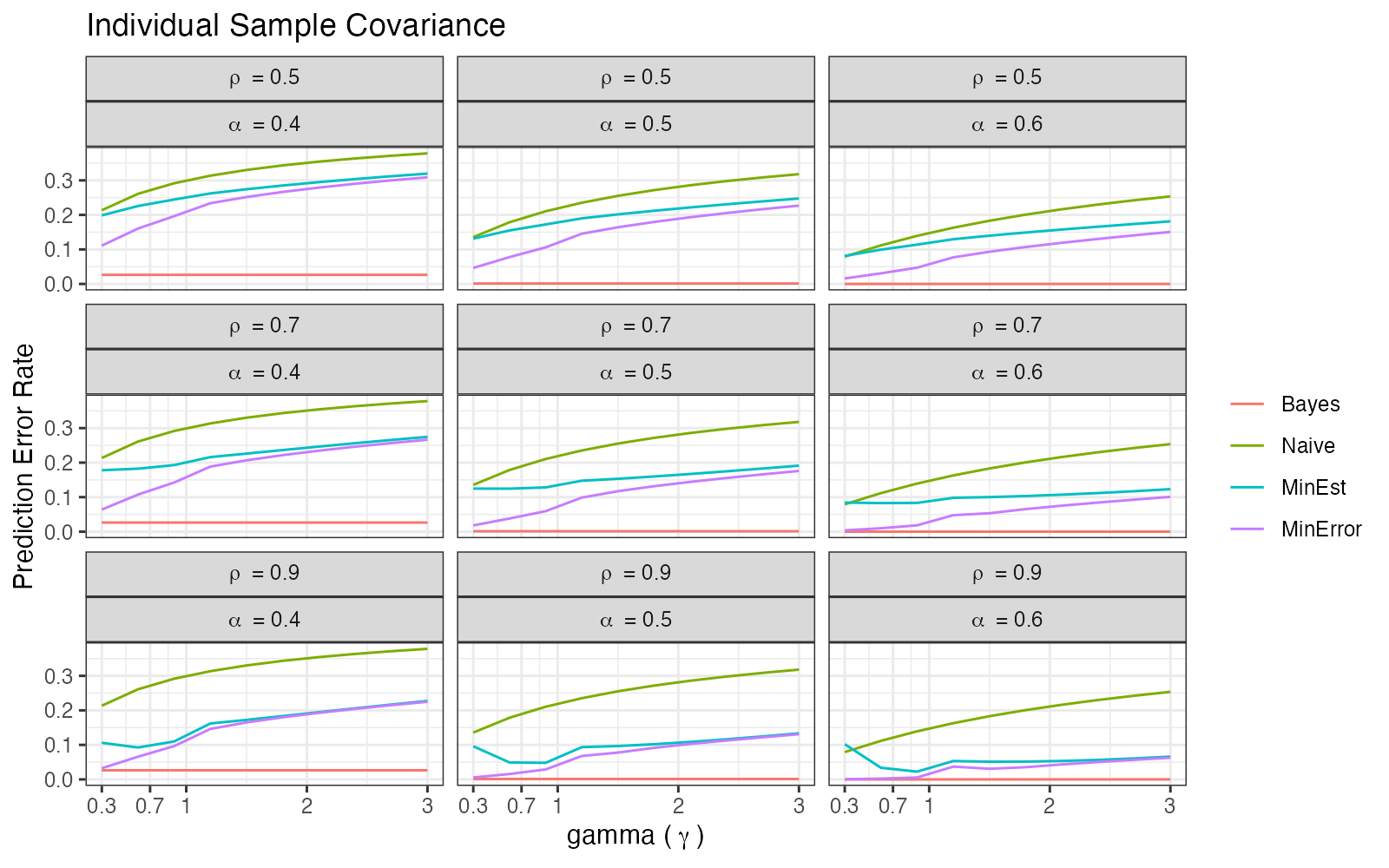}
	\centering
	\caption{Individual sample covariance matrices}
	\label{indcompare}
\end{figure}
\begin{figure}
	\includegraphics[width= 1\textwidth, height=0.32\textheight]{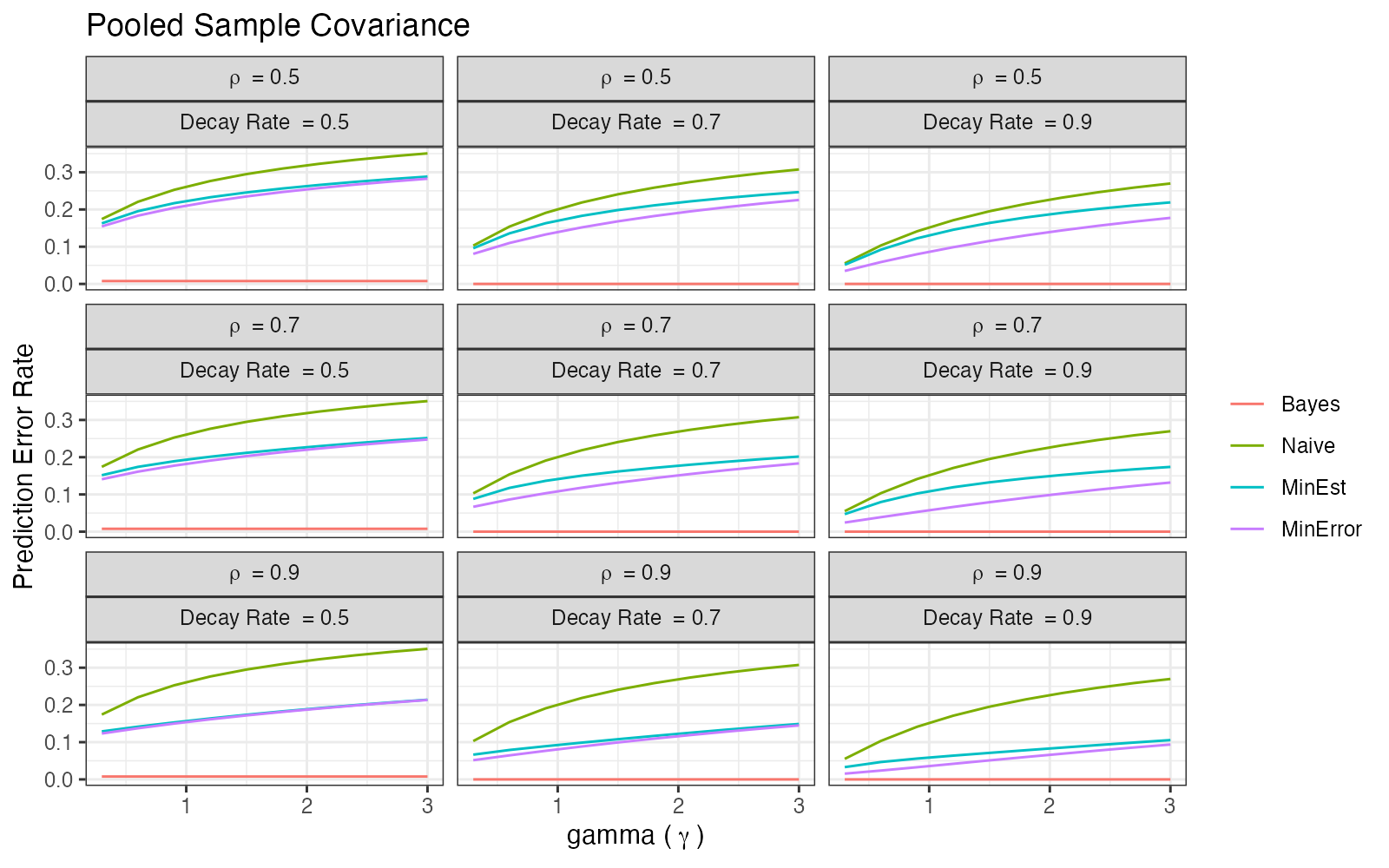}
	\centering
	\includegraphics[width= 1\textwidth, height=0.32\textheight]{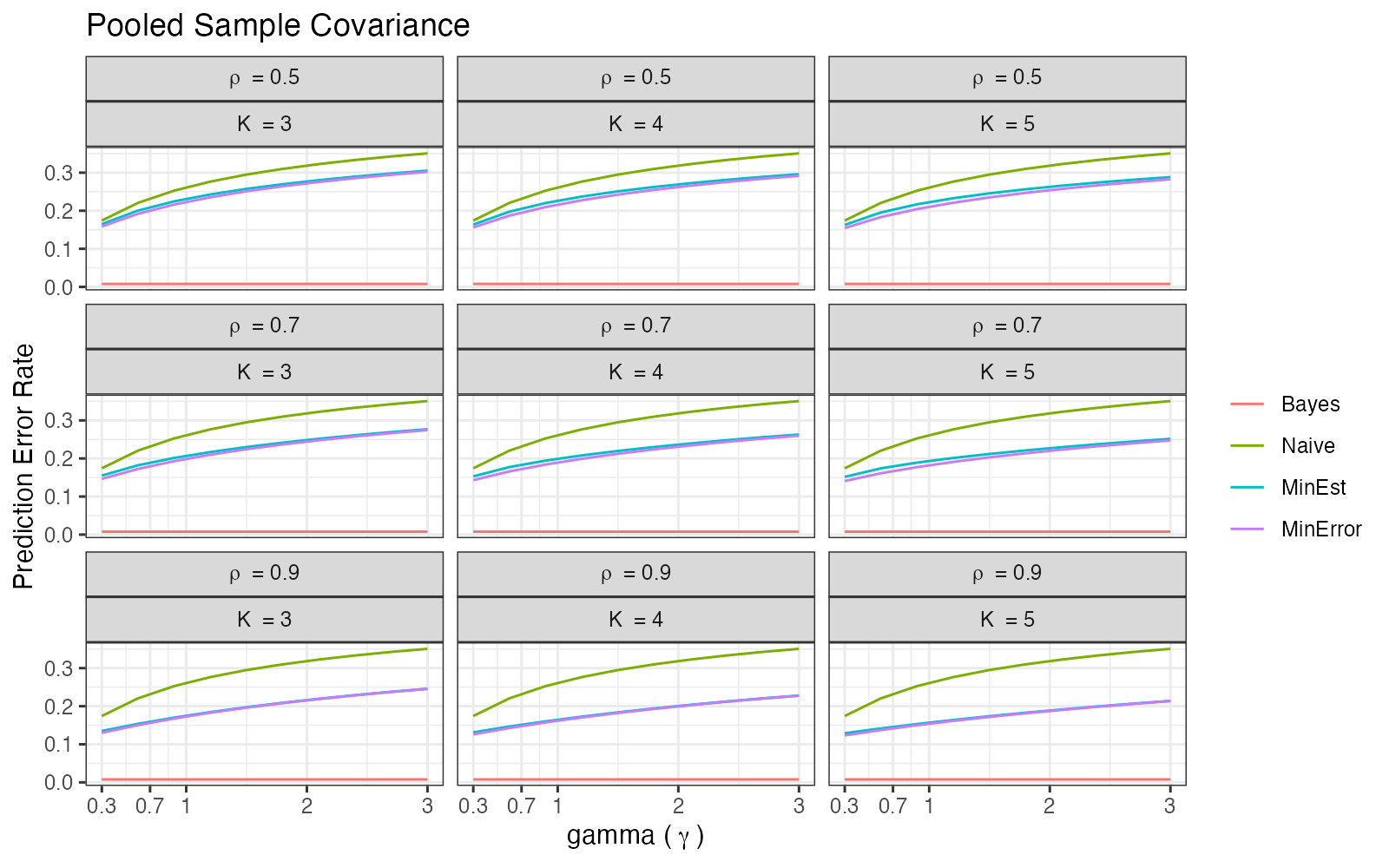}
	\centering
		\includegraphics[width= 1\textwidth, height=0.32\textheight]{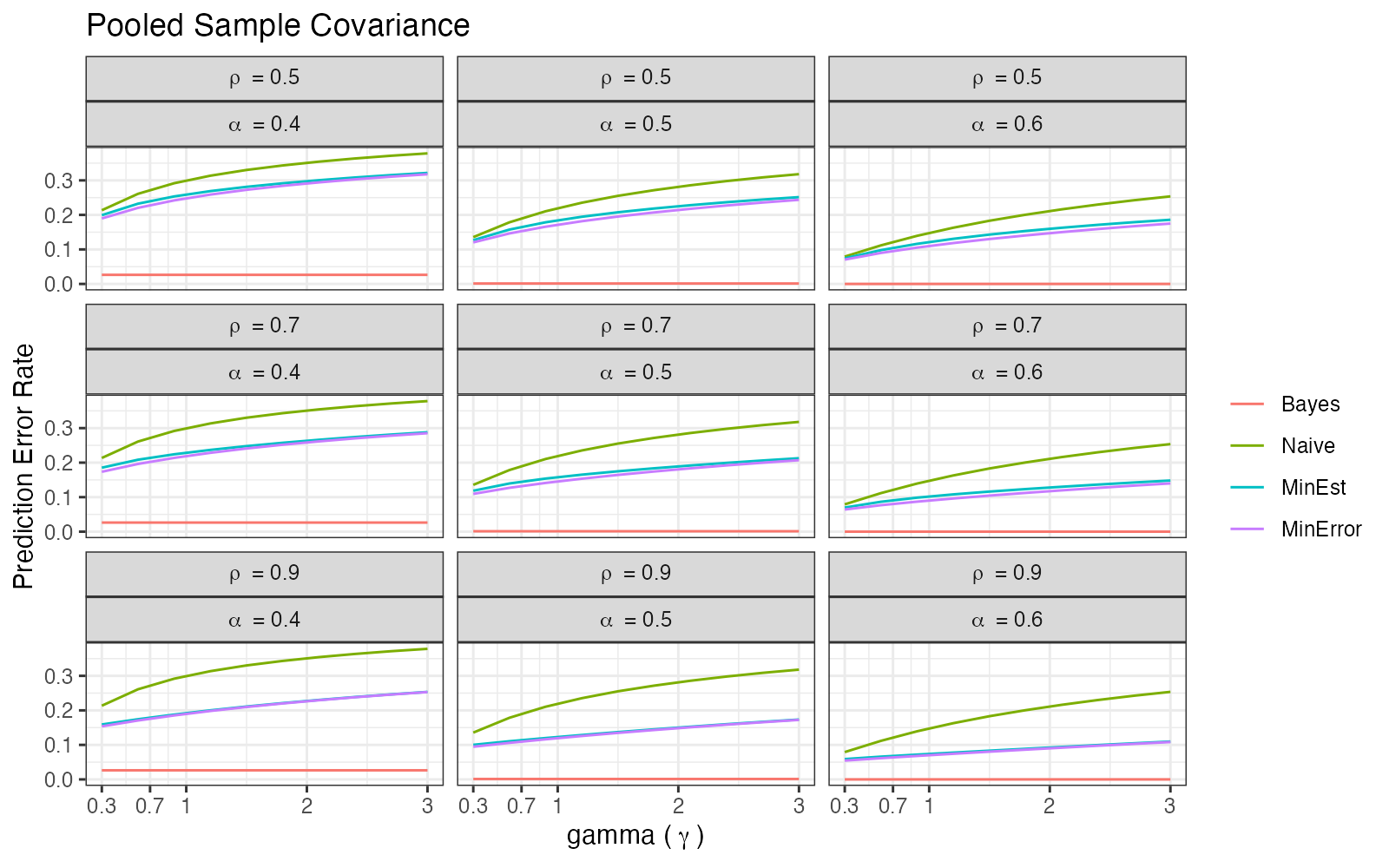}
	\centering
	\caption{Pooled covariance matrices}
		\label{poolcompare}
\end{figure}
\subsection{Estimator for Hyperparameters}\label{sec:estim-alpha-rho}
We have assumed through out this paper that $\alpha_{k}^2, \rho_{k  k'}$ are known constants. We prove that the usual moment estimators, listed below, are consistent after debiasing.
\[\hat{\alpha}_{k}^2 = \sum_{i = 1}^p [(\hat{\mu}_{+1, k})_i - (\hat{\mu}_{-1, k})_i]^2 / 4
-\gamma .
\]
\[\hat{\rho}_{k k'} = \sum_{i = 1}^p [(\hat{\mu}_{+1, k})_i - (\hat{\mu}_{-1, k})_i] [(\hat{\mu}_{+1, k'})_i - (\hat{\mu}_{-1, k'})_i] / 4
-\gamma .
\]
In order to prove the consistensies of the estimators above, it is sufficient to prove the following.
\begin{proposition}\label{pr:alpha} 
	For a length $p$ random vector $\mu$ and a $n \times p$ matrix $\bX$ generated by 
	\[E(\mu) = 0 \ \ \ Var[{\mu}] = \alpha^2 \II_p / p \ \ \ X_{ij} | \mu_j \sim N(\mu_j, 1) \ \ \ Cov(X) = \Sigma \]
	We have as $p, n \rightarrow \gamma, p/n \rightarrow \gamma$,
	\[\hat{\alpha}^2 := p \tilde{\alpha} = \sum_{j = 1}^{p} \hat{\mu}_j^2 \rightarrow \alpha^2 + \gamma. \]
\end{proposition}
\begin{proof}[Proof of proposition \ref{pr:alpha}] Starting with the expectation, let us write the sample means as $\hat{\mu}_j = \frac{\sum_{i = 1}^{n} X_{ij}}{n}$. We then have 
	\begin{align*}
		\EE(\tilde{\alpha}^2) &= \sum_{j = 1}^{p}  \EE( \hat{\mu}_j^2 )/ p \\
		&= \sum_{j = 1}^{p}  \Big[ \EE_{\mu_j} [Var( \hat{\mu}_j | \mu_j = \mu)] + Var_{\mu_j} [\EE( \hat{\mu}_j | \mu_j = \mu)]
		\Big]/ p \\
		&= \sum_{j = 1}^{p}  \Big[ \EE_{\mu_j} (\Sigma_{jj} / n )+ Var_{\mu_j} ( \mu_j )
		\Big]/ p \\
		&= \alpha^2 / p + \frac{\sum_{j = 1}^{p} \Sigma_{jj} }{n p}
	\end{align*}
	Thus $\EE(\hat{\alpha}^2) = \alpha^2  + p/n$. We then move on to prove $\hat{\alpha}^2$ has diminishing variance.  We have
	\begin{align*}
		Var(\hat{\alpha}^2) &= 	Var( \sum_{j = 1}^{p}  \hat{\mu}_j^2 )/ p^2 
		= \left[\sum_{j = 1}^{p} Var(\hat{\mu}_j^2 ) + 2 \sum_{j < k}^{p} Cov(\hat{\mu}_j^2, \hat{\mu}_k^2)\right]/ p^2 
	\end{align*}
	Given the generating model in the propostion, we know marginally $\hat{\mu}_j \sim N(0, \alpha^2 / p + 1 / n)$. Therefore,
	\[\EE(\hat{\mu}_j^4) = 3 (\alpha^2 / p + 1 / n)^2  \ \ \ \EE(\hat{\mu}_j^2) = \alpha^2 / p + 1 / n  \]
	\[Var(\hat{\mu}_j^2) = \EE(\hat{\mu}_j^4) - [\EE(\hat{\mu}_j^2)]^2 = 2 (\alpha^2 / p + 1 / n)^2  \]
	For the covariance terms, we have 
	\[Cov(\hat{\mu}_j^2, \hat{\mu}_k^2) / p^2 
	= Cov\left( (\sum_{i = 1}^{n} X_{ij})^2, (\sum_{i = 1}^{n} X_{ik})^2\right) / p^2 n^4\]
	For notation simplicity $X_i := X_{ij}$ and $Y_i := X_{ik}$
	\begin{align*}
		Cov\left((\sum_{i = 1}^{n} X_{i})^2, (\sum_{j = 1}^{n} Y_{j})^2\right) 
		=&~ \sum_{i=1}^n \sum_{j=1}^n Cov (X_i^2, Y_j^2) + 2\sum_{i=1}^n \sum_{j=1}^n \sum_{l>j}^n Cov(X_i^2, Y_j Y_l) \\
		&~+ 2\sum_{i=1}^n \sum_{k>i}^n \sum_{j=1}^n Cov(X_i X_k, Y_j^2) 
		+ 4\sum_{i=1}^n \sum_{k>i}^n \sum_{j=1}^n \sum_{l>j}^n  Cov(X_i X_k, Y_j Y_l)
	\end{align*}
	\begin{align*}
		Cov (X_i^2, Y_j^2) =&~ \EE(X_i^2 Y_j^2) - \EE(X_i^2)\EE(Y_j^2) 
		= 1 + 2 \Sigma_{i j}^2 - 1 
		= 2 \Sigma_{i j}^2.
	\end{align*}
	Here the second line holds true only for normal random variables by writing $Y_i =  \Sigma_{i j} X_i + \sqrt{1 - \Sigma_{i j}} Z_i$.
	\[Cov(X_i^2, Y_j Y_l) =  Cov(X_i X_k, Y_j^2) = Cov(X_i X_k, Y_j Y_l) = 0\]
	\[Cov(X_i^2, Y_j Y_l) = \EE(X_i^2 Y_j Y_l) - \EE(X_i^2)\EE(Y_j Y_l) = \EE [X_i^2 \EE( Y_j Y_l | X_i)] = 0\]
	\[Cov(\hat{\mu}_j^2, \hat{\mu}_k^2) / p^2 = 2  \frac{\sum_{i=1}^n \sum_{j=1}^n \Sigma_{i j}^2}{n^4 p^2}  \]
	Finally, we reach that
	\begin{align*}
		Var(\hat{\alpha}^2) &= \frac{ 2 (\alpha^2 / p + 1 / n)^2 }{p} 
		+ \frac{2\sum_{i=1}^n \sum_{j=1}^n \Sigma_{i j}^2}{n^4 p^2} * \frac{p(p-1)}{2 p^2} 
		= O\left(\frac{1}{p^3}\right)
	\end{align*}
	and $Var(\hat{\alpha}^2) = O(1 / p)$. Since $p/n\rightarrow \gamma$, this finishes the proof.
\end{proof}

\section{Limiting Error and Optimal Weight under Unequal Sampling}
 Recall the error rate under Assumption~\ref{as:TCG} is 
 	\[Err(\bw) = \pi_{-} \Phi\left(\frac{\hat{d}(\bw)^{\top} \mu_{-1}+ \hat{b}_K}{\sqrt{\hat{d}(\bw)^{\top} \Sigma \hat{d}(\bw)}}\right)+\pi_{+} \Phi\left(-\frac{\hat{d}(\bw)^{\top} \mu_1+ \hat{b}_K}{\sqrt{\hat{d}(\bw)^{\top} \Sigma \hat{d}(\bw)}}\right)\]
 	Now the stochastic representations become
 	\[\hat{\delta}_{k} = \delta_{k} +  \Sigma^{1/2} \left[ \frac{\tilde{\bZ}_{k, +1}}{\sqrt{n_{k, +1}}} - \frac{\tilde{\bZ}_{k, -1}}{\sqrt{n_{k, -1}}}\right] \]
 	\[\hat{\mu}_{k} = \bar{\mu}_k+  \Sigma^{1/2} \left[ \frac{\tilde{\bZ}_{k, +1}}{\sqrt{n_{k, +1}}} + \frac{\tilde{\bZ}_{k, -1}}{\sqrt{n_{k, -1}}}\right] \]
 	The intercept terms are no longer zeros and
 	\begin{align*}
 		\hat{b}_k =& - \hat{\delta}^{\top}_k (\hat{\Sigma}_k + \lambda_k \II_p)^{-1}(\hat{\mu}_{-1, k} + \hat{\mu}_{+1, k}) / 2 \\
 		=& \frac{1}{4 n_{k, -1}} \tilde{\bZ}_{k, -1}^\top \Sigma^{1/2} (\hat{\Sigma}_k + \lambda_k \II_p)^{-1}\Sigma^{1/2} \tilde{\bZ}_{k, -1} - \frac{1}{4 n_{k, +1}} \tilde{\bZ}_{k, +1}^\top \Sigma^{1/2} (\hat{\Sigma}_k + \lambda_k \II_p)^{-1}\Sigma^{1/2} \tilde{\bZ}_{k, +1} \\
 		\rightarrow_{a.s.}& \frac{\gamma_{k, -1} - \gamma_{k, +1}}{4\gamma_k} (\frac{1}{\lambda_k v_{F_{\gamma_k}}(-\lambda_k)} -1 )
 	\end{align*}
 	Where we use $p / n_{k, \pm 1} \rightarrow \gamma_{k, \pm 1}$. For the numerator we have
 	\begin{align*}
 		\hat{d}(\bw)^\top \mu_{\pm 1} = \bw^\top vec \left[ (\bar{\mu}_k \pm \hat{\delta}_{k})^\top (\hat{\Sigma}_k + \lambda_k)^{-1}  \delta_K\right] 
 		\rightarrow_{a.s.} \bw^\top vec \left[\pm \rho_{k  K} \alpha_{k} \alpha_{K} m_{F_{\gamma_k}}(-\lambda_{k})  \right]
 	\end{align*}
 	For the denominator term,
 	\begin{align*}
 		\hat{d}(\bw)^\top \Sigma \hat{d}(\bw) =&~ \bw^\top mat\left[ \hat{\delta}_{k}^\top (\hat{\Sigma}_k + \lambda_k)^{-1}  \Sigma  (\hat{\Sigma}_{k'} + \lambda_{k'})^{-1} \hat{\delta}_{k'} \right] \bw 
 		:= \bw^\top S \bw.
 	\end{align*}
 	For $k \neq k'$
 	\[S_{kk'} \rightarrow_{a.s.} \alpha_{k}^2 \frac{v_{F_{\gamma_k}}(-\lambda) - \lambda v'_{F_{\gamma_k}}(-\lambda)}{\gamma_k[\lambda v_{F_{\gamma_k}}(-\lambda)]^2} + \frac{\gamma_{k, -1} + \gamma_{k, +1}}{4} \frac{v'_{F_{\gamma_k}}(-\lambda) -v^2_{F_{\gamma_k}}(-\lambda)}{\lambda^2 v^4_{F_{\gamma_k}}(-\lambda)} \]
 	while
 	\[S_{kk} \rightarrow_{a.s.} \rho_{k  k'} \alpha_{k} \alpha_{k'} \frac{\tr(M^{k k'}) / p}{\gamma_k}  + \frac{\gamma_{k, -1} + \gamma_{k, +1}}{4} \frac{v'_{F_{\gamma_k}}(-\lambda) -v^2_{F_{\gamma_k}}(-\lambda)}{\lambda^2 v^4_{F_{\gamma_k}}(-\lambda)} \]
 	for $1\le k, k' \le K$.	Again the limit of $\tr(M^{k k'}) / p$  can be found in Lemma~\ref{lem:predijasp}.
 	
 	As for the optimal weights, we know the Bayes optimal prediction is 
 	\[d_{Bayes}(x_0) = \delta_{K}^\top \Sigma^{-1} x_0 + \delta_{K}^\top \Sigma^{-1} \bar{\mu}_K + \log\left(\frac{\pi_{K, +}}{\pi_{K, -}}\right). \]
 	Note $\delta_{K}$ and $\bar{\mu}_K$ are independent by Assumption~\ref{as:RCW}, therefore by Lemma~\ref{lem:quadConv} the second term is zero. In this case, we recommend an estimator in the form of $\sum_{k = 1}^{K} \hat{d}_k + \log(\pi_{K, +} / \pi_{K, -})$ whose intercept term is consistent with the one of the Bayes direction. Further, one can still obtain the optimal weights that minimize criteria \eqref{eq:CrEst} and \eqref{eq:CrPred}. However, the proposition \ref{pr:PredIsOpt} no longer holds and the prediction weight is not guaranteed to minimize the testing data error rate. The error weight on testing data has the following expression.
 	\begin{align*}
 		\pi_{K, -} \Phi\left(\frac{\hat{d}(\bw)^{\top} \mu_{K, -1}}{\sqrt{\hat{d}(\bw)^{\top} \Sigma \hat{d}(\bw)}} + \frac{\hat{b}_K}{\sqrt{\hat{d}(\bw)^{\top} \Sigma \hat{d}(\bw)}}\right)+\pi_{K, +} \Phi\left(-\frac{\hat{d}(\bw)^{\top} \mu_{K,1}}{\sqrt{\hat{d}(\bw)^{\top} \Sigma \hat{d}(\bw)}} -\frac{\hat{b}_K}{\sqrt{\hat{d}(\bw)^{\top} \Sigma \hat{d}(\bw)}} \right) 
 	\end{align*}
The difficulty arises as $\hat{b}_K$ is now non-zero, and merely maximizing the first fraction, as the optimal prediction weight does, is no longer sufficient. When $K$ is low, one can potentially search for a local or even global optimal point for $\bw$ with numeric methods.
\subsection*{Expected Prediction Error}
For any linear classifier ${\rm sign}(\bw^{\top} x_0 + b)$, we can write its prediction error rate as followed
\begin{align*}
	&\PP({\rm sign}(\bw^{\top} x_0 + b) \neq y_0) \\
	=&~ \pi_- \PP({\rm sign}(\bw^{\top} x_0 + b) \neq -1 | y_0 = -1)  +  \pi_+ \PP({\rm sign}(\bw^{\top} x_0 + b) \neq 1 | y_0 = 1) 
\end{align*}
The expressions for the two parts above are symmetric, so we can only look the first part.
\begin{align*}
	&\PP({\rm sign}(\bw^{\top} x_0 + b) \neq -1 | y_0 = -1)  \\
	=&~\PP(\bw^{\top} \mu_{-1} + \bw^{\top} \epsilon + b > 0) \\
	=&~\PP( \bw^{\top} \epsilon  \leq \bw^{\top} \mu_{-1} + b)  = \Phi\left(\frac{\bw^{\top} \mu_{-1} + b}{\sqrt{\bw^{\top} \Sigma \bw}}\right)
\end{align*}

\end{document}